%% file: arxiv.tex
\documentclass[acmsmall,screen,nonacm]{acmart}% no lineno for ACM
\copyrightyear{2026}
\acmYear{2026}
\usepackage{graphicx}
\usepackage{url}
\usepackage{verbatim}
\usepackage{xcolor}
\usepackage[breakable]{tcolorbox}
\tcbuselibrary{breakable}
\usepackage{iftex}
\ifPDFTeX
\else
  \usepackage{fontspec}
\fi
\usepackage{enumitem}
\usepackage{interval}
\usepackage{tabularx}
\usepackage{array}
\usepackage{longtable}
\usepackage{float}
\usepackage{centernot}
\usepackage{tikz-cd}

\newcolumntype{Y}{>{\raggedright\arraybackslash}X}
\newcolumntype{L}[1]{>{\raggedright\arraybackslash}p{#1}}
\newcolumntype{V}[1]{>{\raggedright\arraybackslash\strut}p{#1}}

\usepackage{stmaryrd}

\usepackage{booktabs}
\usepackage[normalem]{ulem}% 'normalem' prevents it from turning italics into underlines

\ifPDFTeX
\else
\fi

\setcopyright{none}
\makeatletter
\let\PipeAddContentsLine\addcontentsline
\newcommand*\PipeFilteredAddContentsLine[3]{%
  \def\PipeTocType{#2}%
  \def\PipeParagraphType{paragraph}%
  \def\PipeSubparagraphType{subparagraph}%
  \ifx\PipeTocType\PipeParagraphType
  \else
    \ifx\PipeTocType\PipeSubparagraphType
    \else
      \PipeAddContentsLine{#1}{#2}{#3}%
    \fi
  \fi
}
\let\addcontentsline\PipeFilteredAddContentsLine
\makeatother
\definecolor{PipeForOmaBackground}{gray}{0.94}
\definecolor{PipeForOmaBorder}{gray}{0.78}
\newtcolorbox{PipeForOmaFrame}{
  breakable,
  enforce breakable,
  lines before break=0,
  colback=PipeForOmaBackground,
  colframe=PipeForOmaBorder,
  boxrule=0.35pt,
  arc=0pt,
  outer arc=0pt,
  boxsep=0pt,
  left=0.7em,
  right=0.7em,
  top=0.35\baselineskip,
  bottom=0.35\baselineskip,
  before skip=0.55\baselineskip,
  after skip=0.55\baselineskip,
  width=\linewidth,
  left skip=0pt,
  right skip=0pt
}
\newenvironment{foroma}{%
  \par\begingroup\small\begin{PipeForOmaFrame}%
  \noindent
  \textit{For Oma:}\hspace{0.25em}%
}{%
  \end{PipeForOmaFrame}\endgroup\par
}
\newcommand{\ltuple}[1]{\langle\nobreak#1,\allowbreak}
\newcommand{\mtuple}[1]{\:#1\nobreak,\allowbreak}
\newcommand{\rtuple}[1]{\:#1\nobreak\rangle}

\newcommand{\pair}[2]{\ltuple{#1}\allowbreak\rtuple{#2}}

\newcommand{\decatuple}[9]% 10 parameters
{\ltuple{#1}\mtuple{#2}\mtuple{#3}\mtuple{#4}\mtuple{#5}\mtuple{#6}\mtuple{#7}\mtuple{#8}\mtuple{#9}\rtuple}
\makeatletter
\newcommand{\undecatuple}[9]% 11 parameters
{\ltuple{#1}\mtuple{#2}\mtuple{#3}\mtuple{#4}\mtuple{#5}\mtuple{#6}\mtuple{#7}\mtuple{#8}\mtuple{#9}\@undecatuple}
\newcommand{\@undecatuple}[2]{\mtuple{#1}\rtuple{#2}}
\newcommand{\dodecatuple}[9]% 12 parameters
{\ltuple{#1}\mtuple{#2}\mtuple{#3}\mtuple{#4}\mtuple{#5}\mtuple{#6}\mtuple{#7}\mtuple{#8}\mtuple{#9}\@dodecatuple}
\newcommand{\@dodecatuple}[3]{\mtuple{#1}\mtuple{#2}\rtuple{#3}}
\newcommand{\tridecatuple}[9]% 13 parameters
{\ltuple{#1}\mtuple{#2}\mtuple{#3}\mtuple{#4}\mtuple{#5}\mtuple{#6}\mtuple{#7}\mtuple{#8}\mtuple{#9}\@tridecatuple}
\newcommand{\@dtridecatuple}[4]{\mtuple{#1}\mtuple{#2}\mtuple{#3}\rtuple{#4}}
\newcommand{\tetradecatuple}[9]% 14 parameters
{\ltuple{#1}\mtuple{#2}\mtuple{#3}\mtuple{#4}\mtuple{#5}\mtuple{#6}\mtuple{#7}\mtuple{#8}\mtuple{#9}\@tetradecatuple}
\newcommand{\@tetradecatuple}[5]{\mtuple{#1}\mtuple{#2}\mtuple{#3}\mtuple{#4}\rtuple{#5}}
\newcommand{\pentadecatuple}[9]% 15 parameters
{\ltuple{#1}\mtuple{#2}\mtuple{#3}\mtuple{#4}\mtuple{#5}\mtuple{#6}\mtuple{#7}\mtuple{#8}\mtuple{#9}\@pentadecatuple}
\newcommand{\@pentadecatuple}[6]{\mtuple{#1}\mtuple{#2}\mtuple{#3}\mtuple{#4}\mtuple{#5}\rtuple{#6}}
\newcommand{\hexadecatuple}[9]% 16 parameters
{\ltuple{#1}\mtuple{#2}\mtuple{#3}\mtuple{#4}\mtuple{#5}\mtuple{#6}\mtuple{#7}\mtuple{#8}\mtuple{#9}\@hexadecatuple}
\newcommand{\@hexadecatuple}[7]{\mtuple{#1}\mtuple{#2}\mtuple{#3}\mtuple{#4}\mtuple{#5}\mtuple{#6}\rtuple{#7}}
\makeatletter
\RequirePackage{relsize}
\newbox\mystrutbox
\newcommand{\@ustrut}[1]{\setbox\mystrutbox\hbox{#1\strut}\hbox{\vrule 
     height\ht\mystrutbox
     depth\z@
     width\z@}}
\newcommand{\@lstrut}[1]{\setbox\mystrutbox\hbox{#1\strut}\hbox{\vrule 
     height\z@
     depth\dp\mystrutbox
     width\z@}}
\newcommand{\ustrut}{\@ustrut{\normalsize}}
\newcommand{\ulstrut}{\@ustrut{\relsize{1}}}
\newcommand{\uLstrut}{\@ustrut{\relsize{2}}}
\newcommand{\uhstrut}{\@ustrut{\relsize{3}}}
\newcommand{\uHstrut}{\@ustrut{\relsize{4}}}
\newcommand{\lstrut}{\@lstrut{\normalsize}}
\newcommand{\llstrut}{\@lstrut{\relsize{1}}}
\newcommand{\lLstrut}{\@lstrut{\relsize{2}}}
\newcommand{\lhstrut}{\@lstrut{\relsize{3}}}
\newcommand{\lHstrut}{\@lstrut{\relsize{4}}}

\newcommand{\usstrut}{\@ustrut{\relsize{-3}}}%\scriptsize
\newcommand{\lsstrut}{\@lstrut{\relsize{-3}}}%\scriptsize
\makeatother
\hypersetup{hypertexnames=false}
\begin{document}

\title{Ceci n'est pas une pipe: A semantics of justified reliance for AI systems}

\input{authors}

\begin{abstract}
An AI agent may keep a successful check in context after an edit has
invalidated its evidence. We give a semantics for specifying when a system may
rely on a message and act on it. Reliability checks an asserted claim against
accepted domain knowledge, sources, and messages in the selected context.
Justified reliance additionally checks current evidence, authority, and the
proposed world change. Relational component semantics and event traces
represent nondeterministic outputs, iteration, and external changes.

We characterise when merging claims preserves derivable conclusion classes,
then construct the coarsest permitted quotient whose rules preserve the exact
statuses of designated concrete claims. Construction is polynomial for unary
rules; deciding even one permitted merge is coNP-complete for acyclic binary
rules. This applies completeness and strong-preservation methods to the
execution of quotiented rules. We also relate transition safeguards to runtime
enforcement. Relative to a supplied set of possible worlds, a sound guard
accepts an observation only when every matching world passes the checks.
Exact acceptance requires observations that distinguish opposite-decision
worlds; choosing sufficient services is NP-complete on explicit finite tables.
A complete finite passport-renewal instance, Lean theorem encodings, Z3
instance checks, and an OCaml harness exercise the semantics. Domain knowledge,
interpretation, source licenses, and policy remain application inputs.
\end{abstract}




% \received is assigned by the journal.

\maketitle

\section{Introduction}

An agent helping Oma renew her passport checks a draft and records that it is
ready. Oma then edits a required field. The successful check remains in the
agent's context, but its evidence concerns the earlier draft. A subsequent
action must establish readiness for the edited draft and the authority to act
on it. Even if the old check remains true of the old draft, using it for the
new one is unjustified. Like Magritte's pipe, the message represents something;
it is not the represented fact or permission.

We give a semantics of justified reliance for AI systems. It distinguishes the
reliability of an asserted claim from the justification of a system step that
uses the claim or its evidence. Specifications state the conditions on those
steps, including evidence freshness, authority, and the proposed world change.
They give runtime acceptance checks and the obligations under which a
safeguard enforces the application policy. A proof or static analysis of
reliance must establish the same semantic obligations.

\subsection{The framework in brief}\label{sec:framework-brief}

The semantics separates three knowledge bases: accepted domain knowledge
\(K_U\), source-derived knowledge \(K_S\), and effective knowledge \(K_E\)
from the selected context. Their objects are respectively claims, source
items, and messages. A source item groups claims that the representation
cannot distinguish; a partial denotation map sends a message to such an item,
and \(\mathsf{Asserts}\) identifies the claims actually asserted. In each
base, an item and its negation may be derived independently, giving four
statuses: justified only, refuted only, both, or neither.
\(\mathsf{ReliableClaim}\) requires an assertion, positive justification at
all three levels, and a source-support link to the claim.
These distinctions impose different obligations: domain justification,
justification from the selected sources, and availability without contrary
effective evidence. Denotation determines candidate claims; assertion selects
what the system asks a user to rely on; a support license must justify that
specific claim across the source selections in which it is used.

An event trace records component results, sources, witnesses, authority, and
world changes. Symbolic and neural components and world-facing services
communicate under an orchestrator; their relational behaviour allows repeated
calls to return different results and executions to retrieve, check, revise,
and act repeatedly. For a proposed event, \(\mathsf{AcceptOK}\) combines the
claim, witness, authority, and actuation checks.
\(\mathsf{AllowedEvent}\) is the independent application policy.
A safeguard may suppress system-controlled events; user and environment steps
remain outside its control. Guarded runs obey the policy when accepted events
pass the checks, the checks imply policy permission, and event records agree
with their producers' results (Theorem~\ref{th:run-enforcement}).

\paragraph{Contributions.}
The principal results concern two ways that information loss limits justified
reliance.
\begin{itemize}[leftmargin=*,itemsep=2pt,topsep=2pt]
\item \emph{Choosing what to distinguish.} We construct the unique coarsest
  permitted quotient whose rules preserve the exact four-way statuses of
  designated concrete claims for every input fact set
  (Theorem~\ref{th:optimal-query-quotient}). The construction is polynomial
  for unary rules, while deciding whether one permitted identification can
  be retained is coNP-complete already for acyclic binary rules and one
  signed query (Theorem~\ref{th:query-complexity}). The proof connects
  contextual query equivalence to execution of the original rules after
  quotienting. It uses completeness and strong-preservation theory
  \cite{DBLP:journals/jacm/GiacobazziRS00,RanzatoTapparo2007}.
  Global stability separately characterises preservation of every
  conclusion class (Theorem~\ref{th:forget-transfer}), with a finite test
  and a successive-merge result
  (Theorems~\ref{th:finite-stability} and~\ref{th:successive-forgetfulness}).
\item \emph{Observing the world.} Relative to a supplied set of possible
  worlds, we construct the most permissive sound guard from service
  observations. It accepts exactly the acceptable cases iff indistinguishable
  worlds require the same decision
  (Theorem~\ref{th:greatest-observable-guard}); a reported rejection reason
  must also hold throughout the observation class
  (Theorem~\ref{th:observable-safeguard}). On explicit finite tables, selecting
  at most \(k\) sufficient observation services is NP-complete by reduction
  from Set Cover (Theorem~\ref{th:selecting-observations}).
\end{itemize}

Endpoint tests additionally establish statuses when \(K_U\) is known only
between lower and upper bounds (Theorem~\ref{th:interval-status}). The
diagnostic normal forms and safeguard contracts check the consistency of the
definitions and locate proof obligations; they are supporting results.
Sections~\ref{sec:finite-readiness-instance}
and~\ref{sec:finite-readiness-execution}, joined by
Table~\ref{tab:readiness-ledger}, calculate a complete finite passport instance,
including accepted readiness, invalidated evidence, and two neural outputs
from the same input. Lean theorem encodings, Z3 instance checks, and the OCaml
harness in Section~\ref{sec:implementation} provide formal and executable
validation within their stated scopes.

The application supplies exact \(K_U\), or sound lower and upper bounds,
together with denotation, assertion, source support, context selection,
authority, and policy. The framework does not construct or verify these
inputs, recover meaning from prose, or model human reasoning. An unsettled
obligation does not justify acceptance. The quotient construction takes a
finite rule representation as input; the finite passport instance supplies
one explicitly. Its executable checks establish the resulting judgments,
while the adequacy of that domain model remains an application obligation.

\section{Knowledge Base \label{sec:knowledge-base}}
At time \(t\in\mathbb N\), a knowledge base \(K=\langle D,F,R,t\rangle\) consists of a
domain \(D\), facts \(F\subseteq D\), and rules \(R\) over \(D\). A rule
\(\frac{P}{c}\) has premises
$P\in\wp(D)\) and a conclusion \(c\in D\). If \(P=\emptyset\) the rule is called an
axiom.  The rule system  \(R\triangleq\Bigl\{\frac{P_i}{c_i}\Bigm|
i\in\rho\Bigr\}\) over~$D\) is a possibly infinite family of rules indexed by a
set~$\rho\).

This rule format includes Hilbert-style proof systems
\cite{Hilbert1928proof-system} for classical
\cite{Shoenfield1967-classical-logic}, intuitionistic
\cite{vanDalen2013intuitionistic-logic}, and modal logic
\cite{Blackburn2001modal-logic}, as well as grammars
\cite{DBLP:journals/tit/Chomsky56}, deductive databases
\cite{DBLP:books/mg/SKS20}, knowledge graphs
\cite{HoganEtAl2022knowledge-graphs}, and deductive systems
\cite{Aczel1977inductive-definitions}. We do not require the facts or their
closure to be consistent.

\begin{foroma}
One rule may say that if the pose, lighting, framing, and image-quality checks
pass for a candidate photograph, then it meets the application's photograph
requirement.
\end{foroma}

We define the consequences of \(F\) under \(R\) as follows.
\begin{itemize}[leftmargin=*,itemsep=0pt,topsep=2pt]
\item \emph{Inference.} 
If \(X\) is the set of facts already reached, then
$\textsf{Infer}_{\!R}(F)(X)\triangleq F\cup X\cup\{c_i\mid i\in\rho\wedge
P_i\subseteq F\cup X\}\) adds facts \(F\) to \(X\), and every conclusion whose
premises are available from \(F\cup X\).

\item \emph{Closure.} For a set \(B\subseteq D\), let
$\textsf{Infer}^{\ast}_{R}(B)\triangleq \textsf{lfp}^{\subseteq}\bigl(\lambda
X.\textsf{Infer}_{R}(B)(X)\bigr)\), where \(\textsf{lfp}^{\subseteq}\) is the
least fixpoint in order \(\subseteq\), which exists by
Tarski's theorem~\cite{Tarski1955fixpoint} since \(X\mapsto
\textsf{Infer}_{R}(B)(X)\) is \(\subseteq\)-monotone in \(X\).  A conclusion
\(c\) is justified from facts \(F\) and hypotheses \(H\) when
$c\in\textsf{Infer}^{\ast}_{R}(F\cup H)\).  

\begin{lemma}[Closure]\label{lem:closure}
\(\textsf{Infer}^{\ast}_{R}\) is extensive, monotone, and idempotent, and
\(\textsf{Infer}^{\ast}_{R}(B)\subseteq X\) whenever 
\(\textsf{Infer}_{R}(B)(X)\subseteq X\) (such an \(X\) contains \(B\), since \(B\subseteq  \textsf{Infer}_{R}(B)(X)\)).
\end{lemma}
\noindent\emph{Informally.} Rule closure contains the starting facts; adding
facts cannot remove conclusions; and closing an already closed set changes
nothing. It is the smallest set that contains the starting facts and is closed
under the rules.

\item \emph{Paraconsistent justification.\label{sec:paracons}} Given
\(K=\langle D,F,R,t\rangle\), we say that \(K\) justifies a claim \(c\) when
$K\vdash c \triangleq c\in\textsf{Infer}^{\ast}_{R}(F)\), and that \(K\) does not
justify \(c\) when \(K\not\vdash c \triangleq
c\notin\textsf{Infer}^{\ast}_{R}(F)\). A knowledge base may justify a fact, its
negation, both, or neither. Derivability from one knowledge base does not by
itself satisfy the reliability predicate of
Section~\ref{sec:candidate-to-reliable-claims}. Absence of justification is not
negation: \(K\not\vdash c\) does not
imply \(K\vdash\neg c\). Each claim \(c\in D\) has a
negation \(\neg c\in D\), with \(\neg\neg c=c\). We distinguish four base statuses: \(K\vdash^{+}c
\triangleq K\vdash c\wedge K\not\vdash\neg c\), \(K\vdash^{-}c \triangleq
K\vdash\neg c\wedge K\not\vdash c\), \(K\vdash^{\pm}c \triangleq K\vdash c\wedge
K\vdash\neg c\), \(K\vdash^{?}c\triangleq K\not\vdash c\wedge K\not\vdash\neg c\).
When the time index is useful, we
write \(K\vdash^{\times}_t c\) for the corresponding base status at time \(t\),
with \(\times\in\{+,-,\pm,?\}\).  
\end{itemize}

This four-way distinction follows the Belnap--Dunn/FDE treatment of independent
support for a claim and its negation~\cite{belnap1977useful,dunn1976intuitive}.
Our closure records whether a claim and its negation are independently
derivable; it contains no attacks, preferences, or choice of extension. Dung
frameworks study acceptability in graphs of attacking arguments; ABA and
ASPIC+ construct arguments from assumptions and rules and add attack, defeat,
or preference relations
\cite{DBLP:journals/ai/Dung95,DBLP:journals/ai/BondarenkoDKT97,DBLP:conf/comma/Prakken26}.
The transfer results below use monotonicity of rule closure. A structured
argumentation semantics would need its own proofs about merging claims.

\section{Messages, Worlds, and Traces\label{sec:objects}}

\subsection{Communication Language\label{sec:communication-language}}
Components exchange messages in a common language \(\textit{Lang}\), including
labels, token sequences, programs, data tables, URLs, and tool-call schemas.
Component encoders and decoders translate between \(\textit{Lang}\) and their
internal languages (Section~\ref{sec:compute-components}).

Every symbol \(m_*\) below denotes a message in \(\textit{Lang}\), not an
object or fact in the world.
\begin{foroma} In the app, \(m_{\textit{ready}}\triangleq\texttt{result =
passed}\) is the message returned when the user runs the readiness check.  A
later event records that message in the trace.  \end{foroma}

\subsection{Semantics of the Communication Language}

Let \(\mathsf{SourceItems}\) be the equivalence classes of in-scope domain
objects defined in Section~\ref{sec:source-derived-kb}. We write
$\mathsf{denotation}_{\textit{Lang}} : \textit{Lang}\rightharpoonup
\mathsf{SourceItems}\) for the denotation of messages. 
Write \(f(x){\downarrow}\) when a partial map \(f\) is defined at \(x\), and
\(f(x){\uparrow}\) otherwise. An equality
\(\mathsf{denotation}_{\textit{Lang}}(m)=s\) asserts that the map is defined
at \(m\) with value \(s\).

Denotation identifies the source item represented by a message; it does not
establish that any candidate claim in that item is supported or current. The
expression/reference distinction is standard in denotational
semantics~\cite{scottStrachey1971mathematicalSemantics}. The application
supplies this partial map; fluent prose alone does not establish meaning or
support~\cite{DBLP:conf/acl/BenderK20}.

\begin{foroma}
In the prototype, a table maps message fields to a source item only when
exactly one row matches. The app records
\(m_{\textit{ready}}=\{\,\texttt{messageType}=\texttt{CheckResultMessage},~
\texttt{actionKind}=\texttt{checkReady},~
\texttt{resourceKind}=\texttt{currentDraft},~
\texttt{payloadKind}=\texttt{readinessResult},~
\texttt{result}=\texttt{passed}\,\}\), which denotes
\(s_{\textit{readinessChecksPassed}}\).
Zero or multiple matches leave denotation undefined. Every value used by the
lookup must occur in the message, including any relevant trace value.
An edit leaves this denotation unchanged but may invalidate its evidence.
\end{foroma}

Let \(\textit{Symbolic}\), \(\textit{Neural}\),
\(\textit{Services}\), and \(\mathsf{External}\) be pairwise disjoint sets of
symbolic components, neural components, world-facing services, and external
producers. Let \(\mathsf{Ctx}\) be the context of Section \ref{sec:Context}. Put
\[
\mathsf{Producer}\triangleq
\textit{Symbolic}\uplus\textit{Neural}\uplus
\textit{Services}\uplus\mathsf{External},
\qquad
\mathsf{Ctx}\triangleq\textit{Lang}^{\ast}.
\]
Let \(\mathord{\textit{States}}_{\mathit{Orch}}\) be the set of orchestrator states. For
each time \(t\), let \(\textit{Src}_t\) be the source pointers available at
that time. Section~\ref{sec:compute-components} defines the behaviour of the
producers, and the three
knowledge bases built from claims, sources, and contexts are defined in Section \ref{sec:information-state-ai-system}.

\subsection{Model of the World}

Let \(W\) be the possible external states, with current state \(w_t\) at time
\(t\). The service-relative observation interface
\(\mathsf{observation}:\textit{Services}\times W\to\textit{Lang}\)
may be lossy or permission-limited. The actuation interface
\(\mathsf{actuation}:W\times\textit{Lang}\to W\)
gives the successor produced by a request; the event policy separately
determines whether that request is allowed.
The world may also change outside the system's control, before any service
observes the change.
\begin{foroma}
Services read the passport, picture, and draft; actuations fill fields, save
the draft, or mark it ready. Oma may then edit a field or replace the picture.
A later readiness check must account for those changes.
\end{foroma}

\subsection{Events}\label{sec:Events}
 Let \(\textit{Ev}\) be the
application-supplied set of event records. An event is a record 
\begin{equation}\begin{array}{@{}l@{}}
e=\langle
\mathsf{name},
\mathsf{actor},
\mathsf{message},
\mathsf{kind},
\mathsf{resource},
\mathsf{observation},
\mathsf{actuation},\\
\qquad\qquad\qquad\qquad\mathsf{authority},
\mathsf{witness},
\mathsf{producer},
\mathsf{context},
\mathsf{sources},
\mathsf{time}
\rangle \in \textit{Ev}.
\end{array}
\label{eq:def:Ev}
\end{equation} 
We write e.g. \(e.\mathsf{message}\) for the \(\mathsf{message}\) field of the event \(e\).
The name identifies the event. The actor initiates the operation; the producer
is the component, service, user, or environmental source that returns the
recorded message. The kind records a read, write, check, authority event,
internal computation, or external mutation, and the resource identifies its
object. The observation and actuation fields hold world-interface messages.
The authority field records permission, consent, revocation, or delegation and
may contain event references;
the witness field records source pointers, tool results, proof objects,
checker results, timestamps, or trace slices. Tags distinguish event
references, trace slices, and source pointers. The context field records the
selected messages, the sources field records the source pointers attributed to
the event, and the time field records when the event occurred.
These event fields have types
\[
e.\mathsf{producer}\in\mathsf{Producer},\qquad
e.\mathsf{context}\in\mathsf{Ctx},\qquad
e.\mathsf{sources}\in\wp\Bigl(\bigcup_t\textit{Src}_t\Bigr),\qquad
e.\mathsf{time}\in\mathbb N,
\]
\begin{foroma}
The readiness event \(e_{\textit{ready}}\) has name \texttt{CheckReady}, actor
\texttt{AppAgent}, and message \(m_{\textit{ready}}\). A picture-check event
may record \(e.\mathsf{witness} =
\{\mathit{sp}_{\textit{gov-pic-guidance}},r_{\textit{pic-check}}\}\), with 
\(\mathit{sp}_{\textit{gov-pic-guidance}}\) a source pointer and
\(r_{\textit{pic-check}}\) a checker result.  \end{foroma}

\subsection{Traces}

A trace is a finite event sequence \(\textit{trace}=e_1\ldots e_n\), where
\[
\mathsf{Trace}\triangleq
\{\,e_1\ldots e_n\in\textit{Ev}^{\ast}\mid
\forall i<j.\ e_i.\mathsf{time}<e_j.\mathsf{time}\,\}.
\]
For a time \(t\), let \(\textit{trace}_{\le t}\) be the longest
prefix whose events have time at most \(t\), with
\(\textit{trace}_{\le t}=\epsilon\) if there is no such event. Event times need
not equal positions in the trace.

The event fields record provenance
\cite{MoreauMissier2013-PROV-DM,DBLP:conf/pods/GreenKT07,DBLP:conf/icdt/BunemanKT01,DBLP:conf/ccs/GargJD11}.
The predicates below use them to check current witnesses, authority, and
permitted effects at run time~\cite{DBLP:journals/jlp/LeuckerS09}.

\begin{foroma}
A trace for the scenario where Oma successfully submits her UK
passport renewal application by scanning her old passport to populate the form
is as follows:
{\small
\( \langle
  \texttt{UKRenewalPicked},\allowbreak
  \texttt{ObservePassport},\allowbreak
  \texttt{PopulateFields},\allowbreak
  \texttt{CheckPicture},\allowbreak
  \texttt{CheckForm},\allowbreak
  \texttt{ConfirmDraft},\allowbreak
  \texttt{CheckReady},\allowbreak
  \texttt{SubmissionPermit},\allowbreak
  \texttt{Submit}
  \rangle .\)}
The events' names are shorthands for the events, e.g.
\(\texttt{CheckReady}\) for the event
\(e_{\textit{ready}}\) above.  \end{foroma}

A read may extend the trace without changing the world; a write may do both.
An unobserved external change may affect the world alone. The context selected
for a call may omit or summarise trace events, so event specifications use the
trace directly to check witnesses and authority.
\vspace*{-3mm}
\subsection{Context}\label{sec:Context}
A context \(\textit{ctx}\in\textit{Lang}^{\ast}\) is the sequence of messages
selected for one call; it may summarise earlier messages and lose their exact
wording. At time \(t\),
\(\mathsf{Select}_t(s,\textit{trace},p)\) is the set of contexts that the
orchestrator may supply to producer \(p\) from state \(s\) and the current
trace:
\[
\mathsf{Select}_t:
\mathord{\textit{States}}_{\mathit{Orch}}\times\mathsf{Trace}\times
\mathsf{Producer}\longrightarrow\wp(\textit{Lang}^{\ast}).
\]
Selection may compress or rearrange messages according to the producer,
orchestrator state, and context-window budget. Define
\[
\mathsf{Select}_t(\textit{trace})
\triangleq
\bigcup_{s,p}\mathsf{Select}_t(s,\textit{trace},p).
\]
The event records the context selected for the call.

Fitting the component's window does not validate a selection. For a proposed
event, \(\mathsf{ReliableAssertions}\) and \(\mathsf{Required}\discretionary{-}{}{}\mathsf{ClaimsOK}\) must
still hold relative to the selected context
(Section~\ref{sec:controlled-soundness}). If compression or summarisation changes
a message's denotation or asserted claim, the result is a new message and is
checked as such. If no context within the budget preserves the claims required
by the event specification, no event using a selected context within that
budget can pass the acceptance checks. When selection is implemented by the
rule quotient of Section~\ref{sec:query-preserving}, the coarsest permitted
quotient preserves the designated signed claims for every input fact set.
Its optimality concerns the claim partition, not a token encoding or budget.
General text summarisation still requires the denotation and assertion
checks, and the independent acceptance conditions remain necessary.

\vspace*{-2mm}
\subsection{Prompts}

A system prompt records standing constraints; a user prompt records the
current request. Both enter the trace and may be selected into later contexts.

\section{Information State\label{sec:information-state-ai-system}}
The information state distinguishes accepted domain knowledge, source
material, and the messages available for one call:
\begin{itemize}[leftmargin=*,itemsep=0pt,topsep=2pt]

\item The universal knowledge base \(K_{U,t}\) gives the accepted domain
semantics relative to which reliance is judged. It need not be available to the
running system (Section~\ref{sec:universal-knowledge-base}).

\item The reference sources \(\textit{Src}_t\) consist of source pointers or
resources such as pages, files, and tool results.

\item The source-derived knowledge base \(K_{S,\textit{Src},t}\) contains facts
extracted from sources and their consequences. These may be wrong, incomplete,
stale, or inconsistent (Section~\ref{sec:source-derived-kb}).

\item The effective knowledge base \(K_{E,\textit{ctx},t}\) contains the
selected context messages and what its rules derive from them, including
mistaken or compressed material (Section~\ref{sec:effective-kb}).
\end{itemize}

Throughout the paper, when we use conceptual objects that exist in the
mathematics but may not appear directly in an AI system implementation, we
highlight them with an overline, e.g. \(\bar{U}\).

\subsection{Universal Knowledge Base \label{sec:universal-knowledge-base}}
Let \(\bar U\) be a fixed, tagged universe containing the representations used
by the application: claims, objects, proofs, programs, events, permissions,
functions, relations, sets, sequences, and traces. Which representations are
accepted as facts at time \(t\) is recorded by
\(\bar H_t\subseteq\bar U\); \(\bar U\) itself is fixed.

\begin{foroma}
Elements of \(\bar{U}\) include picture facts, form field values,
source claims, confirmation events, readiness facts, and claims about whether a
given app action is allowed.
\end{foroma}

Given rules \(\bar R_t\) and accepted hypotheses
\(\bar H_t\subseteq\bar U\), define
\(K_{U,t}\triangleq\langle\bar U,\bar H_t,\bar R_t,t\rangle\).
\begin{foroma}
 \(\bar{H}_t\) may include current domain facts such as the applicable passport
renewal rules, the meaning of form fields, and the validity conditions for a
submission, while \(\bar{R}_t\) contains rules for deriving claims such as
whether a draft is complete, whether a picture satisfies the applicable
requirements, and whether an application may be submitted.
\end{foroma} 

\subsection{Source-derived Knowledge Base \label{sec:source-derived-kb}}

The system represents only a scope
\(\overline{\textit{Scope}}\subseteq\bar U\), and its representation may also
merge distinct elements within that scope. We model restriction and merging as
ignorance and forgetfulness abstractions, using Galois connections from
abstract interpretation
\cite{DBLP:conf/ldrs/CousotC77,DBLP:conf/popl/CousotC79}.

\subsubsection{Ignorance and Forgetfulness Abstractions}

\paragraph{Ignorance.}

The \(\overline{\textit{ignore}}\) abstraction
ignores what falls outside this scope: \(\overline{\textit{ignore}}(X)=X\cap \overline{\textit{Scope}}\) for \(X\in\wp(\bar{U})\).
The same restriction applies to rules:
$\overline{\textit{ignore}}(\bar{R}_t)
\triangleq
\Bigl\{
\frac{P}{c}\in \bar{R}_t
\Bigm|
P\subseteq \overline{\textit{Scope}}\wedge c\in \overline{\textit{Scope}}\,
\Bigr\}\).
Thus \(\overline{\textit{ignore}}(\bar{R}_t)\) is sound but incomplete relative to
$\bar{R}_t$: any proof in \(\overline{\textit{Scope}}\) is also a proof in \(\bar{U}\), but some
proofs in \(\bar{U}\) disappear because some facts, premises, or rules are
outside the system's scope.

\begin{foroma}
For example, a UK-renewal app may omit the rules for first applications or
renewals in another country.
\end{foroma}

\paragraph{Forgetfulness.}
The \(\overline{\textit{forget}}\) abstraction formalises application-chosen
loss of distinctions inside the system's scope. We write \(\equiv_f\) for the
application-specific equivalence relation on \(\overline{\textit{Scope}}\) that
identifies elements represented by the same item,
and \([x]_{\equiv_f}\) for the equivalence class represented by \(x\).
We assume that \(\overline{\textit{Scope}}\) is closed under the
negation \(\neg\) used by the statuses of Section~\ref{sec:paracons}, that
this negation is involutive (\(\neg\neg x=x\)), and that \(\equiv_f\) is a
congruence for it (\(x\equiv_f y\Leftrightarrow\neg x\equiv_f\neg y\)). Then
\(\neg[x]_{\equiv_f}\triangleq[\neg x]_{\equiv_f}\) is well defined; without
this, the status of an abstract item is not determined. On sets,
$\overline{\textit{forget}}(X)=\{[x]_{\equiv_f}\mid x\in X\}\). On rules,
$\overline{\textit{forget}}(R)
\triangleq
\biggl\{
\frac{\{\,[p]_{\equiv_f}\;|\; p\in P\,\}}{\ustrut[c]_{\equiv_f}}
\biggm|
\frac{P}{c}\in R
\biggr\}\,\),
so abstract rules can no longer distinguish equivalent pieces of information.
On sets, \(\alpha_f\triangleq\overline{\textit{forget}}\) and
\(\gamma_f(Y)\triangleq\bigcup Y\) form a Galois connection between
\(\pair{\wp(\overline{\textit{Scope}})}{\subseteq}\) and
\(\pair{\wp(\overline{\textit{Scope}}/{\equiv_f})}{\subseteq}\):
\(\alpha_f(X)\subseteq Y\Leftrightarrow X\subseteq\gamma_f(Y)\).
An application first interprets prose as claims with explicit polarity. The
fixed partition models the distinctions retained by that representation;
overlapping interpretations of arbitrary prose need not form such a
partition. A message may denote a common equivalence class; if it has no
unique source item, denotation is undefined.
Equivalence is defined on claim and resource identifiers. Every identification
of \(x\) with \(y\) must be accompanied by the identification of \(\neg x\)
with \(\neg y\). The application must leave \(x\) and \(y\) in distinct classes
if identifying them would violate negation congruence or erase a distinction
used by an event specification. Stability tests preservation of all conclusion
classes; Section~\ref{sec:query-preserving} constructs a representation for
designated concrete queries.
\begin{foroma}
Forgetting is exemplified by two passport scans that the app reduces to the
same stored thumbnail.  One original scan preserves the digits needed by the
passport check, while the other is already too blurred to do so.  If
\(\equiv_f\) identifies the two scans, the abstract representation preserves
the fact that an image was received but cannot establish that the required
digits are readable. This distinction is needed by the later usability check,
so the scans must remain distinct.
\end{foroma}

\subsubsection{Source-derived Knowledge Base and its semantics}\label{Source-derived:Knowledge:Base}

For \(\textit{Src}\subseteq\textit{Src}_t\), the selected sources are not
themselves facts: source-item facts
$F_{S,\textit{Src},t}\subseteq\mathsf{SourceItems}\) (with
$\mathsf{SourceItems}\triangleq \overline{\textit{Scope}}/{\equiv_f}$) are
extracted from \(\textit{Src}\).  The corresponding rule system is
$R_{S,t}\triangleq
\overline{\textit{forget}}\,(\overline{\textit{ignore}}\,(\bar{R}_t))\). Thus
the source-derived knowledge base is \(K_{S,\textit{Src},t}\triangleq \langle
\mathsf{SourceItems}, F_{S,\textit{Src},t}, R_{S,t}, t \rangle\). Its inference
operator is \(\textsf{Infer}_{R_{S,t}}\).

Each source item is already an equivalence class
\(s=[x]_{\equiv_f}\subseteq\overline{\textit{Scope}}\subseteq\bar U\).
Its candidate claims are exactly its members. No further denotation map from
source items to claims is supplied.
\begin{foroma}
The universal base may distinguish facts about the full-resolution scans,
while the source base contains only the class represented by the stored
thumbnail.
\end{foroma}

\subsubsection{Soundness of forgetting}\label{sec:forget-soundness}
Merging facts also merges rule premises, so two rules that did not connect
before may connect after the merge. A rule system is stable when the merge adds
no conclusion classes, whatever the starting facts.
Throughout, \(R\) is a rule system over \(\overline{\textit{Scope}}\) (in the
application, \(R=\overline{\textit{ignore}}(\bar R_t)\)), and \(\alpha_f\),
\(\gamma_f\) are as in Section~\ref{sec:source-derived-kb}.

\begin{lemma}[Ignorance under-approximates]\label{lem:ignore-sound}
For every \(B\subseteq\overline{\textit{Scope}}\),
\(\textsf{Infer}^{\ast}_{\overline{\textit{ignore}}(\bar R_t)}(B)\subseteq
\textsf{Infer}^{\ast}_{\bar R_t}(B)\cap\overline{\textit{Scope}}\).
\end{lemma}
\noindent\emph{Informally.} Dropping rules that mention material outside the
chosen scope may lose conclusions, but it cannot add any. Every conclusion
left after the restriction was already justified by the original rules.

\begin{definition}[\(\equiv_f\)-stability]\label{def:stable}
\(R\) is \emph{\(\equiv_f\)-stable} when for every rule \(\frac{P}{c}\in R\) and
every \(P'\subseteq\overline{\textit{Scope}}\) such that each \(p\in P\) has
some \(p'\in P'\) with \(p'\equiv_f p\), there is \(c'\equiv_f c\) with
\(c'\in\textsf{Infer}^{\ast}_{R}(P')\).
\end{definition}
\noindent\emph{Informally.} Replace each premise of a rule by a claim that the
representation treats as the same. The original rules must then derive a claim
that the representation treats as the same as the rule's conclusion.

\begin{theorem}[Deciding stability on a finite instance]
\label{th:finite-stability}
Suppose that the scope and rule set are finite, every rule has finitely many
premises, and equality and \(\equiv_f\) are decidable. Then
\(\equiv_f\)-stability is decidable. For each rule \(P/c\), enumerate every
choice of one \(\equiv_f\)-equivalent replacement for each member of \(P\),
compute the rule closure of the chosen replacements, and test whether that
closure contains a member of \([c]_{\equiv_f}\). The rule system is stable if
and only if every such test succeeds.
\end{theorem}
If \(a\) is the maximum number of premises in a rule and \(k\) is the maximum
number of in-scope members of a premise's equivalence class, the procedure
performs at most
\[
  \sum_{P/c\in R}\prod_{p\in P}|[p]_{\equiv_f}|
  \ \leq\ |R|k^a
\]
closure computations. It is therefore exponential in rule arity in the worst
case, in addition to the cost of each finite closure. For bounded \(a\) and
\(k\), the number of closures is linear in the number of rules.

An implementation can stop at the first counterexample and reuse closures for
repeated premise sets, but the worst-case bound remains. The procedure assumes
a supplied finite rule representation; it does not extract one from
natural-language guidance.

\begin{theorem}[Forgetfulness transfer]\label{th:forget-transfer}

\begin{enumerate}[label=(\roman*),leftmargin=*]
\item For every \(B\subseteq\overline{\textit{Scope}}\), \(\alpha_f(\textsf{Infer}^{\ast}_{R}(B))\subseteq
\textsf{Infer}^{\ast}_{\overline{\textit{forget}}(R)}(\alpha_f(B))\);
\item Equality \(\textsf{Infer}^{\ast}_{\overline{\textit{forget}}(R)}(\alpha_f(B))
=\alpha_f(\textsf{Infer}^{\ast}_{R}(B))\) holds for all \(B\) if and only if \(R\)
is \(\equiv_f\)-stable.
\end{enumerate}
\end{theorem}
\noindent\emph{Informally.} After claims are merged, every original conclusion
still appears through its equivalence class, but new classes may also become
derivable. Equality with the original closure holds exactly under stability.
In abstract interpretation this equality is called completeness
\cite{DBLP:conf/popl/CousotC79,DBLP:journals/jacm/GiacobazziRS00}. The theorem concerns classes;
it guarantees a derivable member of each derived class, not every member.
\(\mathsf{Asserts}\) identifies the member a message asserts, whose support
still requires the separate \(\mathsf{SourceSupport}\) check.
The picture example of Section~\ref{sec:symbolic-components} violates
stability.

Suppose extraction merges claims once and storage or context selection merges
the resulting classes again. The following result gives a sufficient test for
these two quotient maps. Let \(S=\overline{\textit{Scope}}\), let
\({\equiv_1}\subseteq{\equiv_2}\) be equivalence relations on \(S\), and put
\(Q_i=S/{\equiv_i}\). Write
\(\pi_i:S\rightarrow Q_i\) for the quotient maps and \(\alpha_i\),
\(\overline{\textit{forget}}_i\) for the induced maps on sets and rules. Since
\({\equiv_1}\subseteq{\equiv_2}\), the map
\[
q:Q_1\longrightarrow Q_2,
\qquad q([x]_{\equiv_1})=[x]_{\equiv_2},
\]
is well defined and satisfies \(q\circ\pi_1=\pi_2\). Its kernel is the
equivalence \(a\approx b\triangleq q(a)=q(b)\) on \(Q_1\), equivalently \([x]_{{\equiv_1}}\approx[y]_{{\equiv_1}}\iff x\equiv_2 y\). For
\(Y\subseteq Q_1\) and a rule system \(R_1\) over \(Q_1\), define
\[
\alpha_q(Y)\triangleq\{q(a)\mid a\in Y\},
\qquad
\overline{\textit{forget}}_q(R_1)
\triangleq
\left\{
\frac{\alpha_q(P)}{q(c)}\;\middle|\;
\frac{P}{c}\in R_1
\right\}.
\]

\begin{theorem}[Successive forgetfulness]
\label{th:successive-forgetfulness}
\begin{enumerate}[label=(\alph*),leftmargin=*]
\item For every rule system \(R\) over \(S\) and every \(B\subseteq S\),
\begin{equation}
\overline{\textit{forget}}_q
  (\overline{\textit{forget}}_1(R))
=\overline{\textit{forget}}_2(R)
\quad\textrm{and}\quad
\alpha_q(\alpha_1(B))=\alpha_2(B).
\label{eq:successive-quotients}
\end{equation}

\item If \(R\) is \(\equiv_1\)-stable then
\(\overline{\textit{forget}}_1(R)\) is \(\mathord\approx\)-stable if and only if
\(R\) is \(\equiv_2\)-stable.

\item If \(R\) is \(\equiv_1\)-stable and
\(\overline{\textit{forget}}_1(R)\) is \(\mathord\approx\)-stable, then
for every \(B\subseteq S\),
\begin{equation}
\textsf{Infer}^{\ast}_{\overline{\textit{forget}}_2(R)}(\alpha_2(B))
=\alpha_q\!\left(
  \textsf{Infer}^{\ast}_{\overline{\textit{forget}}_1(R)}(\alpha_1(B))
 \right)
=\alpha_2(\textsf{Infer}^{\ast}_{R}(B)).
\label{eq:successive-transfer}
\end{equation}

\item\quad The first stability premise alone in (c) is not sufficient since \(R\) can be \(\equiv_1$-stable while \(\overline{\textit{forget}}_1(R)\) is not \(\approx\)-stable.

\item\quad The second stability premise alone in (c) is not sufficient since \(\overline{\textit{forget}}_1(R)\) can be \(\approx\)-stable while \(R\) is not \(\equiv_2\)-stable.

\item\quad \(R\) can be \(\equiv_2$-stable without being \(\equiv_1$-stable.
\end{enumerate}
\end{theorem}
\noindent\emph{Informally.} 
%If the original rules pass the first stability test
%and the once-merged rules pass the second, running the rules and then applying
%both quotient maps gives the same result as quotienting the rules in two stages
%and then running them. This is a sufficient test for the two-stage construction;
%passing the first test alone does not show that the second passes.
Assume the first merge is safe. Then the second merge is safe exactly when the
once-merged rules pass the second stability test. Under the finite hypotheses
of Theorem~\ref{th:finite-stability}, \(\equiv_2\)-stability can therefore be
decided from \(\overline{\textit{forget}}_1(R)\) alone, without access to \(R\).
For general rule systems, the result gives an equivalence. This matters when
the second stage (storage, context selection) only sees the once-merged rules.
The two-stage test as a whole is sufficient but not necessary, because a
coarser merge can be safe even when a finer one is not.

%
% The following counter example has been incorporated in the theorem.
%
%The first stability premise alone is not sufficient. Take all the claims below
%and their negations, with rules
%\[
%\mathsf{observed}_1\longrightarrow\mathsf{candidate}_1,
%\qquad
%\mathsf{observed}_2\longrightarrow\mathsf{candidate}_2,
%\qquad
%\mathsf{checked}\longrightarrow\mathsf{usable}.
%\]
%Let \(\equiv_1\) merge the two \(\mathsf{observed}\) claims and the two
%\(\mathsf{candidate}\) claims, and make the same two identifications between
%their negations. Leave every other class as a singleton. The rules are
%\(\equiv_1\)-stable: the first two rules match under either representative, and
%the third uses singleton classes. Let \(\equiv_2\) additionally merge both
%\(\mathsf{candidate}\) claims with \(\mathsf{checked}\), and likewise merge
%their three negations, with no other new identifications.
%
%The once-merged rule
%\([\mathsf{checked}]_{\equiv_1}\rightarrow
%[\mathsf{usable}]_{\equiv_1}\) witnesses failure of second-stage stability.
%Its premise is \(\approx\)-equivalent to
%\([\mathsf{candidate}_1]_{\equiv_1}\), but the first-stage rules derive no
%class \(\approx\)-equivalent to \([\mathsf{usable}]_{\equiv_1}\) from that
%candidate class. After the second merge, however, the rules derive
%\([\mathsf{usable}]_{\equiv_2}\) from
%\([\mathsf{observed}_1]_{\equiv_2}\); the original rules do not derive
%\(\mathsf{usable}\) from \(\mathsf{observed}_1\).
%

The statuses of Section~\ref{sec:paracons} transfer as follows. Let
\(K=\langle\overline{\textit{Scope}},F,R,t\rangle\) and
\(K^\sharp=\langle\overline{\textit{Scope}}/{\equiv_f},\alpha_f(F),\allowbreak
\overline{\textit{forget}}(R),t\rangle\), with negation on classes as above.

\begin{corollary}[Status transfer]\label{cor:status-transfer}
\begin{enumerate}[label=(\alph*),leftmargin=*]
\item For every \(x\in\overline{\textit{Scope}}\), if
\(K^\sharp\vdash^{?}[x]_{\equiv_f}\) then \(K\vdash^{?}y\) for every
\(y\equiv_f x\);
\item For every \(x\in\overline{\textit{Scope}}\), if
\(K^\sharp\vdash^{+}[x]_{\equiv_f}\) then \(K\not\vdash\neg y\) for
every \(y\equiv_f x\); if moreover \(R\) is \(\equiv_f\)-stable, then
\(K\vdash^{+}y\) for some \(y\equiv_f x\); dually for \(-\);
\item For every \(x\in\overline{\textit{Scope}}\), if
\(K\vdash x\) then
\(K^\sharp\vdash[x]_{\equiv_f}\); hence \(K\vdash^{+}x\) implies
\(K^\sharp\vdash^{+}[x]_{\equiv_f}\) or
\(K^\sharp\vdash^{\pm}[x]_{\equiv_f}\);
\item Quotienting can introduce status \(\pm\): there exist a
scope \(S_0\), an equivalence \(\equiv_0\) on \(S_0\) that is a congruence for
negation, \(K_0=\langle S_0,F_0,R_0,t\rangle\) with \(R_0\)
\(\equiv_0\)-stable, and \(y_1\equiv_0 y_2\), such that \(K_0\) justifies no
claim together with its negation, \(K_0\vdash^{+}y_1\), and
\(K_0\vdash^{-}y_2\), while the quotient abstraction \(K_0^\sharp\) satisfies
\(K_0^\sharp\vdash^{\pm}[y_1]_{\equiv_0}\).
\end{enumerate}
\end{corollary}
\noindent\emph{Informally.} The four parts say the following. If a merged class
is unknown, every claim in it was unknown before the merge. If it is positive,
none of its members had its negation justified; under stability, at least one
member had positive status. Any concrete conclusion survives the merge, although it may become
contradictory there. In particular, merging a positive claim with a different
negative claim can create a contradiction that did not exist before.

Combining Lemma~\ref{lem:ignore-sound} with Theorem~\ref{th:forget-transfer}
gives a soundness condition for the source-derived rules.

\begin{corollary}[Soundness of \(R_{S,t}\)]\label{cor:rs-sound}
If \(\overline{\textit{ignore}}(\bar R_t)\) is \(\equiv_f\)-stable, then for
every \(B\subseteq\overline{\textit{Scope}}\),
\(\textsf{Infer}^{\ast}_{R_{S,t}}(\alpha_f(B))\subseteq
\alpha_f(\textsf{Infer}^{\ast}_{\bar R_t}(B)\cap\overline{\textit{Scope}})\).
Consequently, every class in
\(\textsf{Infer}^{\ast}_{R_{S,t}}(\alpha_f(B))\) contains a member of
\(\textsf{Infer}^{\ast}_{\bar R_t}(B)\).
Without stability this inclusion can fail; \(R_{S,t}\) is then in general
neither sound nor complete with respect to \(\bar R_t\), since
\(\overline{\textit{ignore}}\) under-approximates while
\(\overline{\textit{forget}}\) over-approximates.
\end{corollary}
\noindent\emph{Informally.} Start with concrete premises \(B\), then merge the
distinctions that the source representation has lost. If the rules are stable,
every conclusion reached by the source rules represents at least one
conclusion reached by the original rules from \(B\). Without stability, the
merge may create a rule chain that the original rules do not have.

\begin{corollary}[Sound source facts]\label{cor:sound-source-facts}
Suppose that \(\overline{\textit{ignore}}(\bar R_t)\) is
\(\equiv_f\)-stable and that
\[
F_{S,\textit{Src},t}\subseteq
\alpha_f\bigl(\textsf{Infer}^{\ast}_{\bar R_t}(\bar H_t)
\cap\overline{\textit{Scope}}\bigr).
\]
Then every conclusion of \(K_{S,\textit{Src},t}\) contains at least one claim
justified by \(K_{U,t}\).
\end{corollary}
\noindent\emph{Informally.} Stability checks the rules; the inclusion premise also
checks the extracted facts. If each source fact represents something justified
by the accepted facts and rules, every later source conclusion does too.

\input{query-preservation}
\input{query-complexity}

\subsection{Effective knowledge base \label{sec:effective-kb}}

The effective knowledge base is
$K_{E,\textit{ctx},t}\triangleq \langle
\textit{Lang},F_{E,\textit{ctx},t},R_{E,\textit{ctx},t},t\rangle\).
The orchestrator selects \(\textit{ctx}\) from its state and the trace, and sets
\(F_{E,\textit{ctx},t}\triangleq\{m\mid m\in\textit{ctx\/}\}\).
Context is a message sequence; the knowledge base combines its messages with
rules.
Every context message is therefore justified, which establishes availability,
while assertion and claim support require separate checks.
If \textit{Lang} had no explicit involutive negation \(\neg\), for example
\(\texttt{result} = \texttt{failed}\) instead of
\(\neg(\texttt{result} = \texttt{passed})\),
\(K_{E,\textit{ctx},t}\vdash^{+}m\) would degenerate to
\(K_{E,\textit{ctx},t}\vdash m\).

Sources \(\textit{Src}_t\) supply extracted facts for
\(K_{S,\textit{Src},t}\); retrieved or generated messages selected into context
supply \(K_{E,\textit{ctx},t}\). Their separation lets support checks compare
messages with their stated sources
\cite{DBLP:conf/nips/LewisPPPKGKLYR020,DBLP:journals/corr/abs-2212-08037,DBLP:journals/corr/abs-2310-00935}.
The trace records events and provenance. Context selection may omit or
summarise those events, so validity and witness checks use the trace directly.

\begin{foroma} ``Oma confirmed the current draft and no later edit
invalidated it'' requires checking the trace even if a confirmation message
remains in context.
\end{foroma}

A context window is the bounded part of \(\textit{ctx}\) supplied in one call;
bounds may concern tokens, bytes, attachments, retrieved material, latency,
cost, or policy.

\subsection{Knowledge Base Statuses}\label{sec:KnowledgeBaseStatuses}

The same four statuses apply directly to all three bases:
\(K_{U,t}\vdash^{\times}_t c\) for \(c\in\bar U\),
\(K_{S,\textit{Src},t}\vdash^{\times}_t s\) for
\(s\in\mathsf{SourceItems}\), and
\(K_{E,\textit{ctx},t}\vdash^{\times}_t m\) for
\(m\in\textit{Lang}\), with \(\times\in\{+,-,\pm,?\}\).

\subsection{Relationships between Knowledge Bases \label{sec:ai-system-knowledge}}

The three bases may disagree or be incomplete or inconsistent. In particular,
\(K_{U,t}\) is an application-supplied semantic parameter that the running
system need not possess. Exact judgments are conditional on its facts and
rules. These may be disputed in natural-language domains, or precise in
formally specified languages and controlled English
\cite{DBLP:journals/coling/Kuhn14,DBLP:conf/rweb/FuchsKK08}; even precise rules
may have undecidable or expensive consequences. The application may instead
supply justified lower and upper bounds. The framework constructs neither the
exact base nor its bounds.

\paragraph{Bounding the universal knowledge base.}
\begin{definition}[Justification preorder]\label{def:kb-preorder}
For knowledge bases over the same domain and time,
\(\langle D,F,R,t\rangle\sqsubseteq\langle D,F',R',t\rangle\) iff
\(\textsf{Infer}^{\ast}_{R}(F)\subseteq\textsf{Infer}^{\ast}_{R'}(F')\),
i.e.\ every claim justified by the first is justified by the second.
\end{definition}
When \(K_{U,t}\) cannot be represented exactly, let
\(K_{\mathrm{lo}}\sqsubseteq K_{U,t}\sqsubseteq K_{\mathrm{hi}}\).
A conclusion derived by \(K_{\mathrm{lo}}\) is justified by \(K_{U,t}\); a
conclusion not derived by \(K_{\mathrm{hi}}\) is not justified by \(K_{U,t}\).
The application must establish both bound relations for the intended domain.
Cases they do not settle remain unresolved.

\begin{theorem}[Interval statuses]\label{th:interval-status}
Let \(K_{\mathrm{lo}}\sqsubseteq K_{\mathrm{hi}}\) (definition~\ref{def:kb-preorder}) and
\([K_{\mathrm{lo}},K_{\mathrm{hi}}]\triangleq\{K\mid K_{\mathrm{lo}}\sqsubseteq
K\sqsubseteq K_{\mathrm{hi}}\}\). For every claim \(c\):
\vspace*{-0.5ex}%
\[
\begin{array}{@{}l@{{}\iff{}}l@{\qquad\quad}l@{{}\iff{}}l@{}}
\forall K\in[K_{\mathrm{lo}},K_{\mathrm{hi}}].\ K\vdash^{+}c
 & 
 K_{\mathrm{lo}}\vdash c\wedge K_{\mathrm{hi}}\not\vdash\neg c
 &
\forall K\in[K_{\mathrm{lo}},K_{\mathrm{hi}}].\ K\vdash^{-}c
 & 
 K_{\mathrm{lo}}\vdash\neg c\wedge K_{\mathrm{hi}}\not\vdash c
 \\[0.5ex]
\forall K\in[K_{\mathrm{lo}},K_{\mathrm{hi}}].\ K\vdash^{\pm}c
 & 
 K_{\mathrm{lo}}\vdash c\wedge K_{\mathrm{lo}}\vdash\neg c
&
\forall K\in[K_{\mathrm{lo}},K_{\mathrm{hi}}].\ K\vdash^{?}c
 & 
 K_{\mathrm{hi}}\not\vdash c\wedge K_{\mathrm{hi}}\not\vdash\neg c
\end{array}
\]
\end{theorem}
\noindent\emph{Informally.} Testing the presence and absence of \(c\) and
\(\neg c\) at the endpoints determines when every base in the interval gives
the same status. A failed test leaves the status unresolved. If validated
rules describe the complete domain, both bounds may equal \(K_{U,t}\).

\begin{foroma}
\(K_{U,t}\) specifies the current normative passport-renewal process,
including route, form, picture, and authority requirements. Official pages are
\(\textit{Src}_t\); parsed guidance and OCR/MRZ output are
\(F_{S,\textit{Src},t}\), whose source-derived closure is
\(K_{S,\textit{Src},t}\).
\end{foroma}

A source pointer \(\mathit{sp}\in\textit{Src}_t\) in a trace witness field
can witness retrieval or citation. It does not by itself establish that the
source was current, official, complete, authoritative, or correctly interpreted.
Constructing \(K_{S,\textit{Src},t}\) must justify the interpretation;
\(\mathsf{SupportedClaim}\) then checks the current source-derived and
universal statuses and \(\mathsf{SourceSupport}\).

\section{Claims and Reliability\label{sec:candidate-to-reliable-claims}}

\begin{figure}[t]
\centering
\includegraphics[width=\linewidth]{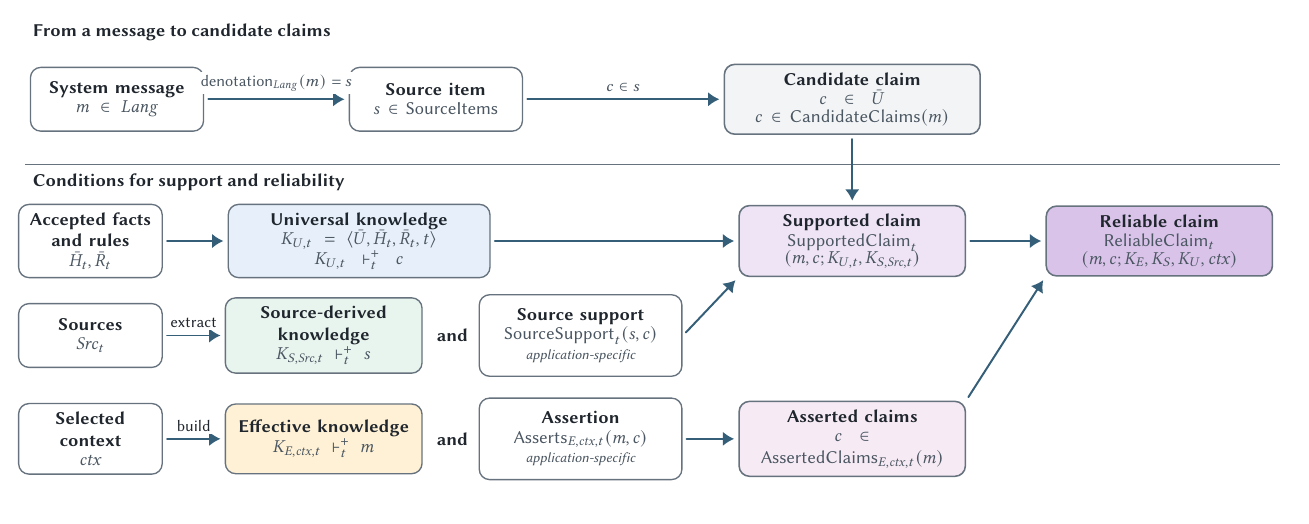}
\vspace*{-3em}%
\caption{From a message to a reliable claim. Support uses
source-derived and universal knowledge; reliability additionally requires
assertion and positive effective justification.}
\label{fig:reliable-claim-flow}
\Description{Messages denote source items, which determine candidate claims.
Source-derived and universal justification establish support; assertion and
effective justification then establish reliability.}
\vspace*{-1em}
\end{figure}
The judgments distinguish what a message may mean, what it asserts, whether
sources and accepted domain knowledge support the assertion, and whether the
assertion is reliable in the selected context. They are separate checks at the
current time. Candidates depend on the fixed denotation and quotient;
assertion, support, and reliability may change with the context, sources, and
knowledge bases.

\subsection{Candidate Claims}

Each source item is already a set of concrete claims. The candidate map is
therefore
\[
\mathsf{CandidateClaims}:\textit{Lang}\to\wp(\bar U),\qquad
\mathsf{CandidateClaims}(m)\triangleq
\begin{cases}
\mathsf{denotation}_{\textit{Lang}}(m),
  &\text{if }\mathsf{denotation}_{\textit{Lang}}(m){\downarrow},\\
\emptyset,&\text{otherwise}.
\end{cases}
\]

An empty candidate-claim set does not make an event valid.
Section~\ref{sec:controlled-soundness} requires every event message other than
the distinguished claimless message to have defined denotation and at least
one assertion.

\begin{foroma}
In the app, the denotation schema maps message \(m_{\textit{ready}}\) to a source
item 
$s_{\textit{readinessChecksPassed}}=
\{\mathsf{ReadinessChecksPassed}(\texttt{currentDraft})\}\). Hence
\(\mathsf{CandidateClaims}(m_{\textit{ready}})=
s_{\textit{readinessChecksPassed}}\). This message can assert only that the app's
local readiness check passed for the current draft. It cannot assert
$\mathsf{UserPermissionGiven}(\texttt{Submission})$; permission requires its
own message or event.
\end{foroma}

\subsection{Asserted Claims}\label{sec:asserted-claims}

The map \(\mathsf{CandidateClaims}\) records what a message could mean.  To say
which of those candidate claims the system is actually asserting through the
message, we use an application-specific predicate
\(\mathsf{Asserts}_{E,\textit{ctx},t}(m,c)\), such that
$\mathsf{Asserts}_{E,\textit{ctx},t}(m,c) \Rightarrow m\in\textit{Lang}\wedge
c\in\mathsf{CandidateClaims}(m)\).

The asserted claims of \(m\) are then:
\[\mathsf{AssertedClaims}_{E,\textit{ctx},t}(m)\triangleq
\{\,c\in\bar{U}\mid \mathsf{Asserts}_{E,\textit{ctx},t}(m,c)\,\}.\]
An assertion asks a user or component to rely on a candidate claim. It remains
an assertion when the effective status is refuted, contradictory, or unknown;
the reliability check then fails. The application supplies
\(\mathsf{Asserts}\), for example from typed output fields, event
interpretation, or checked extraction.

\subsection{Supported Claims}\label{sec:supported-claims}

Write \(\mathsf{SourceSupport}_t\subseteq
\mathsf{SourceItems}\times\bar U\).  The application-defined fact
\(\mathsf{SourceSupport}_t(s,c)\) means that source item \(s\) supports claim
\(c\) at time \(t\).  It must respect denotation:
$\mathsf{SourceSupport}_t(s,c) \Rightarrow
s\in\mathsf{SourceItems}\wedge c\in s\).

An application-specific extractor or checked derivation licenses a pair
\((s,c)\). The signature is independent of the selected source set, so the
application must justify that license uniformly for every admissible source
selection in which the item has positive status and the license can be used.
A license obtained from an unselected source cannot justify a claim merely
because selected sources produce the same item. Applications requiring
different licenses for indistinguishable source occurrences must retain
their distinction in the representation or withhold the pair.

For a singleton source item produced by a validated extractor, support of its
sole member reduces to the membership test above. For a non-singleton class,
membership is insufficient. For example, the license for the clear-scan claim
\(q_c\) in \(s_{\mathsf{thumb}}\) is usable only if the selected sources
license \(q_c\) whenever that item can be used to support it. If an admissible
selection containing only a blurred scan supplies the same item without that
license, the application must distinguish those occurrences or leave support
unresolved.

The application supplies \(\mathsf{SourceSupport}\) and must justify the
extractor or derivation that produces it. Corollary~\ref{cor:sound-source-facts}
has a separate premise requiring every extracted source fact to represent a
claim justified by \(K_{U,t}\). Positive source status, witness freshness, and
authority remain separate checks. A source update may change
\(\textit{Src}_t\), \(F_{S,\textit{Src},t}\), and the time-indexed support
relation.

The predicate
\[
  \begin{aligned}
  &\mathsf{SupportedClaim}_t
  (m,c;K_{U,t},K_{S,\textit{Src},t})\triangleq
  m\in\textit{Lang}
  \wedge c\in\mathsf{CandidateClaims}(m)
  \wedge
  \exists s\in\mathsf{SourceItems}\mathrel{.}
  \\
  &\qquad{}
  \mathsf{denotation}_{\textit{Lang}}(m)=s
  \wedge K_{U,t}\vdash^+_t c
  \wedge K_{S,\textit{Src},t}\vdash^+_t s
  \wedge \mathsf{SourceSupport}_t(s,c).
  \end{aligned}
\]
says that the claim and its source item have positive universal and source
statuses, respectively, and that the source item supports this particular
claim.

\subsection{Reliable Claims}

Reliability requires support, assertion, and positive effective status for the
message:
\[
\begin{aligned}
&\mathsf{ReliableClaim}_t(m,c;
K_{E,\textit{ctx},t},K_{S,\textit{Src},t},K_{U,t},\textit{ctx})
\triangleq\\
&\qquad
\mathsf{Asserts}_{E,\textit{ctx},t}(m,c)
\wedge K_{E,\textit{ctx},t}\vdash^{+}_t m
\wedge\mathsf{SupportedClaim}_t(m,c;K_{U,t},K_{S,\textit{Src},t}) .
\end{aligned}
\] 
The effective base checks the message; the other two bases check its asserted
claim and source item. The following instance calculates all three.

\subsection{A finite readiness instance}
\label{sec:finite-readiness-instance}

We exercise the definitions on a finite instance.  Set \(t_0=5\), let
\(D=\{d_0,d_1\}\) be successive drafts, and set
\(\mathcal P=\{\mathsf{PictureOK},\mathsf{FormComplete},\mathsf{Confirmed}\}\)
and \(\mathcal X=\mathcal P\cup\{\mathsf{Ready}\}\).  Let
\(\rho_{\mathsf{UK}}=\mathsf{UKRenewalPicked}\),
\(q_c=\mathsf{Usable}(d_0,p_c)\), \(q_b=\mathsf{Usable}(d_0,p_b)\), and
\(Q=\{q_c,q_b\}\), where the clear scan \(p_c\) establishes that the required
digits are readable but the blurred scan \(p_b\) does not.  Let
\(\bar U=\{X(d),\neg X(d)\mid X\in\mathcal X,
d\in D\}\cup Q\cup\{\neg q\mid q\in Q\}
\cup\{\rho_{\mathsf{UK}},\neg\rho_{\mathsf{UK}}\}\) and
\(\bar H_{t_0}=\{X(d_0)\mid X\in\mathcal P\}\cup\{q_c\}
\cup\{\rho_{\mathsf{UK}}\}\). Let \(\bar R\) consist
exactly of, for each \(d\in D\), the rule with premises
\(\{X(d)\mid X\in\mathcal P\}\) and conclusion \(\mathsf{Ready}(d)\).
Thus \(K_{U,t_0}=\langle\bar U,\bar H_{t_0},\bar R,t_0\rangle\)
justifies \(\mathsf{Ready}(d_0)\), but not its negation.

\paragraph{The three knowledge bases and four claim judgments.}
Take \(\overline{\textit{Scope}}=\bar U\).  Let \(\equiv_f\) be equality
except that it identifies \(q_c\) with \(q_b\), and \(\neg q_c\) with
\(\neg q_b\) so that it is a congruence for negation
(Section~\ref{sec:source-derived-kb}), because the app stores both scans
as the same thumbnail.  Write \(s_X(d)=[X(d)]_{\equiv_f}\) and
\(s_{\mathsf{thumb}}=[q_c]_{\equiv_f}=[q_b]_{\equiv_f}\).
By the definition above,
\(s_X(d)=\{X(d)\}\) and
\(s_{\mathsf{thumb}}=Q\).
If \(\bar H_{t_0}\) also contained \(\neg q_b\), and extraction retained both facts,
then
\(K_{S,\textit{Src},t_0}\vdash^{\pm}_{t_0}s_{\mathsf{thumb}}\)
would hold even though the concrete knowledge base justified no claim together
with its negation. Thus a contradictory source status in
Table~\ref{tab:taxonomy} can be introduced by the quotient rather than by a
concrete contradiction.
Set
\(R_S=\overline{\textit{forget}}
(\overline{\textit{ignore}}(\bar R))\).
Let \(\textit{Src}_{t_0}=\{g_{\mathsf{official}},p_0,f_0,c_0\}\) and
define the extracted facts for every selected
\(\Sigma\subseteq\textit{Src}_{t_0}\) by
\[
\begin{gathered}
F_{S,\Sigma,t_0}=\bigcup_{\sigma\in\Sigma}\mathcal E(\sigma),\qquad
\mathcal E(g_{\mathsf{official}})=\emptyset,\\[-0.5ex]
\mathcal E(p_0)=\{s_{\mathsf{PictureOK}}(d_0),s_{\mathsf{thumb}}\},\quad
\mathcal E(f_0)=\{s_{\mathsf{FormComplete}}(d_0)\},\quad
\mathcal E(c_0)=\{s_{\mathsf{Confirmed}}(d_0)\}.
\end{gathered}
\]
Here \(p_0\) is the validated clear-scan record. With all sources selected,
\(F_{S,t_0}=F_{S,\textit{Src}_{t_0},t_0}
=\{s_X(d_0)\mid X\in\mathcal P\}\cup\{s_{\mathsf{thumb}}\}\).
Finally, for every \(t\), set
\(\mathsf{SourceSupport}_{t}
=\{(s_X(d),X(d))\mid X\in\mathcal X,\ d\in D\}
\cup\{(s_{\mathsf{thumb}},q_c)\}\).
The four source symbols denote guidance, picture, form, and confirmation.
The singleton licenses follow from their unique members. Only \(p_0\)
extracts \(s_{\mathsf{thumb}}\), and no rule concludes that item, so its
license for \(q_c\) is valid across every selected source set in this
instance. Adding a blurred-scan source that extracts the same item would
require revising that license or retaining the occurrence distinction.
Let \(K_{S,t_0}\) have the source-derived components just defined.  Its
abstract rule gives \(K_{S,t_0}\vdash^+s_{\mathsf{Ready}}(d_0)\).
No rule of \(\bar R\) mentions \(q_c\) or \(q_b\), so \(\bar R\) is
\(\equiv_f\)-stable and, by Corollary~\ref{cor:rs-sound}, every consequence of
\(R_S\) from abstracted facts has a member justified by \(\bar R\) from the
corresponding concrete facts. The extracted facts also meet the premise of
Corollary~\ref{cor:sound-source-facts}: each singleton source fact
\(s_X(d_0)\) has representative \(X(d_0)\in\bar H_{t_0}\), and
\(s_{\mathsf{thumb}}\) has representative \(q_c\in\bar H_{t_0}\).

For \(X=\mathsf{PictureOK},\mathsf{FormComplete},\mathsf{Confirmed}\), let
\(m_X\) denote \(m_{\mathsf{pic}},m_{\mathsf{form}},m_{\mathsf{conf}}\), respectively.
The messages \(o_\bot\), \(m_\bot\), and
\(m_{\mathsf{noClaim}}\) are reserved below for no observation, no actuation,
and no asserted claim. Let \(M\) contain \(m_{\mathsf{pic}}\),
\(m_{\mathsf{form}}\), and \(m_{\mathsf{conf}}\), together with
\(m_{\mathsf{ready}}\), \(m'_{\mathsf{ready}}\), \(m_{\mathsf{thumb}}\),
\(m_{\mathsf{markReady}}\), \(m_{\mathsf{observeEdit}}\),
\(o_{\mathsf{edit}}\), \(o_\bot\),
\(m_{\mathsf{noClaim}}\),
the observation messages \(o_{\mathsf{pic}}\) and
\(o_{\mathsf{form}}\), and \(m_\bot\).
Let \(\textit{Lang}=M\cup\{\neg m\mid m\in M\}\).  Let \(\textit{ctx}_0\)
contain, in order, \(m_{\mathsf{pic}}\), \(m_{\mathsf{form}}\), and
\(m_{\mathsf{conf}}\), followed by \(m_{\mathsf{thumb}}\).
The first three context messages record passed results for \(d_0\).  Define
\(\mathsf{denotation}_{\textit{Lang}}(m_X)=s_X(d_0)\) for \(X\in\mathcal P\)
and \(\mathsf{denotation}_{\textit{Lang}}(m_{\mathsf{ready}})
=s_{\mathsf{Ready}}(d_0)\) and
\(\mathsf{denotation}_{\textit{Lang}}(m'_{\mathsf{ready}})
=s_{\mathsf{Ready}}(d_1)\), and let
\(\mathsf{denotation}_{\textit{Lang}}(m_{\mathsf{thumb}})=s_{\mathsf{thumb}}\).
The rule set \(R_E\) consists exactly of the two rules with premises
\(\{m_X\mid X\in\mathcal P\}\) and respective conclusions
\(m_{\mathsf{ready}}\) and \(m'_{\mathsf{ready}}\), and the two check rules
\(\frac{\{o_{\mathsf{pic}}\}}{m_{\mathsf{pic}}}\) and
\(\frac{\{o_{\mathsf{form}}\}}{m_{\mathsf{form}}}\): a picture or form
checker that reads the observed picture or form may return a passed
result.  The second rule \(\frac{\{m_X\mid X\in\mathcal P\}}{m'_{\mathsf{ready}}}\) is deliberately faulty: it ignores the draft index and derives a readiness message for \(d_1\) from checks on \(d_0\), as a component that conflates successive drafts might. It is absent from \(\bar R\), and the instance shows that the support check, not effective derivability, rejects its conclusion. The indexed assertion relations used here are
\[
\begin{aligned}
\mathsf{Asserts}_{E,\textit{ctx}_0,t_0}
  &=\{(m_{\mathsf{ready}},\mathsf{Ready}(d_0)),
       (m'_{\mathsf{ready}},\mathsf{Ready}(d_1)),
       (m_{\mathsf{thumb}},q_c)\},\quad
\mathsf{Asserts}_{E,\textit{ctx}_{\mathsf{pic}},1}
=\{(m_{\mathsf{pic}},\mathsf{PictureOK}(d_0))\},\\[-1ex]
\mathsf{Asserts}_{E,\textit{ctx}_{\mathsf{form}},2}
  &=\{(m_{\mathsf{form}},\mathsf{FormComplete}(d_0))\},
\end{aligned}
\]
where \(\textit{ctx}_{\mathsf{pic}}=\langle o_{\mathsf{pic}}\rangle\) and
\(\textit{ctx}_{\mathsf{form}}=\langle o_{\mathsf{form}}\rangle\).
Then \(K_{E,\textit{ctx}_0,t_0}=\langle\textit{Lang},
\{m_X\mid X\in\mathcal P\}\cup\{m_{\mathsf{thumb}}\},R_E,t_0\rangle\)
justifies both readiness messages, but neither negation. Direct substitution
in the definitions gives:

\begin{itemize}[leftmargin=*,itemsep=0pt,topsep=2pt]
\item \(\mathsf{CandidateClaims}(m_{\mathsf{ready}})
=\mathsf{AssertedClaims}_{E,\textit{ctx}_0,t_0}(m_{\mathsf{ready}})
=\{\mathsf{Ready}(d_0)\}\).
\item \(\mathsf{SupportedClaim}_{t_0}
(m_{\mathsf{ready}},\mathsf{Ready}(d_0);K_{U,t_0},K_{S,t_0})\).
\item \(\mathsf{ReliableClaim}_{t_0}
(m_{\mathsf{ready}},\mathsf{Ready}(d_0);
K_{E,\textit{ctx}_0,t_0},K_{S,t_0},K_{U,t_0},\textit{ctx}_0)\).
\item \(\mathsf{CandidateClaims}(m'_{\mathsf{ready}})
=\mathsf{AssertedClaims}_{E,\textit{ctx}_0,t_0}(m'_{\mathsf{ready}})
=\{\mathsf{Ready}(d_1)\}\), but both \(K_{U,t_0}\) and \(K_{S,t_0}\)
give the corresponding claim and source item unknown status. Hence
\(\mathsf{diag}_{t_0,\textit{ctx}_0}
(m'_{\mathsf{ready}},\mathsf{Ready}(d_1);\textit{Src}_{t_0})
=\textsf{UnsupportedAssertion}\), and the claim is not reliable.
\end{itemize}
For \(m_{\mathsf{thumb}}\), the candidates are \(Q\).
\(\mathsf{Asserts}\) records whether the message asserts \(q_c\) or
\(q_b\). If it records \(q_c\), the claim has source support; if
it records \(q_b\), source support fails and \(K_{U,t_0}\vdash^?q_b\), so
Section~\ref{sec:mismatches} classifies it as an unsupported assertion.  The
framework evaluates either choice rather than selecting one itself.

Section~\ref{sec:finite-readiness-execution} completes this instance after the
execution, specification, and safeguard definitions.

\section{Reliance Failures\label{sec:mismatches}}

A reliance failure is an asserted claim that fails \(\mathsf{ReliableClaim}\).
Current justification statuses identify the failed check. The provenance
diagnoses in Appendix~\ref{sec:history-diagnostics} explain some failures using
source and derivation records.

\subsection{Reliance failure diagnostics}

A reliance failure occurs when the system asserts a claim that does not meet
the reliability conditions. For a selected source set
\(\Sigma\subseteq\bigcup_u\textit{Src}_u\), write
\(K_{S,\Sigma,t}\) for the source-derived knowledge base of
Section~\ref{sec:source-derived-kb} whose selected sources are exactly
\(\Sigma\), and define:
\vspace*{-1mm}%
\[
\mathsf{RelianceFailure}_{t,\textit{ctx}}(m,c;\Sigma)
\triangleq
\mathsf{Asserts}_{E,\textit{ctx},t}(m,c)
\wedge
\neg\mathsf{ReliableClaim}_t(m,c;
K_{E,\textit{ctx},t},K_{S,\Sigma,t},K_{U,t},\textit{ctx}).
\]
\vspace*{-1mm}%
The assertion contract gives a unique source item \(s\) denoted by \(m\).
Let \(\sigma_E\), \(\sigma_U\), and \(\sigma_S\) be, respectively, the
four-valued status of \(m\) in \(K_{E,\textit{ctx},t}\), of \(c\) in
\(K_{U,t}\), and of \(s\) in \(K_{S,\Sigma,t}\), and put
\(b=\mathsf{SourceSupport}_t(s,c)\). Table~\ref{tab:taxonomy} defines
\(\mathsf{diag}_{t,\textit{ctx}}(m,c;\Sigma)\). Its cases are disjoint and
cover every tuple \((\sigma_E,\sigma_U,\sigma_S,b)\). The labels describe
current justification states; the history diagnoses in
Appendix~\ref{sec:history-diagnostics} identify possible causes.
\begin{table}[h]
\caption{Diagnostic calculated for an asserted claim.\label{tab:taxonomy}}%
\small
\begin{tabularx}{\linewidth}{@{}lY@{}}
\toprule
Diagnostic & Condition\\
\midrule
\(\textsf{EffectiveFailure}(\sigma_E)\)
 & \(\sigma_E\neq+\)\\
\(\textsf{RefutedAssertion}\)
 & \(\sigma_E=+\wedge\sigma_U=-\)\\
\(\textsf{UnsupportedAssertion}\)
 & \(\sigma_E=+\wedge\sigma_U=?\)\\
\(\textsf{SourceUniversalMismatch}(s)\)
 & \(\sigma_E=+\wedge\sigma_U=\pm\wedge\sigma_S=+\)\\
\(\textsf{RefutedSourceContradictoryClaim}(s)\)
 & \(\sigma_E=+\wedge\sigma_U=\pm\wedge\sigma_S=-\)\\
\(\textsf{UnsettledSource}(s)\)
 & \(\sigma_E=+\wedge\sigma_U=\pm\wedge\sigma_S\in\{\pm,?\}\)\\
\(\textsf{SourceStatusFailure}(s,\sigma_S)\)
 & \(\sigma_E=+\wedge\sigma_U=+\wedge\sigma_S\neq+\)\\
\(\textsf{SourceSupportGap}(s)\)
 & \(\sigma_E=\sigma_U=\sigma_S=+\wedge\neg b\)\\
\(\textsf{OK}\)
 & \(\sigma_E=\sigma_U=\sigma_S=+\wedge b\)\\
\bottomrule
\end{tabularx}
\end{table}
\begin{theorem}[Reliance diagnostic normal form]
\label{th:exhaustiveness}
For every \(\Sigma\subseteq\bigcup_u\textit{Src}_u\), if
\(\mathsf{Asserts}_{E,\textit{ctx},t}(m,c)\), then
\[
\begin{aligned}
\mathsf{ReliableClaim}_t(m,c;
K_{E,\textit{ctx},t},K_{S,\Sigma,t},K_{U,t},\textit{ctx})
&\Longleftrightarrow
\mathsf{diag}_{t,\textit{ctx}}(m,c;\Sigma)=\textsf{OK},\\[-0.5ex]
\mathsf{RelianceFailure}_{t,\textit{ctx}}(m,c;\Sigma)
&\Longleftrightarrow
\mathsf{diag}_{t,\textit{ctx}}(m,c;\Sigma)\neq\textsf{OK}.
\end{aligned}
\]
\end{theorem}
\noindent\emph{Informally.} For an asserted claim, \(\mathsf{diag}\) returns
\(\textsf{OK}\) exactly when the effective, universal, and source statuses are
positive and the source item supports the claim. Every other tuple receives one
of the failure labels above.

\noindent Except for \(m_{\mathsf{noClaim}}\), a message with no denotation or no
assertion fails event validity before this calculation. By
Corollary~\ref{cor:status-transfer}(d), a contradictory source
status may itself have been introduced by forgetting distinctions.

\subsection{History failure diagnostics}

Source and derivation records can explain a current reliance failure: a source
may have become stale or refuted, a derivation may need an unsupported
hypothesis, or a later event may make an unsupported message available.
Appendix~\ref{sec:history-diagnostics} defines these diagnoses and proves that
each entails a reliance failure. They require extraction or derivation records
in addition to the current statuses. Witness freshness for a proposed action
is checked separately in Section~\ref{sec:event-specifications}.

\section{Components and Execution
\label{sec:compute-components}}

All components use one producer relation. Symbolic and neural components
transform messages; services can also observe or change the world. The
specialisations below describe their interfaces without requiring a particular
reasoning algorithm or neural architecture.

We use the inclusions into the disjoint union \(\mathsf{Producer}\)
defined in Section~\ref{sec:communication-language} without further notation.
The set \(\mathsf{External}\) contains users and environmental producers for
external actions that are recorded as events. An unrecorded world change is a
world step and has no producer. If a service later records that change, its
record is a system-controlled service call. Let
\(o_{\bot},m_{\bot}\in\textit{Lang}\) be the
distinguished no-observation and no-actuation messages, with
\(\mathsf{actuation}(w,m_{\bot})=w\).

Each \(p\in\mathsf{Producer}\) has local states
\(\mathsf{State}_p\). A stateless producer has the singleton state set
\(\{\star\}\). Its calls form a relation
\begin{equation}
\mathsf{trans}_p{}\subseteq{}
\mathsf{State}_p\times\mathbb N\times\mathsf{Ctx}\times
\wp\Bigl(\bigcup_t\textit{Src}_t\Bigr)\times\textit{Lang}
\times\mathsf{State}_p\times\textit{Lang}\times\textit{Lang}.
\label{eq:tag-C1}
\end{equation}
A tuple
\(
\langle s,t,\textit{ctx},\textit{Src},o,s',m',a\rangle
\in\mathsf{trans}_p
\)
means that \(p\), in state \(s\), is called at time \(t\), with context
\(\textit{ctx}\), sources \(\textit{Src}\),
and observation \(o\). It moves to \(s'\), returns \(m'\), and requests
actuation \(a\). Every transition satisfies
\(\textit{Src}\subseteq\textit{Src}_t\).

\subsection{Symbolic Components}\label{sec:symbolic-components}

A symbolic component derives conclusions from facts and rules, for example
with an SMT solver~\cite{DBLP:conf/tacas/MouraB08}. We compare its inference
relation with the source-derived closure \(\textsf{Infer}^{\ast}_{R_{S,t}}\).
The engine may omit valid inferences or add invalid ones; scoped correctness
below states the required inclusion.

\subsubsection{Definition}
A symbolic component uses a language \(\textit{ModLang}\) that abstracts
$\mathsf{SourceItems}\) (where \(\mathsf{SourceItems}\triangleq
\overline{\textit{Scope}}/{\equiv_f}$). Its intended semantics is
$\mathsf{sem}_{\textit{ModLang}}\in \textit{ModLang}\rightarrow
\mathsf{SourceItems}\).
To communicate with the rest of the AI system, the modelling language may
be embedded in \(\textit{Lang}$; otherwise the component needs communication
encoders and decoders \(\mathsf{enc}_{\textit{ModLang}}\in
\mathsf{Ctx}\rightarrow \wp(\textit{ModLang})\) and
$\mathsf{dec}_{\textit{ModLang}}\in \textit{ModLang}\cup\{\bot\}\rightarrow
\textit{Lang}\).

\begin{foroma}
\(\textit{ModLang}\) may encode picture and form claims;
\(\mathsf{sem}_{\textit{ModLang}}\) maps each encoded claim to a source item;
\(\mathsf{enc}_{\textit{ModLang}}\) turns the selected context into checker
inputs; and \(\mathsf{dec}_{\textit{ModLang}}\) returns the result, failure, or
timeout as a system message.
\end{foroma}

For a black-box symbolic component, we can represent the inference engine as
$\textit{inference}_{\textit{ModLang}}\in
\wp(\wp(\textit{ModLang})\times(\textit{ModLang}\cup\{\bot\}))\), where \(\bot$
represents timeout, failure, or absence of an answer. This relation abstracts
search over \(R_{S,t}\). The scoped-correctness condition below checks it
against those rules; without that condition it may omit valid inferences or add
invalid ones.
For a symbolic component
\(S=\langle\textit{ModLang},\mathsf{sem}_{\textit{ModLang}},
\mathsf{enc}_{\textit{ModLang}},\allowbreak \mathsf{dec}_{\textit{ModLang}},\allowbreak 
\textit{inference}_{\textit{ModLang}}\rangle\allowbreak  \in \allowbreak \textit{Symbolic}\), let
\(\mathsf{State}_S=\{\star\}\). Its producer relation is
\[
\begin{aligned}
&\langle\star,t,\textit{ctx},\textit{Src},o,
\star,m',a\rangle\in\mathsf{trans}_S
\\[-1ex]
&\quad\Longleftrightarrow\quad
o=o_{\bot}
\wedge a=m_{\bot}
\wedge \textit{Src}\subseteq\textit{Src}_t
\wedge
\exists H,r\mathrel{.}\,
H=\mathsf{enc}_{\textit{ModLang}}(\textit{ctx})
\wedge      
\langle H,r\rangle\in\textit{inference}_{\textit{ModLang}}
{}\wedge{} m'=\mathsf{dec}_{\textit{ModLang}}(r).
\end{aligned}
\]
Scoped correctness concerns the internal result \(r\);
\(\mathsf{Validity}\) checks the returned message \(m'\).

\subsubsection{Scoped correctness}

A symbolic component is correct within the system's scope
at time \(t\) if each
non-\(\bot\) result it returns corresponds, through
\(\mathsf{sem}_{\textit{ModLang}}\), to a consequence of the scoped underlying
rules:

  \[
  \begin{aligned}
  &\forall H_{\textit{ModLang}}\subseteq\textit{ModLang}\mathrel{.}\,
  \forall c\in\textit{ModLang}\mathrel{.}\,
  \bigl(
  \langle H_{\textit{ModLang}},c\rangle
  \in\textit{inference}_{\textit{ModLang}}
  \bigr)
  \\
  &\quad{}\Rightarrow
  \mathsf{sem}_{\textit{ModLang}}(c)
  \in
  \textsf{Infer}^{\ast}_{R_{S,t}}
  (\{\mathsf{sem}_{\textit{ModLang}}(h)\mid h\in H_{\textit{ModLang}}\}).
  \end{aligned}
  \]

Scoped correctness is relative to the merged rules \(R_{S,t}\). Under
stability, every output class contains a claim derivable by the original rules
from claims in the input classes.

\begin{corollary}[Scoped correctness transfers under stability]\label{cor:scoped}
Suppose \(\mathsf{sem}_{\textit{ModLang}}\) maps into
\(\mathsf{SourceItems}=\overline{\textit{Scope}}/{\equiv_f}\) and
\(\overline{\textit{ignore}}(\bar R_t)\) is \(\equiv_f\)-stable. Then for every
component that is scoped-correct at time \(t\) and every
\(\langle H_{\textit{ModLang}},c\rangle\in\textit{inference}_{\textit{ModLang}}\),
some member of \(\mathsf{sem}_{\textit{ModLang}}(c)\) belongs to
\(\textsf{Infer}^{\ast}_{\bar R_t}\bigl(\bigcup\{\mathsf{sem}_{\textit{ModLang}}(h)
\mid h\in H_{\textit{ModLang}}\}\bigr)\).
\end{corollary}

Without stability the conclusion can fail, as the following example shows.

\begin{foroma}
Consider the picture part of Oma's application. Rules distinguish observing a picture from checking it:
$R_{\textit{picture}}=
\left\{
\frac{\{c_{\textit{pictureObserved}}\}}{c_{\textit{pictureCandidate}}},
\frac{\{c_{\textit{pictureChecked}}\}}{c_{\textit{pictureUsable}}}
\right\}\).
Thus an observed picture does not imply its usability. Suppose a forgetfulness abstraction collapses
$c_{\textit{pictureCandidate}}\equiv_f c_{\textit{pictureChecked}}$
and
\(\neg c_{\textit{pictureCandidate}}\equiv_f
\neg c_{\textit{pictureChecked}}\), with no other identifications, because
the UI records both as the single status \texttt{picture present}, while
$c_{\textit{pictureObserved}}\) and
$c_{\textit{pictureUsable}}\) stay distinguished.  The abstract rules are
$\overline{\textit{forget}}(R_{\textit{picture}})= \left\{
\frac{\{[c_{\textit{pictureObserved}}]_{\equiv_f}\}}{[c_{\textit{pictureCandidate}}]_{\equiv_f}},
\frac{\{[c_{\textit{pictureChecked}}]_{\equiv_f}\}}{[c_{\textit{pictureUsable}}]_{\equiv_f}} \right\}\). Since 
$[c_{\textit{pictureCandidate}}]_{\equiv_f} =
[c_{\textit{pictureChecked}}]_{\equiv_f}\), we have:
$[c_{\textit{pictureObserved}}]_{\equiv_f} \Rightarrow
[c_{\textit{pictureCandidate}}]_{\equiv_f} =
[c_{\textit{pictureChecked}}]_{\equiv_f} \Rightarrow
[c_{\textit{pictureUsable}}]_{\equiv_f}\). After the merge, the rules derive
picture usability from mere observation; the original rules do not. The rule system
\(R_{\textit{picture}}\) is not \(\equiv_f\)-stable: the rule
\(\frac{\{c_{\textit{pictureChecked}}\}}{c_{\textit{pictureUsable}}}\) and
\(P'=\{c_{\textit{pictureCandidate}}\}\) meet its hypothesis, but
\(\textsf{Infer}^{\ast}_{R_{\textit{picture}}}(P')=\{c_{\textit{pictureCandidate}}\}\)
contains nothing equivalent to \(c_{\textit{pictureUsable}}\). A component that
returns \([c_{\textit{pictureUsable}}]_{\equiv_f}\) from
\([c_{\textit{pictureObserved}}]_{\equiv_f}\) is scoped-correct, yet
\(c_{\textit{pictureUsable}}\notin
\textsf{Infer}^{\ast}_{R_{\textit{picture}}}(\{c_{\textit{pictureObserved}}\})\).

For the representation decision, abbreviate the candidate, checked, and
usable claims by \(a,h,u\). Extend this same example with a second
check \(h'\), checker-specific certificates \(k,k'\), and their negations,
adding exactly
\[
 \{h'\}\longrightarrow u,\qquad
 \{h\}\longrightarrow k,\qquad
 \{h'\}\longrightarrow k'.
\]
The two checks establish the same usability claim. Suppose \(E_0\) permits
merging \(a,h,h'\), with the matching negative merge, and keeps all other
claims distinct. For \(Q=\{u,\neg u\}\),
Theorem~\ref{th:optimal-query-quotient} retains exactly the merge of \(h,h'\)
and their negatives. The empty context distinguishes \(a\) from either
check, whereas both checks give the same usability answers in every context.
Thus the UI may use one checked status but must distinguish a candidate
picture from a checked picture.

The retained merge is not globally stable: replacing premise \(h\) by
\(h'\) in \(h\to k\) derives \(k'\) and \(u\), but not the distinct
certificate \(k\). Nevertheless the designated usability status remains
exact. An event that also relies on a particular certificate must include
that signed claim in \(Q\), which separates the checks. This calculation
chooses the representation according to what the event relies on; it
does not establish the independent claim-support or acceptance conditions.
\end{foroma}

\subsection{Neural Components}

A neural component encodes the selected context into vectors, applies a
network, and decodes the result into messages in \(\textit{Lang}\). Large
language models specialise this interface to token sequences.

\subsubsection{Definition}
A neural component is a possibly set-valued message transformer. Write
\vspace*{-1ex}%
{\belowdisplayskip0pt\[
N=\langle\textit{encode}_{N\!N},N\!N,
\textit{decode}_{N\!N}\rangle\in\textit{Neural}.
\]}
\smash{Its possible outputs are}
\[
\begin{gathered}
\mathsf{NeuralOutputs}_{N}(\textit{ctx})\triangleq
\displaystyle\bigcup_{v\in N\!N(\textit{encode}_{N\!N}(\textit{ctx}))}
\textit{decode}_{N\!N}(v)
\subseteq\textit{Lang},\\[-0.5ex]
\textit{encode}_{N\!N}:\mathsf{Ctx}\rightarrow\mathbb{R}^n,\qquad
N\!N:\mathbb{R}^n\rightarrow\wp(\mathbb{R}^m),\qquad
\textit{decode}_{N\!N}:\mathbb{R}^m\rightarrow\wp(\textit{Lang}).
\end{gathered}
\]
\vspace*{-1ex}%

Information lost by the encoder is unavailable to the network. Set-valued
network and decoder maps allow numerical variation and sampled decoding;
singleton sets recover deterministic behaviour. A call returns one
\(m'\in\mathsf{NeuralOutputs}_{N}(\textit{ctx})\), which the orchestration rule below
records in the event trace.

Let \(\mathsf{State}_{N}=\{\star\}\). The neural producer relation is
\[
\begin{aligned}
&\langle\star,t,\textit{ctx},\textit{Src},o,
\star,m',a\rangle\in\mathsf{trans}_{N}
~~~\Longleftrightarrow~~~
o=o_{\bot}
\wedge a=m_{\bot}
\wedge \textit{Src}\subseteq\textit{Src}_t
\wedge
m'\in\mathsf{NeuralOutputs}_{N}(\textit{ctx}).
\end{aligned}
\]

\begin{foroma}
A neural encoder may convert the candidate passport picture into vectors,
the neural network may classify or extract information, and the decoder
may return a label, field value, or message in \(\textit{Lang}\).
\end{foroma}

Each returned message is subject to the same assertion and reliability checks
as a symbolic result. Verification of network properties~\cite{DBLP:journals/corr/abs-2104-02466}
does not establish those message-level conditions.

\subsection{Agent Services}

Services include tools, sensors, actuators, and external APIs. Their explicit
state and world-facing transitions distinguish them from the
message-transforming components.
\subsubsection{Definition}
An agent service \(\textit{Serv}\) is formalised by internal
states \(\textit{States}_{\textit{Serv}}\) and a transition relation
\(\mathsf{trans}_{\textit{Serv}}\) of type (\ref{eq:tag-C1}), with
\(\mathsf{State}_{\textit{Serv}}=\textit{States}_{\textit{Serv}}\).
A transition
\vspace*{-1ex}%
\[
\langle s,t,\textit{ctx},\textit{Src},o,s',m',a\rangle
\in\mathsf{trans}_{\textit{Serv}}
\]
means that the service is in state \(s\), reads the selected context and the
observed world message \(o\), changes its internal state to \(s'\), returns
\(m'\), and may request the actuation \(a\). In a closed state,
\(o=\mathsf{observation}(\textit{Serv},w)\) and
\(w'=\mathsf{actuation}(w,a)\), as required by (\ref{eq:tag-E0}).

\begin{foroma}
The camera, scanner, picture checker, form filler, and retrieval service use
this interface. Their individual contracts do not establish that readiness or
submission is permitted; those actions require the event checks below.
\end{foroma}

\subsection{Orchestration}
\label{sec:Neuro-SymbolicOrchestration}
\label{sec:agentic-harness}

The orchestrator selects producers and contexts, records results, and invokes
policy checks. A harness implements this role. Its labelled transition system
records component calls and distinguishes them from internal and unobserved
world steps.

Iteration, refinement, and recursion need no separate execution construct. A
loop is a sequence of transitions; a sub-agent is a producer, or a service when
it is world-facing; and an orchestration strategy chooses among enabled
transitions. Different strategies can therefore produce different traces,
contexts, and proposed events. The same acceptance predicates check each
proposal; the semantics does not choose a scheduling strategy.

\subsubsection{Definition}

A closed state
records the orchestrator state, every producer state, the world, and the trace:
\[
Z\triangleq
\mathord{\textit{States}}_{\mathit{Orch}}\times
\prod_{p\,\in\,\mathsf{Producer}}\mathsf{State}_p
\times W\times\mathsf{Trace}.
\]
The set \(Z_0\subseteq Z\) contains the initial states. Write
\(\pi_{\mathit{Orch}}(z)\), \(\pi_p(z)\), \(w(z)\), and \(\tau(z)\) for the
four projections.

For a relation \(B\) of type (\ref{eq:tag-C1}), define
\[
\begin{aligned}
\mathsf{Lift}_p(B)\triangleq
\{\, &
\langle p,t,\textit{ctx},\textit{Src},z,m',o,a,z'\rangle
\mid
 p\in\mathsf{Producer}\wedge z,z'\in Z
{}\wedge{}\langle\pi_p(z),t,\textit{ctx},\textit{Src},o,
       \pi_p(z'),m',a\rangle\in B\\
&{\qquad}\wedge{}(p\in\textit{Services}
\Longrightarrow
o=\mathsf{observation}(p,w(z)))
{}\wedge{}w(z')=\mathsf{actuation}(w(z),a){}{}\\
&{\qquad}\wedge{}
\forall q\in\mathsf{Producer}\setminus\{p\}\mathrel{.}\,
\pi_q(z')=\pi_q(z)
\,\}.
\end{aligned}
\]
\smash{Define}
\vspace*{-1.5ex}%
\begin{equation}
\mathsf{Result}\triangleq
\bigcup_{p\,\in\,\mathsf{Producer}}
\mathsf{Lift}_p(\mathsf{trans}_p).
\label{eq:tag-E0}
\end{equation}
\(\mathsf{Lift}_p\) places one producer's local transition in a
closed system state: the other producer states stay fixed and the recorded
actuation determines the new world. \(\mathsf{Result}\) collects these lifted
transitions for all producers.
No producer may change another producer's state. A symbolic or neural
component cannot observe or change the world. A service may observe the world
and may request a world update; it may change the world only through its
recorded actuation. The orchestrator state may change, and the execution
relation below appends the event to the trace.

An event agrees with a producer result when
\begin{equation}
\begin{aligned}
&\mathsf{InterfaceOK}_t    (z,e,z')\triangleq{}
\bigl(e.\mathsf{time}=t\bigr)\wedge{}\\
&\quad\langle e.\mathsf{producer},t,
e.\mathsf{context},e.\mathsf{sources},z, 
e.\mathsf{message},e.\mathsf{observation},e.\mathsf{actuation},z'\rangle
\in\mathsf{Result}. 
\label{eq:tag:E1}
\end{aligned}
\end{equation}
Thus the producer field identifies who returned the recorded
message. Membership in \(\mathsf{Result}\) records the producer's local state,
observation, actuation, and world update; the execution relation separately
records the orchestrator state change and appends the event to the trace.

Let \(Env\subseteq W\times W\) be the supplied environment
relation.
Execution is one relation
\begin{equation}
{\longrightarrow}\ \subseteq\
Z\times\mathbb N\times
(\textit{Ev}\uplus\{\mathsf{internal},\mathsf{world}\})\times Z,
\qquad z\xrightarrow[t]{\ell}z'.
\label{eq:tag:E2}
\end{equation}
(where \(\ell\) is a label). It has three labelled forms.
\begin{itemize}[leftmargin=*]
\item A recorded call \(z\xrightarrow[t]{e}z'\) records a selected context and
appends the event:
\vspace*{-1ex}%
\[
e.\mathsf{context}\in
\mathsf{Select}_t(\pi_{\mathit{Orch}}(z),\tau(z),e.\mathsf{producer}),
\qquad
\tau(z')=\tau(z)\cdot e.
\]

\vspace*{-1ex}%
\noindent The orchestrator state may change.\llstrut

\item An internal step \(z\xrightarrow[t]{\mathsf{internal}}z'\) changes only
the orchestrator state:
\vspace*{-1ex}%
\[
w(z')=w(z),\quad \tau(z')=\tau(z),\quad
\forall p.\ \pi_p(z')=\pi_p(z).
\]

\item An unrecorded world step \(z\xrightarrow[t]{\mathsf{world}}z'\) changes
only the world and must satisfy the environment relation:
\vspace*{-1ex}%
\[
(\,w(z),w(z')\,)\in Env,\quad \tau(z')=\tau(z),\quad
\pi_{\mathit{Orch}}(z')=\pi_{\mathit{Orch}}(z),\quad
\forall p.\ \pi_p(z')=\pi_p(z).
\]
\end{itemize}
If a service later observes the mutation, that service call appends an ordinary
event.

A run is
\vspace*{-1em}%
\begin{equation}
\begin{gathered}
z_0\xrightarrow[t_0]{\ell_0}z_1
\xrightarrow[t_1]{\ell_1}\cdots
\xrightarrow[t_{n-1}]{\ell_{n-1}}z_n,\\
z_0\in Z_0,\qquad t_0<\cdots<t_{n-1},\qquad
\forall k<n.\ \forall e\in\tau(z_k).\ e.\mathsf{time}<t_k.
\end{gathered}
\label{eq:tag:E3}
\end{equation}
At step \(k\), knowledge, assertion, source support, and context selection use
time \(t_k\). This time is not derived from trace length: internal and world
steps advance time without extending the trace.

An event is system controlled exactly when its producer is a
symbolic component, neural component, or service:
\begin{equation}
\mathsf{SystemControlled}(e)
\quad\Longleftrightarrow\quad
e.\mathsf{producer}\in
\textit{Symbolic}\uplus\textit{Neural}\uplus\textit{Services}.
\label{eq:tag:E4}
\end{equation}
\subsubsection{Orchestrator Invariants}
An invariant \(I\subseteq Z\) holds on every reachable state if
\(Z_0\subseteq I\) and
\[
z\in I\wedge z\xrightarrow[t]{\ell}z'\quad\Longrightarrow\quad z'\in I.
\]
This includes unrecorded world steps. Liveness is a property of
maximal or infinite runs of (\ref{eq:tag:E2}), such as weak fairness when running subagents or background tasks. A sound safeguard may reject every controlled
proposal, so safeguard soundness alone gives no liveness result.

\section{AI systems\label{sec:def-spec-soundness}}

In this framework, an AI system comprises the knowledge, prompts, components,
services, orchestration, and actuations shown in Figure~\ref{fig:AISystem}.
\begin{figure}
\centering
\includegraphics[width=.75\textwidth]{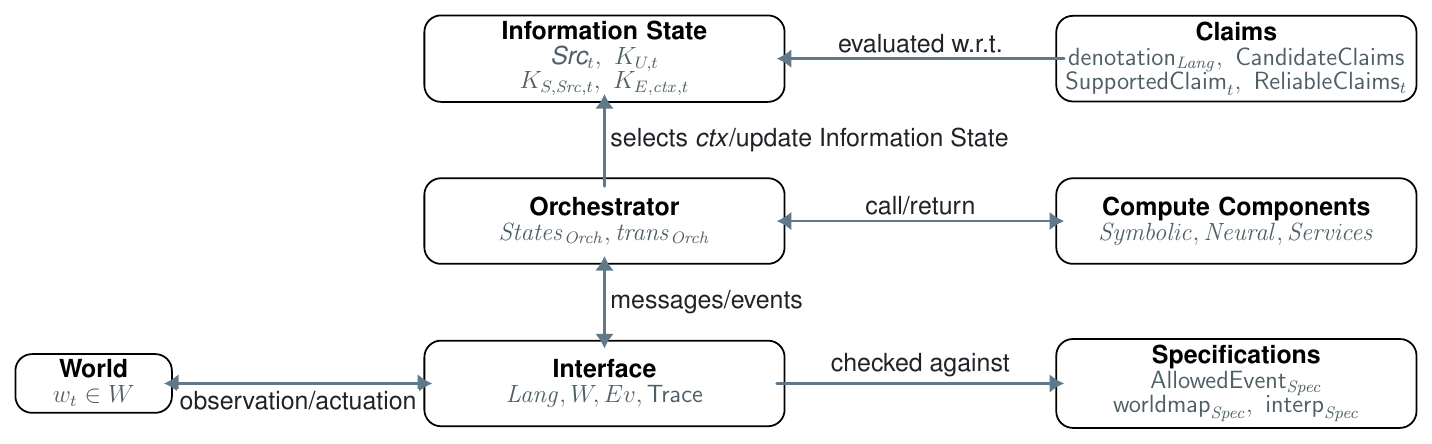}\\[-5mm]
\Description{The orchestrator selects context from the information state and
calls compute components. Messages and events pass through the interface,
whose observations and actuations connect to the world. Specifications check
the interface events; denotation and claim judgments relate the information
state to claims.}
{%\small
\[
\begin{gathered}
z\xrightarrow[t]{\ell}z',\quad
\ell\in\textit{Ev}\uplus\{\mathsf{internal},\mathsf{world}\},\quad%\\[-1mm]
z=
\langle s_{\textit{Orch}},\vec{s},w,\textit{trace}\rangle,
\quad
z'=
\langle s'_{\textit{Orch}},\vec{s}\,',w',\textit{trace}'\rangle.
\end{gathered}
\]
}
\vspace*{-5mm}
\caption{AI-system components and transition form.
The expression below the picture is the transition form defined
in Section~\ref{sec:Neuro-SymbolicOrchestration}. Iteration and refinement are
represented as sequences of these transitions. \label{fig:AISystem}}
\vspace*{-3mm}
\end{figure}
\paragraph{Application choices and derived judgments.}
The application supplies four groups of inputs:
\begin{itemize}[leftmargin=*,itemsep=0pt,topsep=2pt]
\item \emph{Knowledge and sources}: \(\bar U,\bar H_t,\bar R_t\),
\(\overline{\textit{Scope}}\), a supplied \(\equiv_f\) or the permitted
identifications and queries from which it is constructed, sources, and
extracted facts and their records. These determine which claims the domain
justifies and which distinctions the source representation loses.
\item \emph{Representation}: \(\textit{Lang}\),
\(\mathsf{denotation}_{\textit{Lang}}\), \(\mathsf{SourceSupport}_t\),
\(\mathsf{Select}_t\), \(R_E\), and
\(\mathsf{Asserts}_{E,\textit{ctx},t}\).
These relate messages to claims and sources;
producer behaviour alone does not establish meaning or support.
\item \emph{World and execution}: \(W\), \(\textit{Ev}\), observation,
actuation, component maps, producer relations, \(Env\), and orchestration.
These describe possible behaviour, including changes outside system control.
\item \emph{Specification}: the seven maps below and a safeguard.
These state which possible actions are permitted and how they are checked.
\end{itemize}
The construction calculates the coarsest permitted query-preserving quotient;
the definitions calculate the knowledge closures, claim judgments, and
acceptance conditions from these inputs.
Sections~\ref{sec:finite-readiness-instance} and
\ref{sec:finite-readiness-execution} assign finite values and calculate the
resulting verdicts.

\section{Specifications\label{sec:event-specifications}}

Specifications state which system-controlled events are permitted. Claim
reliability alone does not establish that an action has current evidence,
authority, or a permitted effect. The specification supplies those requirements.

Let \(\textit{trace}=e_1\ldots e_n\in\mathsf{Trace}\) be the event trace so far, let
$w_t\in W\) be the current world state, and let \(e\in \textit{Ev}\) be the next event.
Write \(\mathbb B=\{\mathsf{true},\mathsf{false}\}\). A
specification \(\textsf{Spec}\) supplies seven maps describing the requirements on the world, traces, events, and communication language of Section \ref{sec:objects}:
\begin{gather*}
\mathsf{interp}_{\textsf{Spec}}
: (W\uplus\textit{Ev})\rightharpoonup \wp(\bar U),\quad
\mathsf{Requires}_{\textsf{Spec}}
: \mathsf{Trace}\times W\to\wp(\bar U),\quad
\mathsf{NeedsWitness}_{\textsf{Spec}}
: \mathsf{Trace}\times W\times\textit{Ev}\to\wp(\bar U),\\[-0.5ex]
\mathsf{Invalidates}_{\textsf{Spec}}
: \mathsf{Trace}\times W\times\textit{Ev}^{2}\times\bar U\to\mathbb B,\quad
\mathsf{RequiredAuthority}_{\textsf{Spec}}
: \mathsf{Trace}\times W\times\textit{Ev}\to\wp(\textit{Ev}),
\\[-0.5ex]
\mathsf{ActuationClause}_{\textsf{Spec}}
: \mathsf{Trace}\times W\times\textit{Ev}\times\textit{Lang}\to\mathbb B,\quad
\mathsf{AllowedEvent}_{\textsf{Spec}}
: \mathsf{Trace}\times W\times\textit{Ev}\to\mathbb B.
\end{gather*}
We write \(\mathsf{interp}_{\textsf{Spec}}(w)\) when \(w\) is a
world state and \(\mathsf{interp}_{\textsf{Spec}}(e)\) when \(e\) is an event.
Here \(\mathsf{interp}_{\textsf{Spec}}(w)\) records the
semantic facts about world state \(w\) that matter to \(\textsf{Spec}\), while
\(\mathsf{interp}_{\textsf{Spec}}(e)\) records the semantic content of event
\(e\): what it claims, records, witnesses, authorises, observes, or changes.
Irrelevant events may have empty interpretations.
The set \(\mathsf{Requires}_{\textsf{Spec}}(\textit{trace},w_t)\) records the
task-relevant semantic claims that \(\textsf{Spec}\) requires at the current
trace and world state.
\(\mathsf{NeedsWitness}\) lists claims requiring current trace witnesses;
\(\mathsf{Invalidates}\) marks witnesses as stale;
\(\mathsf{RequiredAuthority}\) lists required authority events;
\(\mathsf{ActuationClause}\) tests the recorded actuation; and
\(\mathsf{AllowedEvent}\) tests permission for the proposed event.
An event can be invalidated only for a claim in its
interpretation:
\[
\mathsf{Invalidates}_{\textsf{Spec}}(\textit{trace},w,e_j,e_i,c)
\Rightarrow
\mathsf{interp}_{\textsf{Spec}}(e_i){{\downarrow}}
\wedge c\in\mathsf{interp}_{\textsf{Spec}}(e_i).
\]

\subsection{Witnesses.}
A witness may reference citations that can be resolved against a trace. The citations can be event references,
finite sets of trace positions, and source pointers:
\vspace*{-1ex}%
\[
\mathsf{Citation}\triangleq
\textit{Ev}\uplus\wp_{\mathrm{fin}}(\mathbb N)
\uplus\bigcup_t\textit{Src}_t;
\]
\vspace*{-1em}

\noindent Other objects in the witness field are not citations.  A citation \(x\) resolves to
position \(i\) of the trace when
\[
\begin{aligned}
\mathsf{Resolve}(x,i;\textit{trace}){}\triangleq{}&
i\in[1,|\textit{trace}|]
\wedge\Bigl(
\bigl(x=\textit{trace}[i]\bigr)
\vee
\bigl(\exists I\subseteq[1,|\textit{trace}|]\mathrel{.}\,
x=I\wedge i\in I\bigr)
\vee
\bigl(\exists t_0,\mathit{sp},s\mathrel{.}\,\\[-1ex]
&x=\mathit{sp}
\wedge \textit{trace}[i]\in\textit{trace}_{\le t_0}
\wedge \mathsf{SourceWitness}_{t_0}(\mathit{sp},s;\textit{trace})
\wedge \mathsf{ExtractedFrom}_{t_0}
(s,\mathit{sp};\textit{trace}[i])\bigr)
\Bigr),
\\[-0.5ex]
\mathsf{Cites}(e,i;\textit{trace}){}\triangleq{}&
\exists x\in e.\mathsf{witness}\cap\mathsf{Citation}\mathrel{.}\,
\mathsf{Resolve}(x,i;\textit{trace}).
\end{aligned}
\]
Every citation must resolve:
\[
\begin{aligned}
\mathsf{CitationsOK}(\textit{trace};e)
{}\triangleq{}&
\forall x\in e.\mathsf{witness}\cap\mathsf{Citation}\mathrel{.}\,
\exists i\mathrel{.}\,\mathsf{Resolve}(x,i;\textit{trace}),
\\[-0.5ex]
\mathsf{UnresolvedCitation}(\textit{trace};e,x)
{}\triangleq{}&
x\in e.\mathsf{witness}\cap\mathsf{Citation}
\wedge\neg\exists i\mathrel{.}\,
\mathsf{Resolve}(x,i;\textit{trace}).
\end{aligned}
\]
The event at position \(i\) witnesses \(c\) for \(e\) when
\[
\begin{aligned}
&\mathsf{TraceWitness}_{\textsf{Spec}}(i,e,c;\textit{trace})
\triangleq\begin{array}[t]{@{}l@{}}
i\in[1,|\textit{trace}|]
\wedge \mathsf{Cites}(e,i;\textit{trace})
{}
\wedge{} \mathsf{interp}_{\textsf{Spec}}(\textit{trace}[i]){\downarrow}
{}\wedge{} c\in\mathsf{interp}_{\textsf{Spec}}(\textit{trace}[i]).
\end{array}
\end{aligned}
\]
It is current when no later trace event invalidates it:
\[
\begin{aligned}
&\mathsf{CurrentTraceWitness}_{\textsf{Spec}}
(\textit{trace},w;i,e,c)
\triangleq
\mathsf{TraceWitness}_{\textsf{Spec}}(i,e,c;\textit{trace})
\\[-1mm]
&\quad\hspace{15mm}{}{}\wedge
\neg\exists j\in[i+1,|\textit{trace}|]\mathrel{.}\,
\mathsf{Invalidates}_{\textsf{Spec}}
(\textit{trace},w,\textit{trace}[j],\textit{trace}[i],c).
\end{aligned}
\]
The witness claims are required by \(\mathsf{NeedsWitness}\),
together with the claims interpreted from the witnesses cited by \(e\):
\[
\begin{aligned}
&\mathsf{WitnessClaims}_{\textsf{Spec}}(\textit{trace},w;e)
\triangleq\begin{array}[t]{@{}l@{}}
\mathsf{NeedsWitness}_{\textsf{Spec}}(\textit{trace},w,e)
{}\cup{}
\{c\in\bar U\mid
\exists i\mathrel{.}\,
\mathsf{TraceWitness}_{\textsf{Spec}}(i,e,c;\textit{trace})\}.
\end{array}
\end{aligned}
\]
Every witness claim must have a current trace witness:
\[
\begin{aligned}
&\mathsf{CurrentWitnessesOK}_{\textsf{Spec}}(\textit{trace},w;e)
\triangleq
\mathsf{CitationsOK}(\textit{trace};e)
\\[-1mm]
&\quad{}\wedge
\forall c\in
\mathsf{WitnessClaims}_{\textsf{Spec}}(\textit{trace},w;e)
\mathrel{.}\,
\exists i\in[1,|\textit{trace}|]\mathrel{.}\,
\mathsf{CurrentTraceWitness}_{\textsf{Spec}}
(\textit{trace},w;i,e,c).
\end{aligned}
\]
A witness claim is missing when it has no trace witness:
\[
\begin{aligned}
&\mathsf{MissingWitness}_{\textsf{Spec}}(\textit{trace},w;e,c)
\triangleq\begin{array}[t]{@{}l@{}}
c\in\mathsf{WitnessClaims}_{\textsf{Spec}}(\textit{trace},w;e)
\\[-1mm]
{}\wedge
\neg\exists i\in[1,|\textit{trace}|]\mathrel{.}\,
\mathsf{TraceWitness}_{\textsf{Spec}}(i,e,c;\textit{trace}).
\end{array}
\end{aligned}
\]
A witness claim is stale when one of its trace witnesses was later invalidated
and no current witness replaces it:
{\small
\[
\begin{aligned}
&\mathsf{StaleWitness}_{\textsf{Spec}}
(\textit{trace},w;e,c,e_j)
\triangleq
\\
&\quad{}
c\in\mathsf{WitnessClaims}_{\textsf{Spec}}(\textit{trace},w;e)
{}\wedge
\neg\exists\ell\in[1,|\textit{trace}|]\mathrel{.}\,
\mathsf{CurrentTraceWitness}_{\textsf{Spec}}
(\textit{trace},w;\ell,e,c)
\\
&\qquad{}\wedge
\exists i,j\in[1,|\textit{trace}|]\mathrel{.}\,
i<j
\wedge \textit{trace}[j]=e_j
\wedge
\mathsf{TraceWitness}_{\textsf{Spec}}(i,e,c;\textit{trace})
\wedge
\mathsf{Invalidates}_{\textsf{Spec}}
(\textit{trace},w,e_j,\textit{trace}[i],c).
\end{aligned}
\]
}
Consequently,
\vspace*{-1ex}%
\begin{equation}
\begin{aligned}
&\neg\mathsf{CurrentWitnessesOK}_{\textsf{Spec}}(\textit{trace},w;e)
\Longleftrightarrow
\exists x\in\mathsf{Citation}\mathrel{.}\,
\mathsf{UnresolvedCitation}(\textit{trace};e,x)
\\[-1mm]
&
\qquad{}\vee{}
\exists c\in\bar U\mathrel{.}\,
\mathsf{MissingWitness}_{\textsf{Spec}}(\textit{trace},w;e,c)
\vee
\exists c\in\bar U,e_j\in\textit{Ev}\mathrel{.}\,
\mathsf{StaleWitness}_{\textsf{Spec}}(\textit{trace},w;e,c,e_j).
\label{eq:tag:W}
\end{aligned}
\end{equation}
The witness check fails for one of three reasons: a citation cannot
be resolved, a claim that needs a witness has none, or a later event has made
its recorded witness stale.
For the forward direction, an unresolved citation gives the first alternative.
Otherwise, choose a witness claim with no current trace witness.  If it has no
trace witness, it is missing.  Otherwise, choose a trace witness for the claim;
because it is not current, a later event invalidates it, so
\(\mathsf{StaleWitness}\) holds.  Each alternative on the right contradicts
\(\mathsf{CurrentWitnessesOK}\).

\subsection{Authority.}
A specification may require authority events. Such an event is present if it
occurs in the trace or is recorded in an event's authority field:
\[
\begin{aligned}
&\mathsf{AuthorityPresent}(\textit{trace},a)
\triangleq
\exists i\in[1,|\textit{trace}|]\mathrel{.}\,
\bigl(\textit{trace}[i]=a
\vee a\in\textit{trace}[i].\mathsf{authority}\bigr),
\\
&\mathsf{AuthorityOK}_{\textsf{Spec}}(\textit{trace},w;e)
\triangleq
\forall a\in\mathsf{RequiredAuthority}_{\textsf{Spec}}
(\textit{trace},w,e)\mathrel{.}\,
\mathsf{AuthorityPresent}(\textit{trace},a),
\\
&\mathsf{MissingAuthority}_{\textsf{Spec}}(\textit{trace},w;e,a)
\triangleq
a\in\mathsf{RequiredAuthority}_{\textsf{Spec}}(\textit{trace},w,e)
\wedge{}\neg\mathsf{AuthorityPresent}(\textit{trace},a).
\end{aligned}
\]
Therefore
\vspace*{-1ex}%
\begin{equation}
\neg\mathsf{AuthorityOK}_{\textsf{Spec}}(\textit{trace},w;e)
\quad\Longleftrightarrow\quad
\exists a\in\textit{Ev}\mathrel{.}\,
\mathsf{MissingAuthority}_{\textsf{Spec}}(\textit{trace},w;e,a).
\label{eq:tag:(A)}
\end{equation}
The authority check fails if and only if at least one authority
required for the event is absent from the trace and authority records.

\subsection{Actuation.}
The successor must result from the recorded actuation, and the specification
must permit that actuation:
\[
\begin{aligned}
&\mathsf{ActuationOK}_{\textsf{Spec}}(\textit{trace},w;e,w')
\triangleq
w'=\mathsf{actuation}(w,e.\mathsf{actuation})
{}\wedge
\mathsf{ActuationClause}_{\textsf{Spec}}
(\textit{trace},w,e,e.\mathsf{actuation}),
\\
&\mathsf{IncorrectSuccessor}(w;e,w')
\triangleq
w'\neq\mathsf{actuation}(w,e.\mathsf{actuation}),
\\
&\mathsf{ForbiddenActuation}_{\textsf{Spec}}
(\textit{trace},w;e,w',r)
\triangleq\begin{array}[t]{@{}l@{}}
r=e.\mathsf{actuation}
\wedge w'=\mathsf{actuation}(w,r)
\\[-1mm]
{}\wedge
\neg\mathsf{ActuationClause}_{\textsf{Spec}}(\textit{trace},w,e,r).
\end{array}
\end{aligned}
\]
The failure condition can therefore be written as follows:
\begin{equation}
\begin{aligned}
&\neg\mathsf{ActuationOK}_{\textsf{Spec}}(\textit{trace},w;e,w')
\\[-1mm]
&\quad\Longleftrightarrow
\mathsf{IncorrectSuccessor}(w;e,w')
\vee
\exists r\in\textit{Lang}\mathrel{.}\,
\mathsf{ForbiddenActuation}_{\textsf{Spec}}
(\textit{trace},w;e,w',r).
\label{eq:tag:X}
\end{aligned}
\end{equation}
The actuation check fails if the proposed successor is not the result
of the recorded actuation, or if the specification forbids that actuation.

If the interpretation of \(w_t\), any prior event, or \(e\) is undefined,
\(\mathsf{Validity}\), and hence \(\mathsf{AcceptOK}\), fails.

\begin{foroma}
\(\mathsf{interp}_{\textsf{Spec}}(w)\) may contain the current form value and picture
state; \(\mathsf{interp}_{\textsf{Spec}}(e)\) may contain the claim that a source was checked, a checker
returned a result, or Oma confirmed the draft.
\end{foroma}

$\mathsf{AllowedEvent}_{\textsf{Spec}}\) records whether \textsf{Spec} permits
the proposed event. The framework requires
$\mathsf{AllowedEvent}_{\textsf{Spec}}\) only for system-controlled events.
Invariants may still constrain the orchestrator state after an external
mutation or user action is observed. If \(e\) is performed by the system then the
specification can require
$\mathsf{AllowedEvent}_{\textsf{Spec}}(\textit{trace},w_t,e)\).  We write
\(\mathsf{SystemControlled}(e)\) when \(e\) is performed by the system,
rather than by a user, external agent, or the world.  For reads, this can
express e.g. that the private information is not read, or the read belongs to
the user-specified boundaries.  For writes, this can express e.g. that private
information is not modified, or an actuation had the required authority.  

\begin{foroma}
Section~\ref{sec:finite-readiness-execution} gives a \texttt{cat} trace-shape check for this
example~\cite{amt14}. Proposition~\ref{prop:cat-adequacy} states the relation
interpretation used to compare that check with the semantic witness condition.
A write event \(e\) labelled \texttt{CheckReady} is allowed only if the trace
contains the selected route, a successful picture-check witness, a
form-complete witness, and Oma's confirmation of the draft. Moreover, no
invalidation may occur between those witnesses and the readiness event: editing
the form, changing the picture, the route, or revoking confirmation makes the
prior support stale.  This catches readiness checks without a picture check,
without form completion, without Oma's current draft confirmation, or after an
edit that invalidates the earlier confirmation.
\end{foroma}

\section{Soundness of Transition Control by Safeguards\label{sec:controlled-soundness}}

A safeguard checks each proposed controlled transition before it is admitted.
We state the contracts on event records, acceptance, and rejection, then show
how they combine with the application's policy.

\subsection{Safeguard conditions on events and orchestrator steps\label{sec:cond-evt-orch}}

For each recorded step, the event fields and state change must agree with the
producer result. We use the following judgments:
\begin{itemize}[leftmargin=*,itemsep=0pt,topsep=2pt]
\item \emph{Reliable assertions.}
Let \(m_{\mathsf{noClaim}}\in\textit{Lang}\) be the distinguished
message used when an event's message asserts no claim; it has no denotation and
no asserted claims.  Any other event message must have one denotation, assert
at least one claim, and make only reliable claims:
\[
\begin{aligned}
&\mathsf{ReliableAssertions}_{t,\textit{ctx}}(\textit{trace};e)
\triangleq
e.\mathsf{time}=t
\wedge e.\mathsf{context}=\textit{ctx}
\wedge \textit{ctx}\in\mathsf{Select}_t(\textit{trace})
\\[-1mm]
&\quad{}\wedge{}\Bigl(e.\mathsf{message}=m_{\mathsf{noClaim}}
\ \vee\ \bigl(
\exists s\in\mathsf{SourceItems}\mathrel{.}
\mathsf{denotation}_{\textit{Lang}}(e.\mathsf{message})=s
\\[-1mm]
&\qquad{}\wedge{}
\mathsf{AssertedClaims}_{E,\textit{ctx},t}(e.\mathsf{message})\neq\emptyset
\wedge
\forall c\in
\mathsf{AssertedClaims}_{E,\textit{ctx},t}(e.\mathsf{message})
\mathrel{.}\\[-1mm]
&\hspace{13em}\mathsf{ReliableClaim}_t(e.\mathsf{message},c;K_{E,\textit{ctx},t},
K_{S,e.\mathsf{sources},t},K_{U,t},\textit{ctx})
\bigr)\Bigr).
\end{aligned}
\]

\item \emph{Validity.}
\(\mathsf{Validity}\) requires defined specification interpretations and
reliable assertions in the recorded context. Policy permission remains a
separate check:
\[
\begin{aligned}
\mathsf{Validity}_{t,\textsf{Spec},\textit{ctx}}
(\textit{trace},w;e)
{}\triangleq{}\begin{array}[t]{@{}l@{}}
\mathsf{interp}_{\textsf{Spec}}(w){\downarrow}
{}\wedge{}
\bigl(\forall e_j\in\textit{trace}\mathrel{.}\,
\mathsf{interp}_{\textsf{Spec}}(e_j){\downarrow}\bigr)
\\
{}\wedge{}
\mathsf{interp}_{\textsf{Spec}}(e){\downarrow}
{}\wedge{}
\mathsf{ReliableAssertions}_{t,\textit{ctx}}(\textit{trace};e).
\end{array}
\end{aligned}
\]

\item \emph{Interface Agreement.}
The event record must match what its producer actually did. A
recorded orchestrator step is an event-labelled transition:
\[
\begin{aligned}
&\mathsf{Step}_{\textit{Orch},t}
(s,\vec{s},\textit{trace},w,e,
s',\vec{s}',w')
\triangleq
\langle s,\vec{s},w,\textit{trace}\rangle
\xrightarrow[t]{e}
\langle s',\vec{s}',w',\textit{trace}\cdot e\rangle .
\end{aligned}
\]

\(\mathsf{StepOK}\) requires this agreement for every recorded transition:
\[
\begin{aligned}
&\mathsf{StepOK}(\mathsf{Step}_{\textit{Orch}})
\triangleq
\forall s,\vec{s},s',\vec{s}',\textit{trace},w,w',e,t\mathrel{.}\\[-1mm]
&\qquad
\mathsf{Step}_{\textit{Orch},t}
(s,\vec{s},\textit{trace},w,e,s',\vec{s}',w')
\Rightarrow
\mathsf{InterfaceOK}_t(
\langle s,\vec{s},w,\textit{trace}\rangle,e,
\langle s',\vec{s}',w',\textit{trace}\cdot e\rangle).
\end{aligned}
\]
\end{itemize}

\subsection{Transition Safeguard}\label{sec:transition-safeguard}

A transition safeguard checks validity, required claims, current witnesses,
authority, and actuation. Before adding \(e\) to the trace, the orchestrator calls
$\mathsf{Safeguard}_{\textsf{Spec}}(\textit{trace},w,e,w',\textit{ctx})\), which
accepts or rejects \(e\).

Operationally, the safeguard is a runtime monitor or execution
enforcer~\cite{DBLP:journals/tissec/Schneider00,DBLP:journals/ijisec/LigattiBW05,DBLP:journals/corr/abs-2503-18666}.
\begin{itemize}[leftmargin=*,itemsep=0pt,topsep=2pt]
\item \emph{Acceptance.} The result \(\mathsf{Accept}(e,w')\) permits the guarded
transition that records \(e\) and takes successor \(w'\). We write:
\[
\begin{aligned}
\mathsf{SafeguardAccept}_{\textsf{Spec}}
(\textit{trace},w,e,w',\textit{ctx})
&\triangleq\bigl(\textit{ctx}=e.\mathsf{context}\bigr){}\wedge{}(\mathsf{Safeguard}_{\textsf{Spec}}(\textit{trace},w,e,w',\textit{ctx})
=\mathsf{Accept}(e,w')).
\end{aligned}
\]
 
\item \emph{Acceptance certificate.}
The safeguard must establish that every claim required by \(\textsf{Spec}\)
at the current trace and world has a reliable message:
{%\small
\[
\begin{aligned}
&\mathsf{RequiredClaimsOK}_{\textsf{Spec},t,\textit{ctx}}
(\textit{trace},w;e)
\triangleq
e.\mathsf{time}=t
\wedge e.\mathsf{context}=\textit{ctx}
\wedge
\textit{ctx}\in\mathsf{Select}_t(\textit{trace})\\[-0.5ex]
&\quad{}
{}\wedge{}
\forall c\in\mathsf{Requires}_{\textsf{Spec}}(\textit{trace},w)
\mathrel{.}\,\exists m_c\in\textit{Lang}\mathrel{.}\mathsf{ReliableClaim}_t(m_c,c;
K_{E,\textit{ctx},t},
K_{S,e.\mathsf{sources},t},K_{U,t},\textit{ctx}).
\end{aligned}
\]
}
Both \(\mathsf{Validity}\) and \(\mathsf{RequiredClaimsOK}\)
use the context recorded in \(e\); (\ref{eq:tag:E2}) requires that the orchestrator selected
that context for \(e.\mathsf{producer}\).
Combining this requirement with validity and the witness, authority, and
actuation checks gives the acceptance certificate:
\[
~\mathsf{AcceptOK}_{\textsf{Spec},t,\textit{ctx}}
(\textit{trace},w,e,w')
\triangleq\begin{array}[t]{@{}l@{}}
\mathsf{Validity}_{t,\textsf{Spec},\textit{ctx}}
(\textit{trace},w;e)\\
{}
\wedge
\mathsf{RequiredClaimsOK}_{\textsf{Spec},t,\textit{ctx}}
(\textit{trace},w;e)
\wedge
\mathsf{CurrentWitnessesOK}_{\textsf{Spec}}(\textit{trace},w;e)\\
{}
{}\wedge{}
\mathsf{AuthorityOK}_{\textsf{Spec}}(\textit{trace},w;e)
\wedge
\mathsf{ActuationOK}_{\textsf{Spec}}(\textit{trace},w;e,w') .
\end{array}
\]

\paragraph{Acceptance with bounded universal knowledge.}
For \(K_{\mathrm{lo}}\sqsubseteq K_{\mathrm{hi}}\), define
\[
\begin{aligned}
&\mathsf{AcceptOK}^{[\mathrm{lo},\mathrm{hi}]}_{
\textsf{Spec},t,\textit{ctx}}(\textit{trace},w,e,w')
\\[-1mm]
&\quad\triangleq
\mathsf{AcceptOK}_{\textsf{Spec},t,\textit{ctx}}
(\textit{trace},w,e,w')
\bigl[
K_{U,t}\vdash^+_t c
\leftarrow
c\in\bar U\wedge K_{\mathrm{lo}}\vdash c
\wedge K_{\mathrm{hi}}\not\vdash\neg c
\bigr]
\end{aligned}
\]
(where $\leftarrow$ denotes predicate substitution).
By Theorem~\ref{th:interval-status}, the substituted test holds exactly
when \(c\) has positive status in every
\(K\in[K_{\mathrm{lo}},K_{\mathrm{hi}}]\).
Only the universal knowledge base varies. Since the source rules are derived
from it, bounded acceptance ranges over those
\(K_{U,t}=\langle\bar U,H,R,t\rangle\) for which
\vspace*{-1.5ex}%
\[
R_{S,t}=
\overline{\textit{forget}}\,
(\overline{\textit{ignore}}\,(R)).
\]
The source facts, effective knowledge base, assertions, source support,
context selection, trace, and specification maps remain fixed. Write
\[
\mathcal K_{\mathrm{compatible}}
\triangleq
\left\{K_{U,t}=\langle\bar U,H,R,t\rangle
\in[K_{\mathrm{lo}},K_{\mathrm{hi}}]
\ \middle|\ 
R_{S,t}=\overline{\textit{forget}}
(\overline{\textit{ignore}}(R))\right\}.
\]
The application must supply at least one
\(K^{\circ}\in\mathcal K_{\mathrm{compatible}}\); otherwise the bounds say
nothing about a universal knowledge base that could have produced the fixed
source rules.

\begin{corollary}[Bounded acceptance]\label{cor:interval-accept}
Fix \(\textsf{Spec},t,\textit{ctx},\textit{trace},w,e,w'\) and all
parameters other than the universal knowledge base, and suppose that the
application supplies a witness
\(K^{\circ}\in\mathcal K_{\mathrm{compatible}}\). If the bounded certificate
holds, then the corresponding \(\mathsf{AcceptOK}\) judgment holds for every
\(K_{U,t}\in\mathcal K_{\mathrm{compatible}}\). In particular it holds for
the actual \(K_{U,t}\) whenever that knowledge base belongs to
\(\mathcal K_{\mathrm{compatible}}\). A safeguard is therefore accept-sound
for every member of this family if each acceptance satisfies the bounded
certificate with the same fixed parameters.
\end{corollary}
If the bounded certificate fails, it supplies no basis for acceptance and does
not determine the unresolved universal status. A safeguard may withhold permission in that
case, but the failed bound test alone does not justify a
\(\mathsf{RejectFailed}\) diagnosis of an exact acceptance check. The total,
reject-sound safeguard results below apply only when an actual failed check
can be established.

\item \emph{Acceptance adequacy and safeguard soundness.}
Acceptance adequacy states that the acceptance certificate implies that the
event is allowed:
\[
\begin{aligned}
&\mathsf{AcceptAdequate}_{\textsf{Spec}}
\triangleq
\forall \textit{trace},w,e,w'\mathrel{.}\,
\mathsf{AcceptOK}_{\textsf{Spec},e.\mathsf{time},e.\mathsf{context}}
(\textit{trace},w,e,w'){}\Rightarrow{}
\mathsf{AllowedEvent}_{\textsf{Spec}}(\textit{trace},w,e).
\end{aligned}
\]
The safeguard contract requires each acceptance to satisfy the certificate:
\[
\begin{aligned}
&\mathsf{SafeguardAcceptSound}_{\textsf{Spec}}
(\mathsf{Safeguard}_{\textsf{Spec}})
\triangleq
\forall \textit{trace},w,e,w',\textit{ctx}\mathrel{.}\,\\[-0.5ex]
&\quad
\bigl(\mathsf{Safeguard}_{\textsf{Spec}}
(\textit{trace},w,e,w',\textit{ctx})
=\mathsf{Accept}(e,w')\bigr)
{}\Rightarrow{}
\bigl(\textit{ctx}=e.\mathsf{context}
\wedge
\mathsf{AcceptOK}_{\textsf{Spec},e.\mathsf{time},e.\mathsf{context}}
(\textit{trace},w,e,w')\bigr) .
\end{aligned}
\]

\item \emph{Valid step.}
A valid step is a system-controlled orchestrator step whose proposed event has
been accepted by the safeguard:
\begin{flalign*}
&\hspace{1em}\mathsf{ValidStep}_{\textsf{Spec}}
(\mathsf{Step}_{\textit{Orch}},
\mathsf{Safeguard}_{\textsf{Spec}},
\textit{trace},w,w',e,\textit{ctx},
s,\vec{s},s',\vec{s}')
\triangleq &&
\\
&\hspace{1em}\quad
\mathsf{Step}_{\textit{Orch},e.\mathsf{time}}
(s,\vec{s},\textit{trace},w,e,
s',\vec{s}',w')
\wedge
\mathsf{SystemControlled}(e){}\wedge{}
\mathsf{SafeguardAccept}_{\textsf{Spec}}
(\textit{trace},w,e,w',\textit{ctx}). &&
\end{flalign*}
\end{itemize}

These contracts combine as follows:
\begin{theorem}[Valid steps for accepted controlled transitions]\label{th:valid-accepted-controlled-transitions}
For all \(\textit{trace}\), \(w\), \(w'\), \(e\), \(\textit{ctx}\),
$s\), \(\vec{s}\), \(s'\), \(\vec{s}'$:
\[
\begin{aligned}
&\mathsf{ValidStep}_{\textsf{Spec}}
\bigl(
\mathsf{Step}_{\textit{Orch}},
\mathsf{Safeguard}_{\textsf{Spec}},
\textit{trace},w,w',e,\textit{ctx},
s,\vec{s},s',\vec{s}'
\bigr)
\\[-0.5ex]
&\quad{}\wedge
\mathsf{SafeguardAcceptSound}_{\textsf{Spec}}
(\mathsf{Safeguard}_{\textsf{Spec}})
\wedge
\mathsf{AcceptAdequate}_{\textsf{Spec}}
\wedge
\mathsf{StepOK}(\mathsf{Step}_{\textit{Orch}})
\Rightarrow\\[-0.5ex]
&\qquad
\mathsf{AllowedEvent}_{\textsf{Spec}}(\textit{trace},w,e)
\wedge
\mathsf{InterfaceOK}_{e.\mathsf{time}}(
\langle s,\vec{s},w,\textit{trace}\rangle,e,
\langle s',\vec{s}',w',\textit{trace}\cdot e\rangle).
\end{aligned}
\]
\end{theorem}
The safeguard establishes \(\mathsf{AcceptOK}\); application adequacy connects
that check to policy permission; \(\mathsf{StepOK}\) establishes agreement
with the producer's actual result.

\begin{foroma}
In Section~\ref{sec:finite-readiness-execution}, the event policy requires the
same current-witness check as the safeguard, establishing acceptance adequacy.
The displayed producer and execution relations establish \(\mathsf{StepOK}\).
The accepted \texttt{CheckReady} event is therefore permitted and agrees with
the service call that produced it.
\end{foroma}

\subsection{Rejection diagnostics \label{sec:rej-diag}}

A reject-sound safeguard identifies a failed conjunct of \(\mathsf{AcceptOK}\):
\[
\mathsf{AcceptanceChecks}\triangleq{}
\{\mathsf{ValidityCheck},\mathsf{RequiredCheck},\mathsf{WitnessCheck},
\mathsf{AuthorityCheck},
\mathsf{ActuationCheck}\}.
\]
For \(k\in\mathsf{AcceptanceChecks}\), define
\[
\mathsf{FailedCheck}_{\textsf{Spec}}(k;\textit{trace},w,e,w')
{}\triangleq{}\begin{array}[t]{@{}l@{}}
\phantom{{}\vee{}}
\bigl(k=\mathsf{ValidityCheck}
\wedge
\neg\mathsf{Validity}_{e.\mathsf{time},\textsf{Spec},e.\mathsf{context}}
(\textit{trace},w;e)\bigr)
\\[-0.3ex]
{}\vee
\bigl(k=\mathsf{RequiredCheck}
\wedge
\neg\mathsf{RequiredClaimsOK}_{\textsf{Spec},e.\mathsf{time},e.\mathsf{context}}
(\textit{trace},w;e)\bigr)
\\[-0.3ex]
{}\vee
\bigl(k=\mathsf{WitnessCheck}
\wedge
\neg\mathsf{CurrentWitnessesOK}_{\textsf{Spec}}(\textit{trace},w;e)\bigr)
\\[-0.3ex]
{}\vee
\bigl(k=\mathsf{AuthorityCheck}
\wedge
\neg\mathsf{AuthorityOK}_{\textsf{Spec}}(\textit{trace},w;e)\bigr)
\\[-0.3ex]
{}\vee
\bigl(k=\mathsf{ActuationCheck}
\wedge
\neg\mathsf{ActuationOK}_{\textsf{Spec}}(\textit{trace},w;e,w')\bigr).
\end{array}
\]
These are exactly the ways in which the acceptance certificate can fail:
\begin{equation}
\begin{aligned}
&\neg\mathsf{AcceptOK}_{\textsf{Spec},e.\mathsf{time},e.\mathsf{context}}
(\textit{trace},w,e,w')
\Longleftrightarrow
\exists k\in\mathsf{AcceptanceChecks}\mathrel{.}\,
\mathsf{FailedCheck}_{\textsf{Spec}}(k;\textit{trace},w,e,w').
\end{aligned}
\label{eq:tag:F}
\end{equation}
When \(e.\mathsf{context}\in
\mathsf{Select}_{e.\mathsf{time}}(\textit{trace})\), the required claims that
fail are exactly
\[
\left\{\,c\in\mathsf{Requires}_{\textsf{Spec}}(\textit{trace},w)
\;\middle|\;
\begin{array}{l}
\neg\exists m_c\in\textit{Lang}\mathrel{.}\,
\mathsf{ReliableClaim}_{e.\mathsf{time}}(m_c,c;\\[-0.5ex]
\quad K_{E,e.\mathsf{context},e.\mathsf{time}},
K_{S,e.\mathsf{sources},e.\mathsf{time}},
K_{U,e.\mathsf{time}},e.\mathsf{context})
\end{array}
\,\right\}.
\]
The \(\mathsf{RequiredCheck}\) fails exactly when this set is nonempty; if
the context is not selected, both \(\mathsf{ValidityCheck}\) and
\(\mathsf{RequiredCheck}\) fail.
A rejection records the failed check together with the reason it failed.
Let \(\mathsf{Diag}\) be the tagged set
\[
\begin{aligned}
\mathsf{Diag}\triangleq{}&
\{\textsf{OK},\textsf{RefutedAssertion},\textsf{UnsupportedAssertion}\}
\uplus
\{\textsf{EffectiveFailure}(\sigma)\mid\sigma\in\{-,\pm,?\}\}
\\[-1ex]
&{}\uplus
\{\textsf{SourceUniversalMismatch}(s),
\textsf{UnsettledSource}(s)
\mid s\in\mathsf{SourceItems}\}
\\[-1ex]
&{}\uplus
\{\textsf{RefutedSourceContradictoryClaim}(s),
\textsf{SourceSupportGap}(s)
\mid s\in\mathsf{SourceItems}\}
\\[-1ex]
&{}\uplus
\{\textsf{SourceStatusFailure}(s,\sigma)
\mid s\in\mathsf{SourceItems},\ \sigma\in\{-,\pm,?\}\}.
\end{aligned}
\]
The set of reasons is
\[
\mathsf{ReasonValue}\triangleq{}\begin{array}[t]{@{}l@{}}
\phantom{{}\uplus{}}\{\mathsf{notCheckable}(x)\mid x\in W\cup\textit{Ev}\}\\[-0.3ex]
{}
\uplus
\{\mathsf{unselectedContext},\mathsf{noDenotation},\mathsf{noAssertion},
\mathsf{incorrectSuccessor}\}
\\[-0.3ex]
{}\uplus
\{\mathsf{unreliable}(c,d)\mid
c\in\bar U,\ d\in\mathsf{Diag}\setminus\{\textsf{OK}\}\}
\\[-0.3ex]
{}\uplus
\{\mathsf{missingClaim}(c),\mathsf{missing}(c)\mid c\in\bar U\}
\uplus
\{\mathsf{unresolved}(x)\mid x\in\mathsf{Citation}\}
\\[-0.3ex]
{}\uplus
\{\mathsf{stale}(c,e_j)\mid c\in\bar U,\ e_j\in\textit{Ev}\}
\uplus
\{\mathsf{missingAuthority}(a)\mid a\in\textit{Ev}\}
\\[-0.3ex]
{}\uplus
\{\mathsf{forbidden}(r)\mid r\in\textit{Lang}\}.
\end{array}
\]
\smash{Thus}
\[
\mathsf{Reason}_{\textsf{Spec}}
\subseteq
\mathsf{AcceptanceChecks}\times\mathsf{ReasonValue}\times
\mathsf{Trace}\times W\times\textit{Ev}\times W.
\]
The judgment \(\mathsf{Reason}_{\textsf{Spec}}
(k,\pi;\textit{trace},w,e,w')\), for
\(\pi\in\mathsf{ReasonValue}\), holds
in exactly the following cases, where \(t=e.\mathsf{time}\),
\(\textit{ctx}=e.\mathsf{context}\), and \(m=e.\mathsf{message}\):
{%\small
\[
\begin{array}{@{~}l@{\quad}l@{\quad}l@{~}}
k & \pi & \text{condition}\\\hline
\mathsf{ValidityCheck} & \mathsf{notCheckable}(x)
 & x\in\{w,e\}\cup\textit{trace}\ \wedge\ \mathsf{interp}_{\textsf{Spec}}(x){\uparrow}\\[-0.5ex]
\mathsf{ValidityCheck},\ \mathsf{RequiredCheck} & \mathsf{unselectedContext}
 & \textit{ctx}\notin\mathsf{Select}_t(\textit{trace})\\[-0.5ex]
\mathsf{ValidityCheck} & \mathsf{noDenotation}
 & m\neq m_{\mathsf{noClaim}}\ \wedge\
   \neg\exists s\in\mathsf{SourceItems}.\,
   \mathsf{denotation}_{\textit{Lang}}(m)=s\\[-0.5ex]
\mathsf{ValidityCheck} & \mathsf{noAssertion}
 & m\neq m_{\mathsf{noClaim}}\ {}\wedge{}\
   \mathsf{AssertedClaims}_{E,\textit{ctx},t}(m)=\emptyset\\[-0.5ex]
\mathsf{ValidityCheck} & \mathsf{unreliable}(c,d)
 & \mathsf{Asserts}_{E,\textit{ctx},t}(m,c)\ \wedge\
   d=\mathsf{diag}_{t,\textit{ctx}}(m,c;e.\mathsf{sources})
   \neq\textsf{OK}\\[-0.5ex]
\mathsf{RequiredCheck} & \mathsf{missingClaim}(c)
 & \begin{aligned}[t]
   c&\in\mathsf{Requires}_{\textsf{Spec}}(\textit{trace},w)\wedge\neg\exists m_c\in\textit{Lang}
  {}\mathrel{.}{}\\[-1ex]
   &\quad\mathsf{ReliableClaim}_t(m_c,c; K_{E,\textit{ctx},t},K_{S,e.\mathsf{sources},t},
   K_{U,t},\textit{ctx})
   \end{aligned}\\[-0.5ex]
\mathsf{WitnessCheck} & \mathsf{unresolved}(x)
 & \mathsf{UnresolvedCitation}(\textit{trace};e,x)\\[-0.5ex]
\mathsf{WitnessCheck} & \mathsf{missing}(c)
 & \mathsf{MissingWitness}_{\textsf{Spec}}(\textit{trace},w;e,c)\\[-0.5ex]
\mathsf{WitnessCheck} & \mathsf{stale}(c,e_j)
 & \mathsf{StaleWitness}_{\textsf{Spec}}(\textit{trace},w;e,c,e_j)\\[-0.5ex]
\mathsf{AuthorityCheck} & \mathsf{missingAuthority}(a)
 & \mathsf{MissingAuthority}_{\textsf{Spec}}(\textit{trace},w;e,a)\\[-0.5ex]
\mathsf{ActuationCheck} & \mathsf{incorrectSuccessor}
 & \mathsf{IncorrectSuccessor}(w;e,w')\\[-0.5ex]
\mathsf{ActuationCheck} & \mathsf{forbidden}(r)
 & \mathsf{ForbiddenActuation}_{\textsf{Spec}}(\textit{trace},w;e,w',r)\\\hline
\end{array}
\]%
}%
The \(\mathsf{missingClaim}\) and \(\mathsf{unreliable}\) rows use
the same event source set as \(\mathsf{RequiredClaimsOK}\) and
\(\mathsf{ReliableAssertions}\), respectively.

\begin{lemma}[Reasons characterize failed checks]\label{lem:reasons-characterize}
For every \(k\in\mathsf{AcceptanceChecks}\),
\begin{equation}
\hspace*{-1em}\mathsf{FailedCheck}_{\textsf{Spec}}(k;\textit{trace},w,e,w')
\iff
\exists\pi\in\mathsf{ReasonValue}\mathrel{.}\,
\mathsf{Reason}_{\textsf{Spec}}
(k,\pi;\textit{trace},w,e,w').
\label{eq:tag:R}
\end{equation}
\end{lemma}
A rejection records one failed check and one reason for it:
\[
\mathsf{Reject}_{\textsf{Spec}}
\triangleq
\{\,\mathsf{RejectFailed}(k,\pi)\mid
k\in\mathsf{AcceptanceChecks},\ 
\pi\in\mathsf{ReasonValue}\,\}.
\]
Equations~(\ref{eq:tag:R}) and~(\ref{eq:tag:F}) establish that a reason exists
for each failed check; Section~\ref{sec:enforcement} addresses when it can be
computed. For an unreliable assertion, the reason carries the diagnostic of
Section~\ref{sec:mismatches}.

\begin{definition}\label{def:reject-sound}
\(\mathsf{Safeguard}_{\textsf{Spec}}\) is \textit{reject-sound}, written
\(\mathsf{SafeguardRejectSound}_{\textsf{Spec}}
(\mathsf{Safeguard}_{\textsf{Spec}})\), when
\[
\begin{aligned}
&\forall\textit{trace},w,e,w',\textit{ctx}\mathrel{.}
\forall k\in\mathsf{AcceptanceChecks}\mathrel{.}
\forall\pi\in\mathsf{ReasonValue}\mathrel{.}(\textit{ctx}=e.\mathsf{context}
\\[-1mm]
&{}\wedge
\mathsf{Safeguard}_{\textsf{Spec}}
(\textit{trace},w,e,w',\textit{ctx})=\mathsf{RejectFailed}(k,\pi))
{\quad}\Rightarrow{\quad}
\mathsf{Reason}_{\textsf{Spec}}(k,\pi;\textit{trace},w,e,w').
\end{aligned}
\]
\end{definition}
A reject-sound safeguard may choose among several valid reasons.

\begin{theorem}[Rejection Adequacy]\label{th:diagnostic-adequacy}
Assume
\(\mathsf{SafeguardRejectSound}_{\textsf{Spec}}
(\mathsf{Safeguard}_{\textsf{Spec}})\).
Then for all \(\textit{trace}\), \(w\), \(e\), \(w'\), with
\(\textit{ctx}=e.\mathsf{context}\):
\[
\mathsf{Safeguard}_{\textsf{Spec}}(\textit{trace},w,e,w',\textit{ctx})
\in\mathsf{Reject}_{\textsf{Spec}}
\Rightarrow
\neg\,\mathsf{AcceptOK}_{\textsf{Spec},e.\mathsf{time},e.\mathsf{context}}
(\textit{trace},w,e,w').
\]
\end{theorem}
\noindent
Together with acceptance soundness, a total safeguard then decides the
certificate:
\begin{theorem}[Safeguard Decision]\label{th:safeguard-decision}
Assume the hypotheses of
Theorem~\ref{th:diagnostic-adequacy}, and assume additionally
\(\mathsf{SafeguardAccept}\discretionary{-}{}{}\mathsf{Sound}_{\textsf{Spec}}
(\mathsf{Safeguard}_{\textsf{Spec}})\)
and that \(\mathsf{Safeguard}_{\textsf{Spec}}\) is total
on calls with \(\textit{ctx}=e.\mathsf{context}\): every
such call returns \(\mathsf{Accept}(e,w')\) or a value in
\(\mathsf{Reject}_{\textsf{Spec}}\).
Then for all \(\textit{trace}\), \(w\), \(e\), \(w'\), with
\(\textit{ctx}=e.\mathsf{context}\):
\[
\begin{aligned}
&\mathsf{Safeguard}_{\textsf{Spec}}(\textit{trace},w,e,w',\textit{ctx})
=\mathsf{Accept}(e,w')
\iff
\mathsf{AcceptOK}_{\textsf{Spec},e.\mathsf{time},e.\mathsf{context}}
(\textit{trace},w,e,w'),
\\[-1ex]
&\mathsf{Safeguard}_{\textsf{Spec}}
(\textit{trace},w,e,w',\textit{ctx})
\in\mathsf{Reject}_{\textsf{Spec}}
\iff
\neg\mathsf{AcceptOK}_{\textsf{Spec},e.\mathsf{time},e.\mathsf{context}}
(\textit{trace},w,e,w').
\end{aligned}
\]
\end{theorem}
This does not establish that such a total safeguard can be implemented.
Refusing a proposal because its status is unresolved is a conservative
decision; it is a sound \(\mathsf{RejectFailed}\) result only when the reported
reason establishes a failed check.

\subsection{The instance, completed}\label{sec:finite-readiness-execution}
For Oma, the finite instance fixes the following events, world states,
execution, and safeguard.

Table~\ref{tab:readiness-ledger} joins the two parts of the instance. The
knowledge objects and claim judgments are calculated in
Section~\ref{sec:finite-readiness-instance}; the events, traces, and safeguard
calculations are supplied below. The table is a ledger of those definitions,
not a second instance.

\begin{table}[ht]
\caption{The split finite instance in one ledger.\label{tab:readiness-ledger}\vspace*{-1ex}}
\small
\begin{tabularx}{\linewidth}{@{}L{0.18\linewidth}YYY@{}}
\toprule
Checkpoint & supported \(d_0\) & neural alternative \(d_1\) & post-edit \(d_1\)\\
\midrule
Candidate and assertion
& \(m_{\mathsf{ready}}\): candidate and asserted
  \(\mathsf{Ready}(d_0)\)
& \(m'_{\mathsf{ready}}\): candidate and asserted
  \(\mathsf{Ready}(d_1)\)
& \(m'_{\mathsf{ready}}\): candidate and asserted
  \(\mathsf{Ready}(d_1)\)\\
\(K_E,K_S,K_U\) status
& \((+,+,+)\) at \(t_0,\textit{ctx}_0\)
& \((+,?,?)\) at \(t_0,\textit{ctx}_0\)
& \((+,?,?)\) at \(t_2,\textit{ctx}_1\)\\
Source support
& the pair for \(s_{\mathsf{Ready}}(d_0)\) is present
& the pair is present, but the source item has status \(?\)
& the pair is present, but the source item has status \(?\)\\
Claim judgment
& supported and reliable
& unsupported assertion; not reliable
& unsupported assertion; not reliable\\
Producer and event
& \(\textit{ReadyServ},e_{\mathsf{ready}}\); the neural
  \(e_{\mathsf{N},0}\) returns the same reliable message
& \(N_{\mathsf{ready}},e_{\mathsf{N},1}\)
& \(\textit{ReadyServ},e^1_{\mathsf{ready}}\)\\
Witness and authority
& route, picture, form, and confirmation are current
& none required for this proposal event
& old form evidence is stale; current \(d_1\) witnesses and authority are
  missing\\
Actuation
& \(m_{\mathsf{markReady}}\) takes \(w_0\) to \(w_0^+\);
  \(e_{\mathsf{N},0}\) has identity actuation
& \(m_\bot\) leaves \(w_0\) unchanged
& actuation check passes, but the event is not admitted\\
Policy and verdict
& allowed; \(\mathsf{Accept}(e_{\mathsf{ready}},w_0^+)\) and
  \(\mathsf{Accept}(e_{\mathsf{N},0},w_0)\)
& policy holds; reject on \(\mathsf{ValidityCheck}\)
& current-witness policy fails; reject on the stale witness\\
\bottomrule
\end{tabularx}
\end{table}

\(\textit{Services}=\{\textit{ReadyServ},\textit{EditObserver},
\textit{PictureServ},\textit{FormServ}\}\),\allowbreak
\(\textit{States}_{\textit{ReadyServ}}=\{r_0,r_1\}\),\allowbreak
\(\textit{States}_{\textit{EditObserver}}=\{u_0,u_1\}\), and
\(\textit{States}_{\textit{Orch}}=\{s,s',s''\}\). Set
\(\textit{Symbolic}=\emptyset\),
\(\textit{Neural}=\{N_{\mathsf{ready}}\}\), and
\(\mathsf{External}=\{\textit{Oma}\}\). The state sets of
\(\textit{PictureServ}\), \(\textit{FormServ}\), and \(\textit{Oma}\) are
\(\{\star\}\); the other service state sets are those just displayed. Let
\(\vec v_0,\vec v_1,\vec v_2\) be the producer-state families whose
\((\textit{ReadyServ},\textit{EditObserver})\) coordinates are
\((r_0,u_0),(r_1,u_0),(r_1,u_1)\), with every other coordinate \(\star\).
Take
\(N_{\mathsf{ready}}=\langle\textit{encode}_{N\!N},N\!N,
\textit{decode}_{N\!N}\rangle\). It uses one-dimensional vectors and is fixed by
\[
\begin{aligned}
\textit{encode}_{N\!N}(\textit{ctx})
  &=\begin{cases}~0&\textit{ctx}=\textit{ctx}_0,\\[-0.75ex]~2&\text{otherwise},\end{cases}
&
\quad N\!N(v)&=\begin{cases}~\{0,1\}&v=0,\\[-0.75ex]~\emptyset&\text{otherwise},\end{cases}
&\quad\textit{decode}_{N\!N}(v)
  &=\begin{cases}
       ~\{m_{\mathsf{ready}}\}&v=0,\\[-0.75ex]
       ~\{m'_{\mathsf{ready}}\}&v=1,\\[-0.75ex]
       ~\emptyset&\text{otherwise}.
     \end{cases}
\end{aligned}
\]
Thus
\(\mathsf{NeuralOutputs}_{N_{\mathsf{ready}}}(\textit{ctx}_0)
=\{m_{\mathsf{ready}},m'_{\mathsf{ready}}\}\): the same selected context
admits two messages.
The finite event set is
\vspace*{-1ex}
\[
\textit{Ev}=\{e_{\mathsf{route}},e_{\mathsf{pic}},e_{\mathsf{form}},
e_{\mathsf{conf}},e_{\mathsf{ready}},e_{\mathsf{edit}},
e^1_{\mathsf{conf}},e^1_{\mathsf{ready}},
e_{\mathsf{N},0},e_{\mathsf{N},1}\}.
\]
Here \(e_{\mathsf{conf}}\) records Oma's confirmation of \(d_0\).
Set \(t_1=7\) and \(t_2=9\). Let \(p_1\) be the new picture record and set
\[
\textit{Src}_{t_1}=\textit{Src}_{t_0}\cup\{p_1\},\qquad
\textit{Src}_{t_2}=\textit{Src}_{t_1},\qquad
\textit{ctx}_{\mathsf{edit}}=\langle m_{\mathsf{observeEdit}}\rangle,
\qquad
\textit{ctx}_1=\langle m'_{\mathsf{ready}}\rangle.
\]
The event \(e^1_{\mathsf{ready}}\) has name \texttt{CheckReady}, actor
\texttt{AppAgent}, kind \texttt{Write}, message \(m'_{\mathsf{ready}}\),
resource \(d_1\), witnesses
\(\{e_{\mathsf{route}},\allowbreak
e_{\mathsf{pic}},\allowbreak
e_{\mathsf{form}},\allowbreak
e_{\mathsf{conf}}\}\),
observation \(o_\bot\), actuation \(m_{\mathsf{markReady}}\), empty authority,
producer \(\textit{ReadyServ}\), context \(\textit{ctx}_1\), sources
\(\textit{Src}_{t_2}\), and time \(t_2\). The event
\(e^1_{\mathsf{conf}}\), absent from the displayed traces, has name
\texttt{ConfirmDraft}, actor \texttt{Oma}, kind \texttt{Authority}, message
\(m_{\mathsf{noClaim}}\), resource \(d_1\), observation \(o_\bot\), actuation
\(m_\bot\), empty witness and authority fields, producer \(\textit{Oma}\),
context \(\epsilon\), sources \(\emptyset\), and time \(8\).
For \(j\in\{0,1\}\), the alternative event \(e_{\mathsf{N},j}\) is named
\texttt{ProposeReady}, has actor \texttt{AppAgent}, kind
\texttt{Computation}, producer \(N_{\mathsf{ready}}\), context
\(\textit{ctx}_0\), sources \(\textit{Src}_{t_0}\), time \(t_0\),
observation \(o_\bot\), actuation \(m_\bot\), and empty witness and authority
fields. Its message and resource are \((m_{\mathsf{ready}},d_0)\) for \(j=0\)
and \((m'_{\mathsf{ready}},d_1)\) for \(j=1\). Define
\(\mathsf{interp}_{\textsf{Spec}}(e_{\mathsf{N},0})
=\{\mathsf{Ready}(d_0)\}\) and
\(\mathsf{interp}_{\textsf{Spec}}(e_{\mathsf{N},1})
=\{\mathsf{Ready}(d_1)\}\), and set both events to be system controlled.

\paragraph{A complete execution.}
Let the four-event trace be
\[
\textit{trace}_0=
e_{\mathsf{route}}e_{\mathsf{pic}}e_{\mathsf{form}}e_{\mathsf{conf}}.
\]
The event \(e_{\mathsf{route}}\) is named \texttt{UKRenewalPicked};
\(e_{\mathsf{pic}},e_{\mathsf{form}},e_{\mathsf{conf}}\) are named
\texttt{CheckPicture}, \texttt{CheckForm}, and \texttt{ConfirmDraft}.
Writing \(e_X\) for these three events, each has message \(m_X\) and witness
field \(\{d_0\}\).  The system-controlled candidate
\(e_{\mathsf{ready}}\) is named \texttt{CheckReady}, has actor
\texttt{AppAgent}, kind \texttt{Write}, message \(m_{\mathsf{ready}}\), resource
\(d_0\), witnesses \(e_{\mathsf{route}},e_{\mathsf{pic}},e_{\mathsf{form}},
e_{\mathsf{conf}}\), observation \(o_\bot\), and actuation
\(m_{\mathsf{markReady}}\).
Its producer is \(\textit{ReadyServ}\), its context is
\(\textit{ctx}_0\), its sources are \(\textit{Src}_{t_0}\), and its time is
\(t_0\).

Let
\vspace*{-1ex}%
\[
\begin{gathered}
\begin{aligned}
w_0&=(\texttt{draft}=d_0,\texttt{ready}=\mathsf{false}),\qquad
w_0^+=(\texttt{draft}=d_0,\texttt{ready}=\mathsf{true}),\\
w_1&=(\texttt{draft}=d_1,\texttt{ready}=\mathsf{false}),\qquad
\;W=\{w_0,w_0^+,w_1\}.
\end{aligned}
\end{gathered}
\]
The remaining choices needed to check acceptance are
\[
\mathsf{Select}_{t_0}(s,\textit{trace}_0,\textit{ReadyServ})
=\{\textit{ctx}_0\},\quad
\mathsf{Select}_{t_0}(s,\textit{trace}_0,N_{\mathsf{ready}})
=\{\textit{ctx}_0\},\quad
\mathsf{Requires}_{\textsf{Spec}}(\textit{trace}_0,w_0)
=\{\mathsf{Ready}(d_0)\}.
\]
Let
\[
\begin{aligned}
\mathsf{interp}_{\textsf{Spec}}(w_0)
&=\{X(d_0)\mid X\in\mathcal P\},\quad
\mathsf{interp}_{\textsf{Spec}}(w_0^+)
=\mathsf{interp}_{\textsf{Spec}}(w_0)\cup\{\mathsf{Ready}(d_0)\},\quad
\mathsf{interp}_{\textsf{Spec}}(e_{\mathsf{route}})
=\{\rho_{\mathsf{UK}}\},\\
\mathsf{interp}_{\textsf{Spec}}(e_X)&=\{X(d_0)\}\quad(X\in\mathcal P),\quad
\mathsf{interp}_{\textsf{Spec}}(e_{\mathsf{ready}})
=\{\mathsf{Ready}(d_0)\}.
\end{aligned}
\]
Finally,
\(\mathsf{observation}(\textit{ReadyServ},w)=o_\bot\) and
\(\mathsf{actuation}(w_0,m_{\mathsf{markReady}})=w_0^+\); all other
actuations in this finite instance leave \(w\) unchanged.
Write \(q_0=(\textit{trace}_0,w_0,e_{\mathsf{ready}})\), with the tuple
expanded when passed to a predicate.
Call an event \(e\) named \texttt{CheckReady} a
\emph{readiness candidate} if
\(d=e.\mathsf{resource}\in D\); a \emph{route event} is one named
\texttt{UKRenewalPicked} or \texttt{FRRenewalPicked}.
Readiness candidates in this instance use direct event
references as witnesses; their witness fields contain no trace slices or
source pointers. For a readiness
candidate \(e\):
\begin{itemize}[leftmargin=1.5em,itemsep=0pt,topsep=2pt]
\item \(\mathsf{NeedsWitness}_{\textsf{Spec}}(\textit{trace},w,e)
=\{\rho_{\mathsf{UK}}\}\cup\{X(d)\mid X\in\mathcal P\}\);
\item \(\mathsf{RequiredAuthority}_{\textsf{Spec}}(\textit{trace},w,e)\) is the
singleton of the \texttt{ConfirmDraft} event confirming \(d\): \(e_{\mathsf{conf}}\)
for \(d_0\), and the fixed event
\(e^1_{\mathsf{conf}}\) for \(d_1\), which is absent from the displayed traces;
\item \(\mathsf{ActuationClause}_{\textsf{Spec}}(\textit{trace},w,e,r)\) holds iff
\(r=m_{\mathsf{markReady}}\) and the draft of \(w\) is \(d\);
\item \(\mathsf{AllowedEvent}_{\textsf{Spec}}(\textit{trace},w,e)\)
holds iff
\(\mathsf{CurrentWitnessesOK}_{\textsf{Spec}}(\textit{trace},w;e)\).
\end{itemize}
For an event \(e\) that is not a readiness candidate, set
\(\mathsf{NeedsWitness}_{\textsf{Spec}}(\textit{trace},w,e)=\emptyset\) and 
\(\mathsf{RequiredAuthority}_{\textsf{Spec}}(\textit{trace},\allowbreak w,\allowbreak e)=\emptyset\);
let \(\mathsf{AllowedEvent}\) hold; and let
\(\mathsf{ActuationClause}(\textit{trace},w,e,r)\) hold for every \(r\).

The checks \(\mathsf{CurrentWitnessesOK}_{\textsf{Spec}}\) and
\(\mathsf{AuthorityOK}_{\textsf{Spec}}\) hold for \(q_0\), since
\(e_{\mathsf{ready}}\) cites \(e_{\mathsf{route}}\) and the three \(e_X\), all in
\(\textit{trace}_0\), none invalidated, and \(e_{\mathsf{conf}}\) occurs in
\(\textit{trace}_0\). The check \(\mathsf{ActuationOK}_{\textsf{Spec}}\) holds
for \((q_0,w_0^+)\). Because \(\mathsf{actuation}(w_0,m_{\mathsf{markReady}})
=w_0^+\), no other successor satisfies \(\mathsf{ActuationOK}\) for this
candidate, so \(\mathsf{AcceptOK}(q_0,w')\) holds exactly for \(w'=w_0^+\).

The following trace-shape check (\ref{eq:write-check-readiness-cat}) is written
in \texttt{cat}~\cite{amt14}. Its relations \texttt{feeds\_into},
\texttt{was\_checked\_before}, \texttt{permits}, and \texttt{invalidates} are
read from the trace as in Proposition \ref{prop:cat-adequacy}.\\[1.5ex]
\noindent\mbox{}\qquad\begin{minipage}{0.90\textwidth}
\begin{flushleft}
%\small
\makeatletter
\def\verbatim@font{\linespread{0.9}\normalfont\ttfamily}
\makeatother
\begin{verbatim}
(* Readiness for submission. *)
let s_before = po
let E_ready = E_appAgent & E_write & E_checkReady
let E_filled = E_checkForm
let s_readiness_route = [E_ukRenewalPicked | E_frRenewalPicked]
  ; feeds_into ; [E_ready]
let s_readiness_picture = [E_checkPicture] ; was_checked_before ; [E_ready]
let s_readiness_form = [E_filled] ; was_checked_before ; [E_ready]
let s_readiness_confirmation = [E_confirmDraft] ; permits ; [E_ready]
let s_readiness_support =
  s_readiness_route | s_readiness_picture | s_readiness_form | s_readiness_confirmation
let s_readiness_stale = [E_formEditInvalidatedConfirmation]
  ; invalidates ; s_readiness_support
empty E_ready \ range(s_readiness_route) as s_readiness_requires_route
empty E_ready \ range(s_readiness_picture) as s_readiness_requires_picture
empty E_ready \ range(s_readiness_form) as s_readiness_requires_form
empty E_ready \ range(s_readiness_confirmation) as s_readiness_requires_confirmation
empty s_readiness_support \ s_before as s_readiness_support_must_precede_ready
empty s_readiness_stale & s_before as s_readiness_requires_current_form_state
\end{verbatim}
\end{flushleft}%
\end{minipage}%
\newlength{\rightboxwidth}%
\settowidth{\rightboxwidth}{\qquad}%
\addtolength{\rightboxwidth}{-0.10\textwidth}%
\begin{minipage}{-\rightboxwidth}
\begin{flushright}
\makeatletter
\newcommand{\eqnum}{\refstepcounter{equation}\textup{(\theequation})}
\makeatother
\eqnum\label{eq:write-check-readiness-cat}
\end{flushright}%
\end{minipage}
\begin{proposition}[The \texttt{cat} trace-shape abstraction]\label{prop:cat-adequacy}
For a readiness candidate \(e\) at \((\textit{trace},w)\)
such that \(\textit{trace}\) contains no earlier readiness candidate, every
citation in \(e.\mathsf{witness}\) is an event reference, and every trace witness cited by \(e\) is
current, read the relations of
(\ref{eq:write-check-readiness-cat}) on \(\textit{trace}\cdot e\) as
follows: \texttt{po} is the order of \(\textit{trace}\cdot e\);
\texttt{feeds\_into}, \texttt{was\_checked\_before}, and \texttt{permits}
relate each event of \(e.\mathsf{witness}\) to \(e\); \texttt{invalidates}
relates \(e_j\) to \(e_i\) when \(e_j\) follows \(e_i\) in \(\textit{trace}\)
and, for some \(c\),
\[
\mathsf{Invalidates}_{\textsf{Spec}}(\textit{trace},w,e_j,e_i,c);
\]
each unary event set in the \texttt{cat} model is
defined by the matching name, actor, or kind field. Suppose, as in this instance,
that an event's interpretation contains \(\rho_{\mathsf{UK}}\)
only if it is named \texttt{UKRenewalPicked}, and contains \(\mathsf{PictureOK}(d)\),
\(\mathsf{FormComplete}(d)\), or \(\mathsf{Confirmed}(d)\) only if it is named
\texttt{CheckPicture}, \texttt{CheckForm}, or \texttt{ConfirmDraft},
respectively. Then \(\mathsf{CurrentWitnessesOK}_{\textsf{Spec}}
(\textit{trace},w;e)\) implies that the \texttt{cat} model accepts
\(\textit{trace}\cdot e\). The converse fails.
\end{proposition}
\noindent\emph{Informally.} If the readiness event has all the current
witnesses required by the semantic specification, the \texttt{cat} trace check
accepts it. The converse need not hold: \texttt{cat} checks the recorded event
pattern, while the semantic witness check also checks what those events mean.

Thus \(\mathsf{AcceptAdequate}_{\textsf{Spec}}\) holds in this
instance. For a readiness candidate,
\(\mathsf{CurrentWitnessesOK}_{\textsf{Spec}}\) is a conjunct of
\(\mathsf{AcceptOK}\) and the definition of
\(\mathsf{AllowedEvent}_{\textsf{Spec}}\). For every other
event, \(\mathsf{AllowedEvent}_{\textsf{Spec}}\) holds by definition. The
\texttt{cat} theorem is a separate comparison with the executable
trace-shape check; it is not used to define the semantic event policy.
The readiness-service relation contains
\[
\langle r_0,t_0,\textit{ctx}_0,
\textit{Src}_{t_0},o_\bot,r_1,
m_{\mathsf{ready}},m_{\mathsf{markReady}}\rangle
\in\mathsf{trans}_{\textit{ReadyServ}}.
\]
Let \(a_0=\mathsf{Accept}(e_{\mathsf{ready}},w_0^+)\) and
\(z_0=\langle s,\vec v_0,w_0,\textit{trace}_0\rangle\).  Let
\(z_0^+=\langle s',\vec v_1,w_0^+,\textit{trace}_0e_{\mathsf{ready}}\rangle\).
By (\ref{eq:tag-E0}),
\[
\langle\textit{ReadyServ},t_0,
\textit{ctx}_0,\textit{Src}_{t_0},
z_0,m_{\mathsf{ready}},o_\bot,m_{\mathsf{markReady}},z_0^+\rangle
\in\mathsf{Result}.
\]
The labelled transition \(z_0\xrightarrow[t_0]{e_{\mathsf{ready}}}z_0^+\)
then satisfies \(\mathsf{InterfaceOK}_{t_0}\).
For the unobserved-edit branch, let
\hspace*{-1ex}%
\[
\widehat z_0=\langle s,\vec v_0,w_1,\textit{trace}_0\rangle,\qquad
\widehat z_0^+=
\langle s',\vec v_1,w_1,\textit{trace}_0e_{\mathsf{ready}}\rangle.
\]
The \(\textit{ReadyServ}\) transition also gives
\[
\langle\textit{ReadyServ},t_0,
\textit{ctx}_0,\textit{Src}_{t_0},
\widehat z_0,m_{\mathsf{ready}},o_\bot,m_{\mathsf{markReady}},
\widehat z_0^+\rangle\in\mathsf{Result}.
\]
Since \(\textit{ReadyServ}\in\textit{Services}\),
\(\mathsf{SystemControlled}(e_{\mathsf{ready}})\) holds. Let the safeguard
be defined on the eight inputs listed below. On that
finite domain it returns \(\mathsf{Accept}(e,w')\) exactly when
\(\mathsf{AcceptOK}\) holds and a rejection otherwise. It is accept-sound on
its domain and returns \(a_0\) on the displayed transition. The reliable claim,
together with the witness, authority, interpretation, and actuation checks
above, establishes \(\mathsf{RequiredClaimsOK}\), \(\mathsf{Validity}\), and
\(\mathsf{AcceptOK}\) for the displayed transition. Acceptance adequacy then
gives \(\mathsf{AllowedEvent}\), while the preceding \(\mathsf{InterfaceOK}\)
judgment shows that the event record agrees with its producer result.

The neural relation gives two alternative results from the same closed state.
For \(j\in\{0,1\}\), let
\[
z_{\mathsf{N},j}=\langle s,\vec v_0,w_0,
\textit{trace}_0e_{\mathsf{N},j}\rangle,
\qquad
(m_0,m_1)=(m_{\mathsf{ready}},m'_{\mathsf{ready}}).
\]
By the definition of \(\mathsf{trans}_{N_{\mathsf{ready}}}\) and
(\ref{eq:tag-E0}),
\[
\langle N_{\mathsf{ready}},t_0,\textit{ctx}_0,
\textit{Src}_{t_0},z_0,m_j,o_\bot,m_\bot,z_{\mathsf{N},j}\rangle
\in\mathsf{Result}
\qquad(j\in\{0,1\}).
\]
Both results are therefore possible for one selected context. The event
\(e_{\mathsf{N},0}\) has a reliable assertion of
\(\mathsf{Ready}(d_0)\). It is not a readiness candidate, so its witness and
authority requirements are empty, its identity actuation is allowed, and
\(\mathsf{AcceptOK}(\textit{trace}_0,w_0,e_{\mathsf{N},0},w_0)\) holds.
For \(e_{\mathsf{N},1}\), the required claim
\(\mathsf{Ready}(d_0)\) is still available reliably through
\(m_{\mathsf{ready}}\) in \(\textit{ctx}_0\), and the witness, authority, and
actuation checks also pass. Its own message instead asserts the unsupported
claim \(\mathsf{Ready}(d_1)\), so \(\mathsf{Validity}\), and hence
\(\mathsf{AcceptOK}\), fails. The safeguard can therefore retain the first
neural result and suppress the second; it applies the reliance checks
separately to each neural result.

\paragraph{The edit transition.}
After the accepted readiness step, an external edit first changes the
world and is then observed and recorded. Let
\[
\begin{gathered}
\bar H_{t_1}=\bar H_{t_0}\cup\{\mathsf{PictureOK}(d_1)\},\qquad
\bar H_{t_2}=\bar H_{t_1},\qquad
K_{U,t_1}=\langle\bar U,\bar H_{t_1},\bar R,t_1\rangle,\qquad
K_{U,t_2}=\langle\bar U,\bar H_{t_2},\bar R,t_2\rangle.
\end{gathered}
\]
The event
\(e_{\mathsf{edit}}\), named \texttt{FormEditInvalidatedConfirmation}, has
actor \texttt{Oma}, kind
\texttt{Observation}, message \(m_{\mathsf{noClaim}}\), observation
\(o_{\mathsf{edit}}\), and actuation \(m_\bot\).
Its producer is \(\textit{EditObserver}\), its context is
\(\textit{ctx}_{\mathsf{edit}}\), its sources are
\(\textit{Src}_{t_1}\), and its time
is \(t_1\).
The World step below is Oma's external edit;
\(e_{\mathsf{edit}}\) is the service's later record of that edit. Let
\(\textit{trace}_1=\textit{trace}_0e_{\mathsf{ready}}e_{\mathsf{edit}}\). Define
\(\mathsf{interp}_{\textsf{Spec}}(w_1)
=\{\mathsf{PictureOK}(d_1)\}\),
\(\mathsf{interp}_{\textsf{Spec}}(e_{\mathsf{edit}})=\emptyset\), and
\(\mathsf{Requires}_{\textsf{Spec}}(\textit{trace}_1,w_1)
=\{\mathsf{Ready}(d_1)\}\).
Let \(z_1^-\) be
\(\langle s',\vec v_1,w_1,\textit{trace}_0e_{\mathsf{ready}}\rangle\) and
\(z_1=\langle s'',\vec v_2,w_1,\textit{trace}_1\rangle\).
Set
\(Env=\{(w_0^+,w_1),(w_0,w_1)\}\). The first
allows an edit after the readiness check; the second allows an unobserved edit
before it. The corresponding world transition is
\(z_0^+\xrightarrow[6]{\mathsf{world}}z_1^-\) without changing
the trace.  Set
\(\mathsf{observation}(\textit{EditObserver},w_1)=\allowbreak o_{\mathsf{edit}}\).
The \(\textit{EditObserver}\) transition is
\[
\langle u_0,t_1,\textit{ctx}_{\mathsf{edit}},
\textit{Src}_{t_1},o_{\mathsf{edit}},u_1,
m_{\mathsf{noClaim}},m_\bot\rangle
\in\mathsf{trans}_{\textit{EditObserver}}.
\]
\smash{By (\ref{eq:tag-E0}),}
\[
\langle\textit{EditObserver},t_1,
\textit{ctx}_{\mathsf{edit}},
\textit{Src}_{t_1},z_1^-,m_{\mathsf{noClaim}},o_{\mathsf{edit}},m_\bot,z_1\rangle
\in\mathsf{Result},
\]
and \(z_1^-\xrightarrow[t_1]{e_{\mathsf{edit}}}z_1\).
Hence \(z_1\) is reachable from \(z_0^+\).
For the candidate \(e^1_{\mathsf{ready}}\) fixed above, let us set
\(\mathsf{interp}_{\textsf{Spec}}(e^1_{\mathsf{ready}})
=\{\mathsf{Ready}(d_1)\}\).  Set
\(c_{\mathsf{form},0}=\mathsf{FormComplete}(d_0)\).
Abbreviate \(i_1=(\textit{trace}_1,\allowbreak w_1,\allowbreak e_{\mathsf{edit}},\allowbreak e_{\mathsf{form}})\)
when passed to \(\mathsf{Invalidates}_{\textsf{Spec}}\), and set
\(\mathsf{Invalidates}_{\textsf{Spec}}(i_1,
c_{\mathsf{form},0})
=\mathsf{true}\).  Thus the edit invalidates \(e_{\mathsf{form}}\) as evidence
for \(\mathsf{FormComplete}(d_0)\).
Set \(\mathsf{Invalidates}_{\textsf{Spec}}\) to false on every
other tuple.
There is no later current form witness, so
\(e_{\mathsf{form}}\) is stale evidence for the proposed \(d_1\) check. This does not
say that \(\mathsf{FormComplete}(d_0)\) became false: it remains a claim about
the earlier draft, while its event is no longer a current witness for the
proposed \(d_1\) readiness check.
Hence
\(\mathsf{StaleWitness}_{\textsf{Spec}}
(\textit{trace}_1,w_1;e^1_{\mathsf{ready}},\allowbreak
c_{\mathsf{form},0},\allowbreak
e_{\mathsf{edit}})\) holds.  Moreover, none of the
\(d_1\) picture, form, or confirmation claims in
\(\mathsf{NeedsWitness}_{\textsf{Spec}}
(\textit{trace}_1,w_1,e^1_{\mathsf{ready}})\) has a cited trace witness.
Equation~(\ref{eq:tag:W}) therefore gives
\(\neg\mathsf{CurrentWitnessesOK}_{\textsf{Spec}}
(\textit{trace}_1,w_1;e^1_{\mathsf{ready}})\).
The \texttt{cat} model of (\ref{eq:write-check-readiness-cat})
rejects it through \texttt{s\_readiness\_requires\_current\_form\_state}.

There is also a distinct unsupported-output path. Set
\[
\begin{aligned}
\mathsf{Select}_{t_2}(s'',\textit{trace}_1,\textit{ReadyServ})
  &=\{\textit{ctx}_1\},\qquad
\mathsf{Asserts}_{E,\textit{ctx}_1,t_2}
  &=\{(m'_{\mathsf{ready}},\mathsf{Ready}(d_1))\}.
\end{aligned}
\]
Keep \(R_S\) and
\(\mathsf{SourceSupport}\) unchanged. Extend the extraction table with
\(\mathcal E(p_1)=\{s_{\mathsf{PictureOK}}(d_1)\}\), use the same union
construction for selected sources, and set
\(F_{S,t_1}=F_{S,t_0}\cup\{s_{\mathsf{PictureOK}}(d_1)\}\), and let
\(K_{S,t_1}\) be the resulting source-derived knowledge base.
Set \(F_{S,t_2}=F_{S,t_1}\) and let \(K_{S,t_2}\) have these
facts, \(R_S\), and time \(t_2\). Let
\(K_{E,\textit{ctx}_1,t_2}=\langle\textit{Lang},
\{m'_{\mathsf{ready}}\},R_E,t_2\rangle\).
Then \(K_{E,\textit{ctx}_1,t_2}\vdash^+m'_{\mathsf{ready}}\), while
\(K_{S,t_2}\vdash^?s_{\mathsf{Ready}}(d_1)\) and
\(K_{U,t_2}\vdash^?\mathsf{Ready}(d_1)\).
Thus the message asserts \(\mathsf{Ready}(d_1)\), although neither the
source-derived nor the universal knowledge base supports that claim.
This is an \(\mathsf{UnsupportedAssertion}\), and hence a
\(\mathsf{RelianceFailure}\). Because \(\mathsf{Ready}(d_1)\) is required and
no reliable message for it exists, the following reason holds
(Section~\ref{sec:rej-diag}):
\[
\mathsf{Reason}_{\textsf{Spec}}
(\mathsf{RequiredCheck},\mathsf{missingClaim}(\mathsf{Ready}(d_1));
\textit{trace}_1,w_1,e^1_{\mathsf{ready}},w_1).
\]
The \(d_1\) readiness candidate therefore requires new
\texttt{CheckPicture}, \texttt{CheckForm}, and \texttt{ConfirmDraft} witnesses
before it can be accepted.
The first failure reuses old evidence after an edit; the second asserts a new
readiness claim that the knowledge bases do not support.

\paragraph{Remaining maps and event fields.}
Together with the indexed assertion relations displayed above and
\(\mathsf{Asserts}_{E,\textit{ctx}_1,t_2}\), these are all the source-support
and assertion pairs used in the instance. No event in the instance uses a
source-pointer citation, so
\(\mathsf{ExtractedFrom}_t=\emptyset\) for every \(t\).
The \(\mathsf{denotation}_{\textit{Lang}}\) map is undefined on messages for
which no value was displayed.

The remaining event fields are as follows. Event
\(e_{\mathsf{route}}\) has actor \texttt{Oma}, message \(m_\bot\), kind
\texttt{Selection}, resource \texttt{UKRenewal}, observation \(o_\bot\),
actuation \(m_\bot\), and empty authority and witness fields.  Each \(e_X\)
has message \(m_X\), resource \(d_0\), observation \(o_\bot\), actuation
\(m_\bot\), and witness \(\{d_0\}\). The picture and form events have actor
\texttt{AppAgent}, kind \texttt{Check}, and empty authority. The confirmation
event has actor \texttt{Oma}, kind \texttt{Authority}, and records Oma's confirmation in its authority
field. Event
\(e_{\mathsf{ready}}\) has an empty authority field. Event
\(e_{\mathsf{edit}}\) has empty authority and witness fields and
resource \(d_1\). Set
\(\mathsf{interp}_{\textsf{Spec}}(e^1_{\mathsf{conf}})
=\{\mathsf{Confirmed}(d_1)\}\).

Set
\hspace*{-2em}%
\[
\mathsf{Select}_{t_1}
(s',\textit{trace}_0e_{\mathsf{ready}},\textit{EditObserver})
=\{\textit{ctx}_{\mathsf{edit}}\}.
\]
For the four events in \(\textit{trace}_0\), use times \(0,1,2,3\) in trace
order and producers \(\textit{Oma}, \textit{PictureServ},\allowbreak \textit{FormServ},\allowbreak
\textit{Oma}\), respectively.
Their contexts are
\((\kappa_0,\kappa_1,\kappa_2,\kappa_3)=
(\epsilon,\textit{ctx}_{\mathsf{pic}},\textit{ctx}_{\mathsf{form}},\epsilon)\)
and their sources are
\((S_0,S_1,S_2,S_3)=
(\emptyset,\{g_{\mathsf{official}},p_0\},\allowbreak
\{g_{\mathsf{official}},f_0\},\allowbreak
\emptyset)\).
The route and confirmation events are externally
controlled; the two check events are system controlled. Finally,
\(K_{S,e_{\mathsf{ready}}.\mathsf{sources},t_0}
=K_{S,e_{\mathsf{N},j}.\mathsf{sources},t_0}=K_{S,t_0}\) for
\(j\in\{0,1\}\), and
\(K_{S,e^1_{\mathsf{ready}}.\mathsf{sources},t_2}=K_{S,t_2}\).

Put
\hspace*{-1.5em}%
\[
(e_0,e_1,e_2,e_3)=
(e_{\mathsf{route}},e_{\mathsf{pic}},e_{\mathsf{form}},e_{\mathsf{conf}}),
\qquad
\sigma_0=\epsilon,
\qquad 
\sigma_{i+1}=\sigma_i e_i\quad(0\leq i<4).
\]
Thus \(\sigma_4=\textit{trace}_0\).

Since the two check events are system controlled, the guarded relation keeps
them only if the safeguard accepts them. Their knowledge bases
grow with the trace. Set
\(\textit{Src}_1=S_1\), \(\textit{Src}_2=S_1\cup S_2\), and
\(\textit{Src}_3=\textit{Src}_4=\textit{Src}_{t_0}\). Set
\(\bar H_0=\{\rho_{\mathsf{UK}},q_c\}\),
\(\bar H_1=\bar H_0\cup\{\mathsf{PictureOK}(d_0)\}\),
\(\bar H_2=\bar H_1\cup\{\mathsf{FormComplete}(d_0)\}\), and
\(\bar H_3=\bar H_4=\bar H_{t_0}\).
Let \(K_{U,1}=\langle\bar U,\bar H_1,\bar R,1\rangle\) and
\(K_{U,2}=\langle\bar U,\bar H_2,\bar R,2\rangle\).
Let
\(K_{S,S_1,1}=\langle\mathsf{SourceItems},
\{s_{\mathsf{PictureOK}}(d_0)\},R_S,1\rangle\) and
\(K_{S,S_2,2}=\langle\mathsf{SourceItems},
\{s_{\mathsf{FormComplete}}(d_0)\},R_S,2\rangle\). The rule of \(R_S\)
needs all three premises, so neither base derives a negation:
\(K_{S,S_1,1}\vdash^+s_{\mathsf{PictureOK}}(d_0)\) and
\(K_{S,S_2,2}\vdash^+s_{\mathsf{FormComplete}}(d_0)\). Likewise
\(K_{U,1}\vdash^+\mathsf{PictureOK}(d_0)\) and
\(K_{U,2}\vdash^+\mathsf{FormComplete}(d_0)\).
Let
\[
K_{E,\kappa_1,1}=\langle\textit{Lang},\{o_{\mathsf{pic}}\},R_E,1\rangle,
\qquad
K_{E,\kappa_2,2}=\langle\textit{Lang},\{o_{\mathsf{form}}\},R_E,2\rangle.
\]
The check rules of \(R_E\)
give \(K_{E,\kappa_i,i}\vdash^+m_{e_i}\) for \(i\in\{1,2\}\), where
\(m_{e_1}=m_{\mathsf{pic}}\) and \(m_{e_2}=m_{\mathsf{form}}\). With the
assertion pairs above and \(\mathsf{SourceSupport}\), each check event
satisfies
\[
\mathsf{ReliableClaim}_i(m_{e_i},X_i(d_0);
K_{E,\kappa_i,i},K_{S,S_i,i},K_{U,i},\kappa_i)
\qquad(i\in\{1,2\}),
\]
with \(X_1=\mathsf{PictureOK}\) and \(X_2=\mathsf{FormComplete}\). Hence
\(\mathsf{ReliableAssertions}\) and \(\mathsf{Validity}\) hold. The check
events are not readiness candidates, so \(\mathsf{Requires}\),
\(\mathsf{NeedsWitness}\), and \(\mathsf{RequiredAuthority}\) are empty, and
\(\mathsf{ActuationClause}\) and \(\mathsf{AllowedEvent}\) hold. Their
witness field \(\{d_0\}\) contains no citation, and the actuation
\(m_\bot\) leaves \(w_0\) unchanged. Therefore
\(\mathsf{AcceptOK}(\sigma_i,w_0,e_i,w_0)\) holds for
\(i\in\{1,2\}\).

For \(0\leq i\leq4\), let
\vspace*{-1ex}%
\[
x_i=\langle s,\vec v_0,w_0,\sigma_i\rangle;
\qquad x_4=z_0;
\qquad Z_0=\{x_0\}.
\]
For \(0\leq i<4\), set
\vspace*{-1ex}%
\[
\mathsf{Select}_i(s,\sigma_i,e_i.\mathsf{producer})=\{\kappa_i\}
\]
and include the producer transition
\[
\langle\star,i,\kappa_i,S_i,o_\bot,
\star,e_i.\mathsf{message},m_\bot\rangle
\in\mathsf{trans}_{e_i.\mathsf{producer}}.
\]
Equation~\textup{(\ref{eq:tag-E0})} then gives
\[
\langle e_i.\mathsf{producer},i,\kappa_i,S_i,x_i,
e_i.\mathsf{message},o_\bot,m_\bot,x_{i+1}\rangle
\in\mathsf{Result},
\]
and hence \(\mathsf{InterfaceOK}_i(x_i,e_i,x_{i+1})\).

The finite execution relation consists exactly of
the eleven transitions used
above and below:
\[
\begin{aligned}
x_i\xrightarrow[i]{e_i}x_{i+1}~~(0\leq i<4),\hskip0.5em
z_0\xrightarrow[t_0]{e_{\mathsf{ready}}}z_0^+,\hskip0.5em
z_0\xrightarrow[t_0]{e_{\mathsf{N},j}}z_{\mathsf{N},j}~~(j\in\{0,1\}),\hskip0.5em
z_0^+\xrightarrow[6]{\mathsf{world}}z_1^-,\hskip0.5em
z_1^-\xrightarrow[t_1]{e_{\mathsf{edit}}}z_1,\hskip0.5em
z_0\xrightarrow[4]{\mathsf{world}}\widehat z_0,\hskip0.5em
\widehat z_0\xrightarrow[t_0]{e_{\mathsf{ready}}}\widehat z_0^+ .
\end{aligned}
\]
In particular,
\[
x_0\xrightarrow[0]{e_0}x_1
\xrightarrow[1]{e_1}x_2
\xrightarrow[2]{e_2}x_3
\xrightarrow[3]{e_3}z_0
\xrightarrow[t_0]{e_{\mathsf{ready}}}z_0^+ .
\]%
Since \(0<1<2<3<t_0=5\), this is a run of (\ref{eq:tag:E3}), and its trace
before the readiness step is \(\textit{trace}_0\).
It is also a run of the guarded relation
\({\longrightarrow}^{G}\): steps \(0\) and \(3\) are uncontrolled, and the
safeguard accepts steps \(1\), \(2\), and the readiness step.
The displayed producer transitions give the nine
\(\mathsf{Result}\) tuples used here by (\ref{eq:tag-E0}). Together with the
recorded event times and the fact that \(\longrightarrow\) contains only the
displayed transitions, they establish \(\mathsf{InterfaceOK}\) for each
recorded transition and hence \(\mathsf{StepOK}\).
The eight safeguard inputs used in the calculation are
\[
\begin{gathered}
(\sigma_1,w_0,e_{\mathsf{pic}},w_0,\textit{ctx}_{\mathsf{pic}}),\quad
(\sigma_2,w_0,e_{\mathsf{form}},w_0,\textit{ctx}_{\mathsf{form}}),\quad
(\textit{trace}_0,w_0,e_{\mathsf{ready}},w_0^+,\textit{ctx}_0),\quad
(\textit{trace}_0,w_0,e_{\mathsf{N},0},w_0,\textit{ctx}_0),\\
(\textit{trace}_0,w_0,e_{\mathsf{N},1},w_0,\textit{ctx}_0),\hskip0.45em
(\textit{trace}_0\cdot e_{\mathsf{ready}},w_1,e_{\mathsf{edit}},w_1,\hskip0.45em
 \textit{ctx}_{\mathsf{edit}}),\hskip0.45em
(\textit{trace}_0,w_1,e_{\mathsf{ready}},w_1,\textit{ctx}_0),\hskip0.45em
(\textit{trace}_1,w_1,e^1_{\mathsf{ready}},w_1,\textit{ctx}_1).
\end{gathered}
\]
The finite set of safeguard inputs contains no other tuple.
On the first two inputs, \(\mathsf{AcceptOK}\) holds as shown
above. On the edit input, \(m_{\mathsf{noClaim}}\) satisfies
\(\mathsf{ReliableAssertions}\). The set of required claims is empty,
\(\mathsf{NeedsWitness}\) and \(\mathsf{RequiredAuthority}\) are
empty, \(\mathsf{ActuationClause}\) and \(\mathsf{AllowedEvent}\) hold, and
the actuation is the identity;
hence \(\mathsf{AcceptOK}\) holds. The safeguard returns:
\begin{itemize}[leftmargin=1.5em,itemsep=0pt,topsep=2pt]
\item \(\mathsf{Accept}(e_{\mathsf{pic}},w_0)\) and
\(\mathsf{Accept}(e_{\mathsf{form}},w_0)\) on the two check inputs;
\item \(a_0\) on the readiness input on \(w_0\);
\item \(\mathsf{Accept}(e_{\mathsf{N},0},w_0)\) on the supported neural
proposal;
\item \(\mathsf{RejectFailed}(\mathsf{ValidityCheck},
\mathsf{unreliable}(\mathsf{Ready}(d_1),
\textsf{UnsupportedAssertion}))\) on the alternative neural proposal;
\item \(\mathsf{Accept}(e_{\mathsf{edit}},w_1)\) on the edit input;
\item \(\mathsf{RejectFailed}(\mathsf{ActuationCheck},
\mathsf{forbidden}(m_{\mathsf{markReady}}))\) on the readiness input on
\(w_1\);
\item \(\mathsf{RejectFailed}(\mathsf{WitnessCheck},
\mathsf{stale}(c_{\mathsf{form},0},e_{\mathsf{edit}}))\) on the last input.
\end{itemize}
Every acceptance check fails on the last input except
\(\mathsf{ActuationCheck}\). Besides the stale witness, its sound reasons are
\[
\begin{gathered}
(\mathsf{ValidityCheck},\mathsf{unreliable}(\mathsf{Ready}(d_1),
\textsf{UnsupportedAssertion})),~~
(\mathsf{RequiredCheck},~~\mathsf{missingClaim}(\mathsf{Ready}(d_1))),\\[-1mm]
(\mathsf{WitnessCheck},\mathsf{missing}(X(d_1)))\ \text{for}\
X\in\mathcal P,~~
(\mathsf{AuthorityCheck},\mathsf{missingAuthority}(e^1_{\mathsf{conf}})),
\end{gathered}
\]
and the safeguard need return only one.

Finally, unlisted tuples are absent from finite relations, and unlisted
arguments of partial maps are undefined. The total maps use the following
defaults: \(\mathsf{Select}\) and \(\mathsf{Requires}\) return \(\emptyset\).
For a readiness candidate, \(\mathsf{NeedsWitness}\) and
\(\mathsf{RequiredAuthority}\) have the values specified above.
\(\mathsf{ActuationClause}\) and \(\mathsf{AllowedEvent}\) also have the values
specified above. For every other event, \(\mathsf{NeedsWitness}\) and
\(\mathsf{RequiredAuthority}\) are empty, while \(\mathsf{ActuationClause}\)
and \(\mathsf{AllowedEvent}\) hold.
\(\mathsf{Invalidates}\) is false except on
\((i_1,c_{\mathsf{form},0})\).
On otherwise unspecified arguments, the
\(\mathsf{observation}\) map returns \(o_\bot\), and
\(\mathsf{actuation}\) leaves the world unchanged.  These conventions specify
every value used in the displayed calculations.

The mechanised and executable checks for this instance are reported in
Section~\ref{sec:implementation-validation}.

\section{Safeguards as Enforcement Mechanisms\label{sec:enforcement}}

A safeguard can suppress only events controlled by the system and
can base its decision only on information it observes. We define the policy
induced by a safeguard and give conditions for sound enforcement,
transparency, observability, and computability. These are the standard
questions for execution monitors
\cite{DBLP:journals/tissec/Schneider00}.

\subsection{The policy enforced by a safeguard}\label{sec:policy-enforced-by-safeguard}

A \emph{run} of (\ref{eq:tag:E2}) is a finite sequence
\(z_0\xrightarrow[t_0]{\ell_0}\cdots
\xrightarrow[t_{n-1}]{\ell_{n-1}}z_n\), where
\(z_k=\langle s_k,\vec s_k,w_k,\textit{trace}_k\rangle\). Step \(k\) is
\emph{controlled} when \(\ell_k=e_k\in\textit{Ev}\) and
\(\mathsf{SystemControlled}(e_k)\); every other step (an event not controlled by
the system, an internal step, or a World step) is \emph{uncontrolled}. The
\emph{guarded} relation \({\longrightarrow}^{G}\) keeps the
uncontrolled steps of \(\longrightarrow\) and those controlled
steps for which
\(\mathsf{Safeguard}_{\textsf{Spec}}
(\textit{trace}_k,w_k,e_k,w_{k+1},e_k.\mathsf{context})
=\mathsf{Accept}(e_k,w_{k+1})\). The
policy is the set of runs
\[
\mathcal P_{\textsf{Spec}}\triangleq\{\,z_0
\xrightarrow[t_0]{\ell_0}\cdots z_n\mid
\forall k<n.\ \text{step }k\text{ controlled}\Rightarrow
\mathsf{AllowedEvent}_{\textsf{Spec}}(\textit{trace}_k,w_k,e_k)\,\}.
\]

The policy \(\mathcal P_{\textsf{Spec}}\) depends on world state as
well as the recorded trace: \(\mathsf{AllowedEvent}_{\textsf{Spec}}\) reads
\(w_k\), while a World step may change \(w_k\) without extending the trace.
Accept- and reject-soundness constrain individual safeguard calls. The next
theorem lifts those contracts, acceptance adequacy, and interface agreement to
complete runs.

\begin{theorem}[Run-level enforcement]\label{th:run-enforcement}
\begin{enumerate}[label=(\roman*),leftmargin=*]
\item \(\mathcal P_{\textsf{Spec}}\) is prefix closed. Every run
outside it has a shortest prefix outside it, and that prefix ends with a
controlled step. In particular, an uncontrolled step cannot create a
violating prefix.
\item If
\(
\begin{gathered}
\mathsf{SafeguardAcceptSound}_{\textsf{Spec}}
  (\mathsf{Safeguard}_{\textsf{Spec}}),\quad
\mathsf{AcceptAdequate}_{\textsf{Spec}},\quad
\mathsf{StepOK}(\mathsf{Step}_{\textit{Orch}})
\end{gathered}
\)
hold, then every run of
\({\longrightarrow}^{G}\) is in \(\mathcal P_{\textsf{Spec}}\), and
every controlled step \(k\) in it satisfies
\(
\mathsf{InterfaceOK}_{t_k}(z_k,e_k,z_{k+1}).
\)
\end{enumerate}
\end{theorem}
\noindent\emph{Informally.} A policy violation first appears at a controlled
event, so the guard has a chance to stop it. If safeguard acceptances pass
\(\mathsf{AcceptOK}\), \(\mathsf{AcceptOK}\) implies policy permission, and
recorded steps agree with their producer results, then every controlled event
that the guard lets through is allowed and accurately recorded. User and
environment steps remain outside the guard's control.

The shortest-violating-prefix property is the controllability condition used
in execution monitoring
\cite{DBLP:journals/tissec/Schneider00,DBLP:journals/toplas/HamlenMS06,
DBLP:journals/tissec/BasinJKZ13}.
Table~\ref{tab:enforcement-map} gives the corresponding enforcement
terminology~\cite{DBLP:journals/ijisec/LigattiBW05,DBLP:journals/sttt/FalconeFM12}.
\begin{table}[h]
\caption{Transition safeguards as enforcement mechanisms.\label{tab:enforcement-map}\vspace*{-1ex}}
%\small
\begin{tabularx}{0.8\linewidth}{@{~~~}lY@{}}
\toprule
This paper & Enforcement mechanisms\\
\midrule\(\mathsf{Accept}(e,w')\) & let the action through\\
rejection with no further action & suppression~\cite{DBLP:journals/ijisec/LigattiBW05}\\
rejection, then pause or fallback & truncation, or insertion of a corrective action\\
\(\mathsf{SystemControlled}\) events & controllable actions~\cite{DBLP:journals/tissec/BasinJKZ13}\\
user, environment, and World steps & uncontrollable, possibly unobservable actions\\
accept-soundness and adequacy & soundness of enforcement\\
\(\Leftarrow\) of Theorem~\ref{th:safeguard-decision}
 & transparency, but relative to \(\mathsf{AcceptOK}\)
 (Theorem~\ref{th:transparency})\\
\bottomrule
\end{tabularx}
\end{table}

The policy \(\mathcal P_{\textsf{Spec}}\) governs authorisation of system
actions; it is not a state invariant. A predicate \(I\) on closed states holds
throughout every guarded run if it holds initially and every reachable
uncontrolled step and every controlled step retained by the guard preserves
\(I\).

\subsection{Transparency}\label{sec:Transparency}
An enforcement mechanism is \emph{transparent} when it leaves
policy-satisfying behaviour unchanged~\cite{DBLP:journals/ijisec/LigattiBW05}.
Theorem~\ref{th:safeguard-decision} gives transparency with respect to
\(\mathsf{AcceptOK}\), not with respect to the policy.

\begin{theorem}[Transparency]\label{th:transparency}
\abovedisplayskip0pt{\belowdisplayskip0pt%
Under the hypotheses of Theorem~\ref{th:safeguard-decision}, consider a
controlled step \(k\) of \(\longrightarrow\) such that
\[
\mathsf{AllowedEvent}_{\textsf{Spec}}
  (\textit{trace}_k,w_k,e_k).
\]}
The guard removes this step iff
\[
\neg\mathsf{AcceptOK}_{\textsf{Spec},
e_k.\mathsf{time},e_k.\mathsf{context}}
(\textit{trace}_k,w_k,e_k,w_{k+1}).
\]
Hence \({\longrightarrow}^{G}\) removes no allowed controlled step
iff, on controlled steps, \(\mathsf{AllowedEvent}_{\textsf{Spec}}\) implies
\(\mathsf{AcceptOK}\) for the recorded context: the converse of
\(\mathsf{AcceptAdequate}_{\textsf{Spec}}\).
\end{theorem}
\noindent\emph{Informally.} The guard leaves every policy-allowed system step
untouched exactly when every such step also passes \(\mathsf{AcceptOK}\).
Acceptance adequacy establishes the other direction: anything that passes
\(\mathsf{AcceptOK}\) is policy-allowed.

The finite instance separates acceptance adequacy from
its converse. At \((\textit{trace}_0,w_1)\), the \texttt{cat} model accepts
\(e_{\mathsf{ready}}\) because its required events occur in the trace.
The semantic event policy also permits it because the cited
\(d_0\) witnesses have not been invalidated in that trace.
\(\mathsf{AcceptOK}\) rejects it because \(w_1\) has current draft \(d_1\),
whereas \(e_{\mathsf{ready}}\) requests the readiness actuation for \(d_0\).
The guarded relation therefore removes an event allowed by the
semantic witness policy.

\subsection{What a safeguard must observe}
Although the abstract safeguard takes \(w\) and \(w'\) as arguments, an
implementation in this framework obtains information about the world through
service observations. Let
\(G\subseteq\textit{Services}\) be the services whose observations the
safeguard can use, let \(O=\textit{Lang}^{G}\), and define
\[
\textit{obs}_G:W\to O,
\qquad
\textit{obs}_G(w)(a)=\mathsf{observation}(a,w)
\quad(a\in G).
\]

Fix
\(\textit{trace}\) and \(e\), put
\(\textit{ctx}=e.\mathsf{context}\), and let
\(R_{\textit{trace},e,\textit{ctx}}\subseteq W\) include every
world appearing in a reachable closed state with that trace and selected
context from which \(e\) is enabled. It may be any larger set, including
\(W\). For each \(w\in R_{\textit{trace},e,\textit{ctx}}\), put
\(w_e(w)=\mathsf{actuation}(w,e.\mathsf{actuation})\), and write
\begin{equation}
A(w)\triangleq
\mathsf{AcceptOK}_{\textsf{Spec},
e.\mathsf{time},e.\mathsf{context}}
(\textit{trace},w,e,w_e(w)).
\label{eq:def:A(w)}
\end{equation}
An observation-based guard is a function \(g:O\to\{0,1\}\). It is sound on
\(R_{\textit{trace},e,\textit{ctx}}\) when
\(g(\textit{obs}_G(w))=1\) implies \(A(w)\) for
every \(w\in R_{\textit{trace},e,\textit{ctx}}\).
A guard accepts every acceptable case, and no unacceptable case, when
\(g(\textit{obs}_G(w))=1\) if and only if \(A(w)\), for every
\(w\in R_{\textit{trace},e,\textit{ctx}}\). Write
\(R=R_{\textit{trace},e,\textit{ctx}}\) while this supplied set is fixed.
Define
\[
g_R^*(o)=1\quad\Longleftrightarrow\quad
\forall w\in R\mathrel{.}
\textit{obs}_G(w)=o\Rightarrow A(w).
\]
Thus \(g_R^*\) accepts an observation only when every world in the
supplied set that produces that observation passes the acceptance test. The
set contains every reachable world, so every accepted reachable transition
passes the test.

\begin{theorem}[Most permissive sound guard relative to a supplied
world set]
\label{th:greatest-observable-guard}
For the fixed trace, event, and context above:
\begin{enumerate}[label=(\alph*),leftmargin=*]
\item \(g_R^*\) is sound. If \(g\) is any other sound
observation-based guard, then \(g(o)\leq g_R^*(o)\) for every
observation \(o\).
\item There is an observation-based guard satisfying
\[
g(\textit{obs}_G(w))=1\quad\Longleftrightarrow\quad A(w)
\]
for every \(w\in R_{\textit{trace},e,\textit{ctx}}\) if and only if \(A\) has
the same value at any two worlds
\(w_a,w_b\in R_{\textit{trace},e,\textit{ctx}}\) with the same observation.
When this condition holds, \(g_R^*\) satisfies the displayed
equivalence.
\item If \(R_{\textit{trace},e,\textit{ctx}}\) is supplied as a finite list,
or a terminating procedure enumerates the whole finite set, and
\(A\) is decidable and \(\textit{obs}_G\) is computable on it,
and equality on \(O\) is
decidable, \(g_R^*\) can be computed by grouping
the states by observation and accepting a group exactly when \(A\) holds at
every state in it. It returns \(1\) for an observation that does
not occur in the list.
\item If \(R\subseteq R'\), then
\(g_{R'}^*(o)\leq g_R^*(o)\) for every observation \(o\).
\end{enumerate}
\end{theorem}
\noindent\emph{Informally.} A guard must make the same decision in worlds that
look the same to it. The most permissive sound guard accepts an observation
only if the event is acceptable in every such world. It accepts every
acceptable case exactly when worlds with the same observation never require
different decisions. On a finite instance, this guard can be calculated by
grouping worlds with the same observation.
By part~(d), enlarging the supplied world set can only remove acceptances.
A sound over-approximation of reachable worlds therefore yields a sound guard.
Here a guard's rejection withholds permission. It need not establish that the
event fails the checks in the actual world: a mixed observation class contains
both passing and failing worlds. Reporting a sound failed-check diagnosis
requires the additional condition below.

Taking \(R_{\textit{trace},e,\textit{ctx}}=W\), item~(b)
shows that \(g_R^*(\textit{obs}_G(w))=1\) if and only if \(A(w)\)
whenever \(A(w)\) depends on \(w\) only through
\(\textit{obs}_G(w)\). Using \(g_R^*\) on the calls
\((\textit{trace},w,e,w_e(w),\textit{ctx})\) and rejecting every other call
gives an accept-sound safeguard: on those calls
\(\textit{ctx}=e.\mathsf{context}\), and the definition of \(A\)
gives \(\mathsf{AcceptOK}\) whenever the safeguard accepts.

An observation-based guard determines acceptance; the safeguards considered
here must also return a sound reason whenever they reject. For
\(o\in\textit{obs}_G(R_{\textit{trace},e,\textit{ctx}})\), let
\[
\begin{aligned}
F(o)=\{\,(k,\pi)\in{}
&\mathsf{AcceptanceChecks}\times\mathsf{ReasonValue}\mid{}
\forall w\in R_{\textit{trace},e,\textit{ctx}}\mathrel{.}\\[-1mm]
&\textit{obs}_G(w)=o\Rightarrow{}
\mathsf{Reason}_{\textsf{Spec}}
(k,\pi;\textit{trace},w,e,w_e(w))\,\}.
\end{aligned}
\]
The set \(F(o)\) contains the rejection reasons that are true in
every world in the supplied set producing observation \(o\). These are the only reasons
a sound observation-based safeguard can report from \(o\) alone.

\begin{theorem}[Safeguard decisions from observations]
\label{th:observable-safeguard}
On the calls
\((\textit{trace},w,e,w_e(w),\textit{ctx})\) with
\(w\in R_{\textit{trace},e,\textit{ctx}}\), there is a
safeguard defined on all these calls that returns either
\(\mathsf{Accept}(e,w_e(w))\) or \(\mathsf{RejectFailed}(k,\pi)\) for some
\(k\in\mathsf{AcceptanceChecks}\) and
\(\pi\in\mathsf{ReasonValue}\). Its decision and any check and reason it
reports depend on \(w\) only through \(\textit{obs}_G(w)\). Every acceptance
satisfies \(A(w)\), the acceptance certificate (\ref{eq:def:A(w)}) for the
fixed \textit{trace} and \(e\), and every result
\(\mathsf{RejectFailed}(k,\pi)\) satisfies
\(\mathsf{Reason}_{\textsf{Spec}}
(k,\pi;\textit{trace},w,e,w_e(w))\). Such a safeguard exists if and only if
\begin{equation}
g_R^*(o)=0\quad\Longrightarrow\quad F(o)\ne\varnothing
\qquad
\bigl(o\in\textit{obs}_G(R_{\textit{trace},e,\textit{ctx}})\bigr).
\label{eq:tag:O}
\end{equation}
Condition~(\ref{eq:tag:O}) implies that indistinguishable worlds have the
same value of \(A\), as required by
Theorem~\ref{th:greatest-observable-guard}.(b).
Assume Condition~(\ref{eq:tag:O}), the decidability hypotheses of
Theorem~\ref{th:greatest-observable-guard}.(c), and decidability of
\(\mathsf{Reason}_{\textsf{Spec}}\). Suppose further that, for each
\(w\in R_{\textit{trace},e,\textit{ctx}}\), a terminating procedure returns a
finite list \(C_w\) of elements of
\(\mathsf{AcceptanceChecks}\times\mathsf{ReasonValue}\) containing every pair
\((k,\pi)\) for which
\(\mathsf{Reason}_{\textsf{Spec}}
(k,\pi;\textit{trace},\allowbreak w,\allowbreak e,\allowbreak w_e(w))\) holds.
A rejection reason can then be computed as follows: for each rejected
observation class, choose one world \(w_0\) in that class and scan \(C_{w_0}\)
for the first pair that satisfies \(\mathsf{Reason}_{\textsf{Spec}}\) at every
world in the class.
\end{theorem}
\noindent\emph{Informally.} If several possible worlds look the same to the
safeguard, it can report a rejection reason only when that reason is true in
all of them. Condition~(\ref{eq:tag:O}) requires every observation that the
safeguard rejects to have at least one reason shared by all of those worlds.

The condition in Theorem~\ref{th:greatest-observable-guard}.(b) is the counterpart of
observability in supervisory control of discrete-event systems
\cite{DBLP:journals/mcss/WonhamR88,DBLP:journals/isci/LinW88,DBLP:journals/tac/YooL02a}.

The preceding results take \(G\) as given.  Suppose now that \(R\) is finite,
and let \(S\subseteq\textit{Services}\) be a finite candidate set.  Put
\[
R^+=\{w\in R\mid A(w)\},\qquad
R^-=\{w\in R\mid\neg A(w)\},\qquad
P=R^+\times R^-.
\]
For each \(a\in S\), let
\[
P_a=\{(w_+,w_-)\in P\mid
\mathsf{observation}(a,w_+)\ne
\mathsf{observation}(a,w_-)\}.
\]
Thus \(P_a\) contains the opposite-decision pairs that service \(a\)
distinguishes.  Call \(G\subseteq S\) \emph{sufficient} when \(A\) has the
same value at worlds in \(R\) with the same \(G\)-observation.

\begin{theorem}[Selecting sufficient observations]
\label{th:selecting-observations}
For every \(G\subseteq S\), the following are equivalent:
\begin{enumerate}[label=(\alph*),leftmargin=*]
\item there is a guard \(g\) such that
\(g(\textit{obs}_G(w))=1\Longleftrightarrow A(w)\) for every \(w\in R\);
\item \(G\) is sufficient;
\item \(G\) distinguishes every opposite-decision pair:
\(\displaystyle 
P=\bigcup_{a\in G}P_a.
\)
\end{enumerate}
When these conditions hold, \(g_R^*\) is such a guard.

After \(A\) has been evaluated, take the decision-problem input to be explicit
finite tables for \(R\), \(S\), the values \(A(w)\), and the observations,
with polynomial-time equality on observation values.  Checking a fixed \(G\) uses
\(O(|G||R^+||R^-|)\) observation comparisons.  Deciding whether there is a
sufficient \(G\subseteq S\) with \(|G|\leq k\) is NP-complete, even when
service observations are binary and \(|R^-|=1\).  Hence finding a minimum
sufficient observation set is NP-hard.
\end{theorem}
\noindent\emph{Informally.} A guard can accept exactly the cases satisfying
\(A\) only if its services distinguish every \(A\)-world from every
\(\neg A\)-world. Each service determines a subset of these pairs, so selecting
at most \(k\) services is Set Cover on \(P\). This establishes the verdict
only; a reported rejection reason must still satisfy
Condition~(\ref{eq:tag:O}) throughout the observation class.

\begin{foroma}
In the instance of
Section~\ref{sec:finite-readiness-instance}, take
\(\textit{trace}=\textit{trace}_0\), \(e=e_{\mathsf{ready}}\),
\(\textit{ctx}=\textit{ctx}_0\), \(w_a=w_0\), and \(w_b=w_1\): Oma edits the form
before the readiness check, and nothing records it. Both closed states are
reachable, the second from \(z_0\) by a World step since \((w_0,w_1)\in Env\).
\(\mathsf{AcceptOK}\) holds at \(w_0\) with successor \(w_0^+\), but fails at
\(w_1\), whose successor is \(w_1\) itself, because the actuation clause
requires the current draft to be \(d_0\). Let
\(R_{\textit{trace}_0,e_{\mathsf{ready}},\textit{ctx}_0}=\{w_0,w_1\}\).
For the candidate set
\(S=\{\textit{ReadyServ},\textit{EditObserver}\}\), the relevant table is
\[
\begin{array}{c|cc|c}
 & \textit{ReadyServ} & \textit{EditObserver} & A\\ \hline
w_0 & o_\bot & o_\bot & 1\\[-0.75ex]
w_1 & o_\bot & o_{\mathsf{edit}} & 0
\end{array}
\]
For \(G=\{\textit{ReadyServ}\}\),
\(\textit{obs}_G(w_0)=\textit{obs}_G(w_1)=(o_\bot)\), while \(A\) is true at
\(w_0\) and false at \(w_1\). Theorem~\ref{th:greatest-observable-guard}.(b)
therefore says that no guard can accept the acceptable request at \(w_0\)
and reject the unacceptable request at \(w_1\): because the observations are
the same, it must either reject both or accept both.  In the notation of
Theorem~\ref{th:selecting-observations},
\[
P=\{(w_0,w_1)\},\qquad
P_{\textit{ReadyServ}}=\varnothing,\qquad
P_{\textit{EditObserver}}=P.
\]
Consequently, \(\{\textit{EditObserver}\}\) is the unique minimum sufficient
subset of \(S\) for this fixed readiness event; \(\textit{ReadyServ}\) may
still be needed for other events.  For
\(G=\{\textit{ReadyServ},\textit{EditObserver}\}\) the observations of
\(w_0\) and \(w_1\) differ. Each observation class is now a singleton, so
\(g_R^*\) accepts the request at \(w_0\) and rejects it at \(w_1\). Moreover,
\((\mathsf{ActuationCheck},
\mathsf{forbidden}(m_{\mathsf{markReady}}))\in F(\textit{obs}_G(w_1))\), so
Condition~(\ref{eq:tag:O}) holds and
Theorem~\ref{th:observable-safeguard} supplies the rejection reason as
well as the verdict.
\end{foroma}
The calculation determines a service requirement: more computation over the
unchanged trace cannot reveal the unrecorded edit. The guard needs an
observation that distinguishes the current drafts at the modeled event
boundary.

\subsection{What a safeguard can compute}
Hamlen et al.~\cite{DBLP:journals/toplas/HamlenMS06} show that monitor-enforceable
policies must also be computable from the observed prefix. For a recursively
enumerable rule system with finite premise sets and recursively enumerable
facts, \(K\vdash c\) is semi-decidable, by enumerating proofs, but
\(K\not\vdash\neg c\), required for positive universal and source status,
is in general only co-semi-decidable.
Finite knowledge bases make their justification statuses decidable,
but they do not by themselves provide a safeguard that can accept
every acceptable event and reject every unacceptable one. Every other
predicate in \(\mathsf{AcceptOK}\) must also be decidable, and
\(\textit{obs}_G\) must determine \(\mathsf{AcceptOK}\) as required by
Theorem~\ref{th:greatest-observable-guard}.(b). For the recursively enumerable
systems considered above, a checkable proof object can establish \(K\vdash c\),
but cannot in general establish that no derivation exists. When
the closure of \(K_{\mathrm{hi}}\) is computable,
\(K_{\mathrm{hi}}\not\vdash\neg c\) establishes that no knowledge base in the
interval justifies \(\neg c\).

\paragraph{Gateways.}
An API or AI gateway can act as an XACML-style policy-enforcement point for
calls it can suppress~\cite{oasis2013xacml}. A check is enforceable there only
when its inputs are available to the gateway.
If the gateway's observation determines \(\mathsf{AcceptOK}\),
Theorem~\ref{th:greatest-observable-guard}.(b) gives a
guard that accepts if and only if \(\mathsf{AcceptOK}\) holds. If two reachable
worlds have the same gateway observation but different values of
\(\mathsf{AcceptOK}\), item~(b) says that no guard can
accept all and only the acceptable events. Checks of support, freshness, and authority therefore require the
outcomes of those checks to be determined by information
available to the gateway.

\section{Implementation: A Generic Harness Library}
\label{sec:implementation}

We implement the finite checks as a reusable OCaml harness library. An
application supplies its events, knowledge, specification, \texttt{cat}
contracts, candidate producer, service interfaces, and effect adapters. The
library calculates claim and acceptance judgments, checks the contracts, and
admits a selected proposal before its effect can run. Passport, Travel, and
Minute Maid instantiate the same library for renewal forms, disrupted
journeys, and bounded transcription, respectively. The release build pins
OCaml 4.14.1, Yojson 2.1.0, and the Herd source revision used to check the
contracts~\cite{amt14}.

The implementation computes consequences of the supplied inputs. Their
factual correctness, source authenticity, observation accuracy, and policy
adequacy remain application obligations.

\subsection{Finite semantic checks and partial-order executions}

The library computes finite rule closures, four-valued statuses, the claim
judgments, and the five conjuncts of \(\mathsf{AcceptOK}\), retaining a failed
conjunct and its rejection reason. It also calculates the bounded
acceptance test of Corollary~\ref{cor:interval-accept} from application-supplied
lower and upper knowledge bases and a compatible witness.  A bounded result
cannot be passed to the runtime in place of an exact acceptance result.

The paper uses sequential traces; the library uses a finite event set with a
strict partial order \texttt{po}. It checks \texttt{rf}, connecting each read
to one authorised write, and derives \texttt{hb} as the transitive closure of
\texttt{po}, \texttt{rf}, and the application relations
\texttt{dispatches}, \texttt{feeds\_into},
\texttt{was\_checked\_before}, and \texttt{permits}.  The
\texttt{invalidates} relation is not included in \texttt{hb}.

A paper trace followed by candidate \(e\) embeds by preserving its order and
placing \(e\) last. The freshness check then gives the sequential rule: the
witness precedes \(e\), and no relevant invalidator occurs between them.
For partially ordered events, the implementation also rejects a witness
unordered with \(e\), or an invalidator whose position does not establish
freshness. This extension is outside the sequential theorems.

Semantic invalidation names an invalidator, witness, and claim; the binary
\texttt{cat} relation omits the claim. The application must keep these inputs
consistent. Under the correspondence of Proposition~\ref{prop:cat-adequacy},
the semantic witness check implies the printed \texttt{cat} check. The converse
fails because the event pattern alone does not establish meaning or support.

The finite verifier can check \(\mathsf{StepOK}\) and
\(\mathsf{AcceptAdequate}\) over supplied lists of transitions and candidates;
it cannot establish that those lists exhaust an application's behaviour.

\subsection{Runtime admission and protected effects}

In the implementation, a \emph{runtime candidate} proposes one execution and
effect descriptor. It binds the event to its knowledge and specification
snapshot, evaluation instant, state, evidence, authority, and checker inputs.
The runtime checks policy and orchestration contracts separately against this
same proposal; a registered service interface supplies an independently
calculated successor. Admission requires
\[
\begin{aligned}
\mathsf{RuntimeAdmissionOK}\ \triangleq\ {}&
\mathsf{SystemControlled}
\wedge \mathsf{AcceptOK}
\wedge \mathsf{AllowedEvent}
\wedge \mathsf{CatOK}\\[-0.5ex]
&{}\wedge \mathsf{CurrentStepInterfaceOK}.
\end{aligned}
\]
The last conjunct checks interface agreement and that the candidate is the
recorded current orchestrator step. The runtime supplies the recording
premise; evaluating this predicate alone does not execute or protect an effect.

The orchestrator assesses all supplied candidates, refuses if none passes,
selects the sole passing candidate, or requires an explicit choice among
several. The application supplies the candidate set. Before dispatch, the
runtime rechecks the selected transition's contracts, state, and effect
descriptor and consumes its one-use effect permission. Failed or uncertain
effects cannot reuse that permission. These implementation checks bind admission to an
effect; they do not establish the application-wide premises of
Theorem~\ref{th:run-enforcement}.

\subsection{Controlled evaluation}

\paragraph{A runtime execution.}
We exercised a matched pair of synthetic Passport PDF exports through the
Passport runtime, Herd contract checks, and the native PDF renderer.
One case exports the confirmed current draft; the other edits the draft after
confirmation. Both pass
\(\mathsf{SystemControlled}\wedge\mathsf{AllowedEvent}\wedge
\mathsf{AuthorityOK}\wedge\mathsf{InterfaceOK}\), and both policy and
orchestration \texttt{cat} contracts accept. Full admission accepts the current
export, consumes its one-use effect permission, invokes the renderer once, and commits
the PDF. It rejects the edited export before effect execution, leaving the
state unchanged. Reusing the consumed permission is rejected. The cases and
runtime records are included in the harness artifact under
\path{evaluation/passport-admission-pilot}.

In the edited case, both \(\mathsf{CurrentWitnessesOK}\) and the implementation's
\(\mathsf{Step\_Orch}\) fail, so the rejection cannot be attributed to freshness
alone. The cases have study-authored labels, and \(\mathsf{AuthorityOK}\) is
vacuous for PDF export. The PDF check establishes structural framing, expected
form content, and embedding of the synthetic picture; it does not assess
visual quality. The browser download is not executed. This pair demonstrates
admission at one effect boundary, without a live model or an estimate of
task success or error rates.

\paragraph{Finite conformance and cost.}
We ran an offline study over committed Passport and Travel cases.  These are
regression and generated cases, not a held-out sample.  The study selected ten
distinct Passport \texttt{cat} cases: the four expected-accept cases passed and
the six expected-reject cases failed.  Each failed obligation matched either
the exact expected name or its permitted set.  It also ran eleven Travel
properties---four that preserve admission and seven that exercise
rejection---over 1,750 fixed-seed property-test cases; all eleven properties
passed. The Travel results are aggregate property checks, without per-case
admission or diagnostic records.

A separate typed probe compared four checks on the same three valid and five
invalid fixtures.  ``Typed value only'' records only that an
\texttt{Evaluation.t} was constructed; it runs no admission check.
``\texttt{cat} only'' reads the \(\mathsf{CatOK}\) judgment supplied by the
fixture; it does not run Herd.  ``Permission and interface'' is
\(\mathsf{SystemControlled}\wedge\mathsf{AllowedEvent}\wedge
\mathsf{AuthorityOK}\wedge\mathsf{InterfaceOK}\).
Table~\ref{tab:implementation-probe} gives the results for these eight fixtures.

\begin{table}[H]
\caption{Results for the eight-fixture typed probe.
\label{tab:implementation-probe}}
\small
\begin{tabular}{@{}lrrr@{}}
\toprule
Check & Valid admitted & Invalid rejected & Invalid admitted\\
\midrule
Typed value only & 3/3 & 0/5 & 5/5\\
\texttt{cat} only & 3/3 & 4/5 & 1/5\\
Permission and interface & 3/3 & 1/5 & 4/5\\
Full \(\mathsf{RuntimeAdmissionOK}\) & 3/3 & 5/5 & 0/5\\
\bottomrule
\end{tabular}
\end{table}

Removing \(\mathsf{CatOK}\) admitted one invalid fixture. Removing any one of
the five \(\mathsf{AcceptOK}\) conjuncts admitted none: the remaining checks
still rejected each invalid fixture. These overlapping failures do not
establish the separate contribution of those five conjuncts.

We also timed the in-memory evaluation of
\(\mathsf{RuntimeAdmissionOK}\) on dense Passport fixtures.  From 8 events and
28 \texttt{po} edges to 304 events and 46,056 \texttt{po} edges, the median
check time rose from \(4\,\mu\mathrm{s}\) to \(93\,\mu\mathrm{s}\), and the
95th percentile from \(8\,\mu\mathrm{s}\) to \(143\,\mu\mathrm{s}\).  Building
the corresponding fixtures took \(157\,\mu\mathrm{s}\) and
\(124{,}882\,\mu\mathrm{s}\), respectively.  These measurements were taken on
an Apple M2 Pro with OCaml 4.14.1; they exclude Herd, provider, network, and
effect time.

\subsection{Validation and limits}
\label{sec:implementation-validation}

The supplied Z3 script~\cite{DBLP:conf/tacas/MouraB08} checks 110 finite Passport
judgments, verdicts, and rejection reasons by refuting their negations under
67,674 satisfiable ground axioms. It checks that four closure iterations reach
a fixpoint, and that deleting the single \(\mathsf{Invalidates}\) tuple makes
exactly the three invalidation-dependent judgments fail. These results concern
the encoded instance; correspondence with the manuscript is not mechanised.

Lean 4.24.0 checks the theorem encodings using its core library. The only
axioms are \texttt{propext}, \texttt{Classical.choice}, and
\texttt{Quot.sound}; the project README states the coverage. The finite
instance is checked separately. The query-preservation encoding uses the
same rule closure and quotient-rule construction to prove the execution
equivalence, coarsest permitted quotient, and negation compatibility.
The unary algorithm and binary-rule complexity reduction have paper proofs.
The runtime harness uses its supplied representation; it does not implement
quotient synthesis.

The \texttt{cat} reference test checks relation cases against the implementation
contract \path{passport-submit.cat}, rather than the listing printed here. The paper
Makefile runs these checks through \texttt{make check}. The harness tests also
exercise partial orders, stale and concurrent evidence, claim and acceptance
judgments, contract disagreement, candidate selection, and one-use effects.

The evaluations establish conformance, checker cost, and the displayed runtime
outcomes for synthetic cases. They do not validate the application-supplied
semantic inputs or estimate deployment failures, false rejections, or task
success. No live model was used. Passport has no government submission
connection; Travel uses local sandbox adapters, and Minute Maid is a bounded
prototype.

\section{Related Work\label{sec:related-work}}

\paragraph{Argumentation and evidence.}
Dung's acceptability semantics~\cite{DBLP:journals/ai/Dung95} and
assumption-based argumentation~\cite{DBLP:journals/ai/BondarenkoDKT97} define
acceptability through attacks between arguments. Our monotone rule closure has no defeat or
extension selection (Section~\ref{sec:knowledge-base}). The system semantics
additionally relates message assertions to sources and records the evidence,
authority, and effects of a proposed event. An argumentation-based component
would still need to discharge these interface obligations.

Doyle's truth-maintenance system tracks dependencies of beliefs on their
justifications~\cite{DBLP:journals/ai/Doyle79}; PROV-DM records entities,
activities, and derivations~\cite{MoreauMissier2013-PROV-DM}. Our invalidation
predicate concerns whether an earlier event remains a witness for a later
step. It need not retract the earlier claim: a claim about the earlier draft
can remain true when its check no longer establishes readiness of the edited
draft. Source provenance and present support are therefore separate checks.

\paragraph{Completeness and query preservation.}
Giacobazzi et al.\ study construction of complete abstract
interpretations~\cite{DBLP:journals/jacm/GiacobazziRS00}.
Ranzato and Tapparo characterise strongly preserving domains and the coarsest
strongly preserving partitions~\cite{RanzatoTapparo2007}.
The contextual profiles of Section~\ref{sec:query-preserving} use this
preservation principle. Theorem~\ref{th:optimal-query-quotient} connects them
to execution of the original rules after quotienting and to readback of
designated concrete claims. Its proof accounts for the repeated interaction
of rule closure and class saturation.

Oetsch et al.\ compare projected program outcomes under every fact
extension~\cite{OetschTompitsWoltran2007}; contextual query equivalence is a
definite-Horn instance of that notion. Horn explanation relevance is also
closely related~\cite{EiterGottlob1995}. Theorem~\ref{th:query-complexity}
establishes the unary polynomial construction and coNP-completeness for the
particular quotient decision, already with acyclic binary rules, a single
signed query, and one permitted merge and its negative counterpart.

\paragraph{Generation and retrieval.}
LMQL constrains language-model queries and their outputs through a programming
language~\cite{DBLP:journals/pacmpl/Beurer-Kellner023}. Retrieval-augmented
generation combines parametric memory with retrieved documents
\cite{DBLP:conf/nips/LewisPPPKGKLYR020}. In our semantics, selected messages
contribute to \(K_E\) and extracted source claims to \(K_S\); neither
establishes \(K_U\). Lahiri studies formalising user intent for verification
\cite{DBLP:conf/fmcad/Lahiri24}. Denotation, assertion, and policy are supplied
inputs here, so the framework does not solve that formalisation problem.

\paragraph{Structured agent state.}
Intentest stores penetration-testing state in a persistent Fact--Intent DAG,
with verified network states as fact nodes and exploration directions bounded
by predecessor facts~\cite{wang2026staying}. These play roles similar to trace
witnesses and required claims. Our semantics separates the three knowledge
bases and checks whether later events invalidate earlier evidence, including
when a historical fact remains true but is no longer a current witness.

\paragraph{Enforcement and permissions.}
Schneider's security automata enforce safety policies by stopping prohibited
actions~\cite{DBLP:journals/tissec/Schneider00}; Ligatti et al.'s edit automata
also insert and suppress actions~\cite{DBLP:journals/ijisec/LigattiBW05}.
Our safeguard uses suppression. Section~\ref{sec:enforcement} applies this
standard enforcement model to the reliance checks and separates policy
adequacy, transparency, available observations, and computability. The
observation condition is related to supervisory-control observability in
Section~\ref{sec:enforcement}; selecting sufficient services yields the finite
Set Cover problem stated there.

AC4A enforces application-defined resource permissions for API and browser
agents~\cite{sharma2026ac4a}. Given a mapping from our events and resources to
its actions and resources, its permission decision can be a conjunct of
\(\mathsf{AllowedEvent}\), with permission evidence checked by
\(\mathsf{AuthorityOK}\). The reliance checks then add obligations for claim
support, current witnesses, and actuation. Two suppression monitors retain
the intersection of their admitted transitions; transparency requires both
to accept the step. This composition does not establish the adequacy of either
application policy.

\section{Conclusion}

The semantics distinguishes claim reliability from justified use of claims and
evidence in a system step. It gives specifications for that use and the
obligations that runtime checks, proofs, and analyses must establish. In the
passport instance, a successful readiness check remains a recorded result
after an edit; accepting a later action requires evidence for the current
draft.

The results support two design decisions. Given permitted representation
loss and designated claims, the coarsest query-preserving quotient retains
exactly the identifications compatible with their concrete justification
statuses. This quotient can be constructed in polynomial time for unary
rules, while even a single merge decision is coNP-complete for acyclic binary
rules. Global stability separately characterises preservation of all
conclusion classes. For execution, an observation-based guard can accept
every acceptable case only when its services distinguish worlds requiring
different decisions; otherwise a sound guard must reject some acceptable
cases.

The guarantees remain relative to application-supplied knowledge or sound
bounds, denotation, source support, and policy. The finite instance and harness
exercise the resulting checks; their validation does not establish these
inputs or deployment effectiveness.


\bibliographystyle{ACM-Reference-Format}
\bibliography{bib}

\newpage
\setcounter{page}{1}
\thispagestyle{empty}
\fancyhead[RE]{\textup{J.\ Alglave \& P.\ Cousot}}
\fancyhead[LO]{\textup{Ceci n’est pas une pipe: A semantics of justified reliance for AI systems (auxiliary material)}}
\begin{center}
{\LARGE\bfseries Auxiliary material} \Large for\\[0.5ex]
{\LARGE\bfseries  ``Ceci n’est pas une pipe: A semantics of justified reliance for AI systems''}\\[1ex]
{\Large\sffamily Jade Alglave} and {\Large\sffamily Patrick Cousot}
\end{center}

\appendix
\section{History failure diagnostics}\label{sec:history-diagnostics}

\(\mathsf{diag}\) identifies the failed current-time condition. The predicates
below use provenance to identify an earlier source, premise, or event; more than
one may hold. The added-hypothesis diagnosis additionally requires a record of
the derivation premises. Acceptance does not require that diagnostic input.

\paragraph{Source witness.}
An extraction record relates a source item to the source pointer and event from
which it was obtained:
\[
\begin{aligned}
\mathsf{ExtractedFrom}_{t_0}(s,\mathit{sp};e_s)
\Rightarrow{}&
e_s\in\textit{Ev}
\wedge s\in\mathsf{SourceItems}
\\[-0.5ex]
&{}\wedge \mathit{sp}\in\textit{Src}_{t_0}\wedge \mathit{sp}\in e_s.\mathsf{witness}
\\[-0.5ex]
&{}
\wedge \mathsf{denotation}_{\textit{Lang}}(e_s.\mathsf{message})=s.
\end{aligned}
\]
Define
\[
\mathsf{SourceWitness}_{t_0}(\mathit{sp},s;\textit{trace})
\triangleq
\exists e_s\in\textit{trace}_{\leq t_0}\mathrel{.}
\mathsf{ExtractedFrom}_{t_0}(s,\mathit{sp};e_s){}\wedge s\in F_{S,\textit{Src}_{t_0},t_0}.
\]
Thus \(\mathsf{SourceWitness}\) holds only when the trace contains an
event linking the source pointer to an extracted source item, and that item is
one of the source facts at that time.

\paragraph{Stale source.}
Given a source set $\Sigma$, a
source is stale at time \(t\) for an asserted claim when it witnessed positive universal
status  at earlier time \(t_0<t\), remains positive in the source-derived knowledge base, but the
claim no longer has positive universal status:
\[
\begin{aligned}
&\mathsf{StaleSource}_t
(m,\mathit{sp},s,c;\textit{ctx},\textit{trace},\Sigma)
\triangleq
\Sigma\subseteq\bigcup_u\textit{Src}_u
\wedge\mathsf{Asserts}_{E,\textit{ctx},t}(m,c)
\wedge \mathsf{denotation}_{\textit{Lang}}(m)=s\\[-0.75ex]
&\quad{}\wedge\mathsf{SourceSupport}_t(s,c)
\wedge \exists t_0<t\mathrel{.}
\bigl(\mathsf{SourceWitness}_{t_0}(\mathit{sp},s;\textit{trace})
\wedge\mathsf{SourceSupport}_{t_0}(s,c)\\[-0.5ex]
&\quad{}
\wedge K_{U,t_0}\vdash^+_{t_0}c\bigr)\wedge K_{S,\Sigma,t}\vdash^+_t s
\wedge\neg(K_{U,t}\vdash^+_t c).
\end{aligned}
\]

\paragraph{Refuted source.}
A source item may retain positive source status after accepted domain knowledge
has refuted the claim attributed to it:
\[
\begin{aligned}
&\mathsf{RefutedSource}_t(m,s,c;\textit{ctx},\Sigma)
\triangleq
\Sigma\subseteq\bigcup_u\textit{Src}_u
\wedge\mathsf{Asserts}_{E,\textit{ctx},t}(m,c)
\wedge \mathsf{denotation}_{\textit{Lang}}(m)=s\\[-1ex]
&\qquad{}\wedge\mathsf{SourceSupport}_t(s,c)
\wedge K_{S,\Sigma,t}\vdash^+_t s
\wedge K_{U,t}\vdash^-_t c.
\end{aligned}
\]

\paragraph{Added hypothesis.}
Let \(\mathsf{UsedPremises}_{E,\textit{ctx},t}(m,c,P)\) record that the
derivation actually used in asserting \(c\) through \(m\) has semantic premise
set \(P\subseteq\bar U\).  Define essential dependence by
\[
\begin{aligned}
&\mathsf{EssentialPremise}_{E,\textit{ctx},t}(m,c,h)
\triangleq
\exists P\subseteq\bar U\mathrel{.}
\mathsf{UsedPremises}_{E,\textit{ctx},t}(m,c,P)
\wedge h\in P{}
\\[-0.5ex]
&\qquad{}
\wedge c\in\textsf{Infer}^{\ast}_{\bar R_t}(P)
\wedge c\notin\textsf{Infer}^{\ast}_{\bar R_t}(P\setminus\{h\}).
\end{aligned}
\]
A premise is essential to the recorded derivation when the rules
derive the claim from the recorded premises but no longer derive it after that
premise is removed.
Let
\[
\mathcal R_{t,\textit{ctx},\Sigma}
\triangleq
\{d\in\bar U\mid
\exists m_d\in\textit{Lang}\mathrel{.}
\mathsf{ReliableClaim}_t(m_d,d;
K_{E,\textit{ctx},t},K_{S,\Sigma,t},K_{U,t},\textit{ctx})\}.
\]
\(\mathsf{AddedHypothesis}\) holds when the recorded derivation needs \(h\),
\(h\) is not among the reliable claims, adding \(h\) to those claims makes
\(c\) derivable, \(c\) is not derivable from them alone, and current support
for \(c\) fails:
\[
\begin{aligned}
&\mathsf{AddedHypothesis}_t(h,c;m,\textit{ctx},\Sigma)
\triangleq
\Sigma\subseteq\bigcup_u\textit{Src}_u
\wedge\mathsf{Asserts}_{E,\textit{ctx},t}(m,c)
\wedge h\in\bar U\wedge \mathsf{EssentialPremise}_{E,\textit{ctx},t}(m,c,h)\\[-0.75ex]
&\qquad{}
\wedge h\notin\mathcal R_{t,\textit{ctx},\Sigma}
\wedge c\in\textsf{Infer}^{\ast}_{\bar R_t}
(\mathcal R_{t,\textit{ctx},\Sigma}\cup\{h\})
\wedge c\notin\textsf{Infer}^{\ast}_{\bar R_t}
(\mathcal R_{t,\textit{ctx},\Sigma})
\wedge\neg\mathsf{SupportedClaim}_t
(m,c;K_{U,t},K_{S,\Sigma,t}).
\end{aligned}
\]
When no derivation or provenance record is available, this diagnostic does not
hold merely because a suitable hypothesis could be invented afterwards.

\paragraph{Unsupported use.}
Let \(e\) extend \(\textit{trace}\) at time \(t\), and let
\(\textit{ctx}_t\in\mathsf{Select}_t(\textit{trace})\) and
\(\textit{ctx}_{t+1}\in\mathsf{Select}_{t+1}(\textit{trace}\cdot e)\) be the
contexts used before and after that extension.  An unsupported use occurs when
the event makes a message available for use while its asserted claim remains
unsupported:
\[
\begin{aligned}
&\mathsf{UnsupportedUse}_{t}
(e,m,s,c;\textit{trace},\textit{ctx}_t,\textit{ctx}_{t+1},\Sigma)
\triangleq
\Sigma\subseteq\bigcup_u\textit{Src}_u
\wedge e.\mathsf{message}=m
\wedge\mathsf{denotation}_{\textit{Lang}}(m)=s\\[-0.75ex]
&\qquad{}\wedge\neg\bigl(K_{E,\textit{ctx}_t,t}\vdash_t^+m\bigr)
\wedge K_{E,\textit{ctx}_{t+1},t+1}\vdash_{t+1}^+m
\wedge \mathsf{Asserts}_{E,\textit{ctx}_{t+1},t+1}(m,c)\wedge\neg\mathsf{SupportedClaim}_{t+1}
(m,c;K_{U,t+1},K_{S,\Sigma,t+1}).
\end{aligned}
\]

\begin{theorem}[Soundness of provenance diagnoses]
\label{th:trace-soundness}
For every \(\Sigma\subseteq\bigcup_u\textit{Src}_u\),
\(\mathsf{StaleSource}_t
(m,\allowbreak \mathit{sp},\allowbreak s,\allowbreak c;\allowbreak \textit{ctx},\allowbreak \textit{trace},\allowbreak \Sigma)\),
\(\mathsf{RefutedSource}_t(m,s,c;\textit{ctx},\Sigma)\), and
\(\mathsf{AddedHypothesis}_t(h,c;m,\textit{ctx},\Sigma)\) each imply
\(\mathsf{RelianceFailure}_{t,\textit{ctx}}(m,c;\Sigma)\).
For every \(\Sigma\subseteq\bigcup_u\textit{Src}_u\),
\(\mathsf{UnsupportedUse}_{t}
(e,m,s,c;\textit{trace},\allowbreak \textit{ctx}_t,\allowbreak \textit{ctx}_{t+1},\allowbreak \Sigma)\)
implies
\(\mathsf{RelianceFailure}_{t+1,\textit{ctx}_{t+1}}(m,c;\Sigma)\).
\end{theorem}
\noindent\emph{Informally.} Each predicate entails a reliance failure: stale or
refuted universal support, an unestablished essential premise, or a newly
available unsupported message falsifies one conjunct of
\(\mathsf{ReliableClaim}\).

\section{Proofs}

\input{query-preservation-proofs}
\input{query-complexity-proofs}

\begin{proof}[Proof of Lemma \ref{lem:closure} on closure]
The last claim is Tarski's characterisation of the least fixpoint as the
least pre-fixpoint. Extensivity holds since
\(B\subseteq\textsf{Infer}_{R}(B)(X)\). If \(B\subseteq B'\), then
\(C'=\textsf{Infer}^{\ast}_{R}(B')\) satisfies
\(\textsf{Infer}_{R}(B)(C')\subseteq\textsf{Infer}_{R}(B')(C')=C'\), so
\(\textsf{Infer}^{\ast}_{R}(B)\subseteq C'\). For idempotence, let
\(C=\textsf{Infer}^{\ast}_{R}(B)\); every rule whose premises lie in \(C\)
concludes in \(C\), so \(\textsf{Infer}_{R}(C)(C)=C\) and
\(\textsf{Infer}^{\ast}_{R}(C)\subseteq C\); the converse is extensivity.
\end{proof}

\begin{proof}[Proof of Lemma~\ref{lem:ignore-sound} on ignorance under-approximates]
Let \(X=\textsf{Infer}^{\ast}_{\bar R_t}(B)\cap\overline{\textit{Scope}}\). Then
\(B\subseteq X\), since \(B\subseteq\overline{\textit{Scope}}\). If
\(\frac{P}{c}\in\overline{\textit{ignore}}(\bar R_t)\) and \(P\subseteq X\), then
\(c\in\overline{\textit{Scope}}\) by definition of \(\overline{\textit{ignore}}\),
and \(c\in\textsf{Infer}^{\ast}_{\bar R_t}(B)\) since that set is closed under
\(\bar R_t\). So \(X\) is a pre-fixpoint and Lemma~\ref{lem:closure} applies.
\end{proof}

\begin{proof}[Proof of Theorem~\ref{th:finite-stability} on deciding stability on a finite instance]
There are finitely many rules and finitely many replacement choices. Their
closures are computable by iteration over the finite scope. If the test
succeeds, take any \(P'\) from Definition~\ref{def:stable} and choose in
\(P'\) one replacement for each premise. The tested conclusion belongs to the
closure of those choices and hence, by monotonicity, to the closure of \(P'\).
The converse follows by applying stability to each enumerated choice.
\end{proof}

\begin{proof}[Proof of Theorem~\ref{th:forget-transfer} on forgetfulness transfer]
Write \(C=\textsf{Infer}^{\ast}_{R}(B)\) and
\(C^\sharp=\textsf{Infer}^{\ast}_{\overline{\textit{forget}}(R)}(\alpha_f(B))\).
The abstract rule for \(\frac{P}{c}\) has premises \(\alpha_f(P)\) and
conclusion \([c]_{\equiv_f}\).

(i) Let \(X=\gamma_f(C^\sharp)=\{x\mid[x]_{\equiv_f}\in C^\sharp\}\). Then
\(B\subseteq X\), since \(\alpha_f(B)\subseteq C^\sharp\). If
\(\frac{P}{c}\in R\) and \(P\subseteq X\), then \(\alpha_f(P)\subseteq
C^\sharp\), so the abstract rule fires in \(C^\sharp\) and \(c\in X\). By
Lemma~\ref{lem:closure}, \(C\subseteq X\), hence \(\alpha_f(C)\subseteq
\alpha_f(\gamma_f(C^\sharp))=C^\sharp\).

(ii, if) Assume stability, and let \(Y=\alpha_f(C)\). Then
\(\alpha_f(B)\subseteq Y\). Let \(\frac{P}{c}\in R\) with
\(\alpha_f(P)\subseteq Y\): each \(p\in P\) has some \(p'\in C\) with
\(p'\equiv_f p\); collect these into \(P'\subseteq C\). Stability gives
\(c'\equiv_f c\) with \(c'\in\textsf{Infer}^{\ast}_{R}(P')\subseteq
\textsf{Infer}^{\ast}_{R}(C)=C\) by Lemma~\ref{lem:closure}, so
\([c]_{\equiv_f}\in Y\). Thus \(Y\) is a pre-fixpoint of the abstract
operator and \(C^\sharp\subseteq Y\); with (i), \(C^\sharp=Y\).

(ii, only if) Let \(\frac{P}{c}\in R\) and \(P'\) as in
Definition~\ref{def:stable}, and take \(B=P'\). Then \(\alpha_f(P)\subseteq
\alpha_f(B)\), so \([c]_{\equiv_f}\in C^\sharp=\alpha_f(\textsf{Infer}^{\ast}_{R}(P'))\),
which yields the required \(c'\).
\end{proof}

\begin{proof}[Proof of Theorem~\ref{th:successive-forgetfulness} on successive forgetfulness]
\textit{(a)}\quad Equation~(\ref{eq:successive-quotients}) follows directly from the definitions:
both sides send \(x\) to \([x]_{\equiv_2}\), and they do the same to every
premise and conclusion of a rule.

\textit{(b)}\quad For stability, fix \(P/c\in R\) and an arbitrary \(P'\subseteq S\) such that
every \(p\in P\) has some \(p'\in P'\) with \(p'\equiv_2p\). For each premise
\([p]_{\equiv_1}\) of the first-stage rule, \(\alpha_1(P')\) therefore contains
some \([p']_{\equiv_1}\approx[p]_{\equiv_1}\). Second-stage stability gives a
class \(a\approx[c]_{\equiv_1}\) in
\(\textsf{Infer}^{\ast}_{\overline{\textit{forget}}_1(R)}(\alpha_1(P'))\).
First-stage stability and Theorem~\ref{th:forget-transfer}.(ii) make this closure
equal to \(\alpha_1(\textsf{Infer}^{\ast}_{R}(P'))\). Hence
\(a=[c']_{\equiv_1}\) for some
\(c'\in\textsf{Infer}^{\ast}_{R}(P')\), and
\(a\approx[c]_{\equiv_1}\) gives \(c'\equiv_2c\). Thus \(R\) is
\(\equiv_2\)-stable.

Conversely, assume \(R\) is \(\equiv_2$-stable, \(\alpha_1(P)/[c]_{{\equiv_1}}\) is an abstract rule  with \(P/c\in R\), and
$Y'\subseteq Q_1\) such that every \([p]_{{\equiv_1}}\) with \(p\in P\) has some \(a\in Y'\) with
$a\approx[p]_{{\equiv_1}}\). Put \(P'=\gamma_1(Y')\). Each such \(a\) is \([p']_{{\equiv_1}}\) for some \(p'\in P'\),
and \(a\approx[p]_{{\equiv_1}}\) gives \(p'\equiv_2 p\). By \(\equiv_2$-stability there is \(c'\equiv_2 c\) with
$c'\in\textsf{Infer}^{\ast}_R(P')\). By \(\equiv_1$-stability and Theorem~\ref{th:forget-transfer}.(ii).(b) applied to $R$ and $\equiv_1$,
\[
[c']_{{\equiv_1}}\in\alpha_1(\textsf{Infer}^{\ast}_R(P'))=\textsf{Infer}^{\ast}_{\overline{\textit{forget}}_1(R)}(\alpha_1(P'))=\textsf{Infer}^{\ast}_{\overline{\textit{forget}}_1(R)}(Y'),
\]
and \([c']_{{\equiv_1}}\approx[c]_{{\equiv_1}}\). Hence \(\overline{\textit{forget}}_1(R)\) is \(\approx$-stable.

\textit{(c)}\quad Apply Theorem~\ref{th:forget-transfer}(ii) first to \(R\) and \(\equiv_1\),
then to \(\overline{\textit{forget}}_1(R)\) and \(\approx\). Together with
Equation~(\ref{eq:successive-quotients}), these two applications give
Equation~(\ref{eq:successive-transfer}).

%
% The following counter example has been incorporated in the theorem.

\textit{(d)}\quad  The first stability premise alone is not sufficient. Take all the claims below
and their negations, with rules
\[
\mathsf{observed}_1\longrightarrow\mathsf{candidate}_1,
\qquad
\mathsf{observed}_2\longrightarrow\mathsf{candidate}_2,
\qquad
\mathsf{checked}\longrightarrow\mathsf{usable}.
\]
Let \(\equiv_1\) merge the two \(\mathsf{observed}\) claims and the two
\(\mathsf{candidate}\) claims, and make the same two identifications between
their negations. Leave every other class as a singleton. The rules are
\(\equiv_1\)-stable: the first two rules match under either representative, and
the third uses singleton classes. Let \(\equiv_2\) additionally merge both
\(\mathsf{candidate}\) claims with \(\mathsf{checked}\), and likewise merge
their three negations, with no other new identifications.

The once-merged rule
\([\mathsf{checked}]_{\equiv_1}\rightarrow
[\mathsf{usable}]_{\equiv_1}\) witnesses failure of second-stage stability.
Its premise is \(\approx\)-equivalent to
\([\mathsf{candidate}_1]_{\equiv_1}\), but the first-stage rules derive no
class \(\approx\)-equivalent to \([\mathsf{usable}]_{\equiv_1}\) from that
candidate class. After the second merge, however, the rules derive
\([\mathsf{usable}]_{\equiv_2}\) from
\([\mathsf{observed}_1]_{\equiv_2}\); the original rules do not derive
\(\mathsf{usable}\) from \(\mathsf{observed}_1\).

\textit{(e)}\quad Take \(R=\{a\longrightarrow b\}\) and let \({\equiv_1}={\equiv_2}\) merge \(a\) with \(a'$. The negations are merged accordingly, so $\equiv_i$ remains a congruence for $\neg$. Then \(\approx\) is the
identity on \(Q_1\), and every rule system is stable for the identity: each \(P'\) in Definition~\ref{def:stable}
contains \(P\), so its closure contains \(c\). Hence \(\overline{\textit{forget}}_1(R)\) is \(\approx\)-stable. But \(R\) is not
\(\equiv_2\)-stable: with \(P'=\{a'\}\) the closure is \(\{a'\}\), while \([b]_{{\equiv_2}}=\{b\}\).

\textit{(f)}\quad  Take \(R=\{a\longrightarrow b\}\), with \(\equiv_1\) merging \(a\) and \(a'\), and \(\equiv_2\) merging \(a\), \(a'\) and \(b\). The negations are merged accordingly, so $\equiv_i$ remains a congruence for $\neg$.
Then \(R\) is not \(\equiv_1\)-stable, by the replacement \(P'=\{a'\}\) as in (e). 
It is \(\equiv_2\)-stable: any admissible \(P'\) contains some \(p'\equiv_2 a\), and 
\(p'\in\textsf{Infer}^{\ast}_R(P')\) with \(p'\equiv_2 b\).
\end{proof}

\begin{proof}[Proof of Corollary~\ref{cor:status-transfer} on status transfer]
By Theorem~\ref{th:forget-transfer}(i),
\(K^\sharp\not\vdash[z]_{\equiv_f}\) implies \(K\not\vdash z'\) for every
\(z'\equiv_f z\). By congruence, \(y\equiv_f x\) implies \(\neg y\equiv_f\neg x\).

(a)\quad \(K^\sharp\not\vdash[x]_{\equiv_f}\) gives \(K\not\vdash y\), and
\(K^\sharp\not\vdash[\neg x]_{\equiv_f}\) gives \(K\not\vdash\neg y\).

(b)\quad \(K^\sharp\not\vdash[\neg x]_{\equiv_f}\) gives \(K\not\vdash\neg y\) for
every \(y\equiv_f x\). Under stability,
\(K^\sharp\vdash[x]_{\equiv_f}\) and
Theorem~\ref{th:forget-transfer}(ii)
give some \(y\equiv_f x\) with \(K\vdash y\), hence \(K\vdash^{+}y\). For \(-\),
apply the \(+\) case to \(\neg x\) (using \(\neg\neg x=x\)) to get
\(z\equiv_f\neg x\) with \(K\vdash^{+}z\); then \(y=\neg z\equiv_f x\) and
\(K\vdash^{-}y\).

(c)\quad is Theorem~\ref{th:forget-transfer}(i) applied to \(x\).

(d)\quad Take \(S_0=\{y_1,y_2,\neg y_1,\neg y_2\}\), and let
\(\equiv_0\) identify \(y_1\) with \(y_2\) and \(\neg y_1\) with
\(\neg y_2\).
\(F_0=\{y_1,\neg y_2\}\), and \(R_0=\emptyset\), which is stable because it
has no rules. Let
\(K_0=\langle S_0,F_0,R_0,t\rangle\) and let
\(K_0^\sharp\) be its quotient abstraction. Then
\(K_0\vdash^{+}y_1\) and \(K_0\vdash^{-}y_2\), while the
quotient facts \(\{[y_1]_{\equiv_0},[\neg y_1]_{\equiv_0}\}\) give
\(K_0^\sharp\vdash^{\pm}[y_1]_{\equiv_0}\).
\end{proof}

\begin{proof}[Proof of Corollary~\ref{cor:rs-sound} on the soundness of \(R_{S,t}\)]
By Theorem~\ref{th:forget-transfer}(ii) for
\(R=\overline{\textit{ignore}}(\bar R_t)\), monotonicity of \(\alpha_f\), and
Lemma~\ref{lem:ignore-sound}:
\(\textsf{Infer}^{\ast}_{R_{S,t}}(\alpha_f(B))
=\alpha_f(\textsf{Infer}^{\ast}_{\overline{\textit{ignore}}(\bar R_t)}(B))
\subseteq\alpha_f(\textsf{Infer}^{\ast}_{\bar R_t}(B)\cap\overline{\textit{Scope}})\).
The picture example of Section~\ref{sec:symbolic-components} shows that the
inclusion can fail without stability.
\end{proof}

\begin{proof}[Proof of Corollary~\ref{cor:sound-source-facts} on sound source facts]
Choose \(B\subseteq
\textsf{Infer}^{\ast}_{\bar R_t}(\bar H_t)\cap
\overline{\textit{Scope}}\) with
\(\alpha_f(B)=F_{S,\textit{Src},t}\). Corollary~\ref{cor:rs-sound} gives
\[
\textsf{Infer}^{\ast}_{R_{S,t}}(F_{S,\textit{Src},t})
\subseteq
\alpha_f\bigl(\textsf{Infer}^{\ast}_{\bar R_t}(B)\cap
\overline{\textit{Scope}}\bigr).
\]
Since \(B\subseteq\textsf{Infer}^{\ast}_{\bar R_t}(\bar H_t)\), closure
monotonicity and idempotence put the right-hand side inside
\(\alpha_f(\textsf{Infer}^{\ast}_{\bar R_t}(\bar H_t)\cap
\overline{\textit{Scope}})\).
\end{proof}

\begin{proof}[Proof of Theorem~\ref{th:interval-status} on interval statuses]
By definition of \(\sqsubseteq\), \(K\vdash c\) is monotone and
\(K\not\vdash c\) antitone in \(K\). Each right-hand side is a conjunction of a
monotone condition evaluated at \(K_{\mathrm{lo}}\) and antitone conditions
evaluated at \(K_{\mathrm{hi}}\), so it implies the status at every \(K\) in
the interval. Conversely, \(K_{\mathrm{lo}}\) and \(K_{\mathrm{hi}}\) belong to
the interval, and each conjunct on the right is a conjunct of the status at
one of them.
\end{proof}

\begin{proof}[Proof of Theorem~\ref{th:exhaustiveness} on reliance diagnostic normal form]
The assertion contract supplies the unique source item \(s\).  Each of
\(m\), \(c\), and \(s\) has exactly one of the four statuses.  Expanding
\(\mathsf{ReliableClaim}\) leaves the conjunction
\(\sigma_E=+\), \(\sigma_U=+\), \(\sigma_S=+\), and
\(\mathsf{SourceSupport}_t(s,c)\), which gives the
\(\textsf{OK}\) case. The second equivalence follows from the definition of
\(\mathsf{RelianceFailure}_{t,\textit{ctx}}(m,c;\Sigma)\).
\end{proof}

\begin{proof}[Proof of Theorem~\ref{th:trace-soundness} on soundness of provenance diagnoses]
The first three predicates assert \(c\) at time \(t\). Staleness excludes
positive universal status; a refuted source gives negative universal status;
and \(\mathsf{AddedHypothesis}\) states that support fails.
\(\mathsf{UnsupportedUse}\) contains both the assertion and failed support at
time \(t+1\). Each therefore falsifies a conjunct of
\(\mathsf{ReliableClaim}\) at the time stated in the conclusion.
\end{proof}

\begin{proof}[Proof of Corollary~\ref{cor:scoped} on the scoped correctness transfers under stability]
Let \(B=\bigcup\{\mathsf{sem}_{\textit{ModLang}}(h)\mid h\in
H_{\textit{ModLang}}\}\). Each \(\mathsf{sem}_{\textit{ModLang}}(h)\) is a class,
so \(\alpha_f(B)=\{\mathsf{sem}_{\textit{ModLang}}(h)\mid h\in
H_{\textit{ModLang}}\}\). Scoped correctness gives
\(\mathsf{sem}_{\textit{ModLang}}(c)\in\textsf{Infer}^{\ast}_{R_{S,t}}(\alpha_f(B))\),
and Corollary~\ref{cor:rs-sound} gives a member of it in
\(\textsf{Infer}^{\ast}_{\bar R_t}(B)\).
\end{proof}

\begin{proof}[Proof of Corollary~\ref{cor:interval-accept} on bounded acceptance]
By the first line of
Theorem~\ref{th:interval-status}, the replacement predicate implies
\(K_{U,t}\vdash^+_t c\) for every \(K_{U,t}\) in the interval,
and hence for every member satisfying the displayed rule condition.
\(\mathsf{AcceptOK}\) uses the predicate \(c\mapsto K_{U,t}\vdash^+_t c\)
only through conjunction, universal and existential quantification, and
implication with that predicate only in the conclusion, so it is monotone
in that predicate.
\end{proof}

\begin{proof}[Proof of Theorem \ref{th:valid-accepted-controlled-transitions} on valid steps for accepted controlled transitions]\hfill
From\linebreak \(\mathsf{ValidStep}\) we get the recorded orchestrator step and the
accepted safeguard result.  Accept-soundness gives
\[
\mathsf{AcceptOK}_{\textsf{Spec},e.\mathsf{time},e.\mathsf{context}}
(\textit{trace},w,e,w').
\]
By \(\mathsf{AcceptAdequate}_{\textsf{Spec}}\), this certificate entails
\[
\mathsf{AllowedEvent}_{\textsf{Spec}}(\textit{trace},w,e).
\]
Instantiating \(\mathsf{StepOK}\) with the same recorded step gives
\(\mathsf{InterfaceOK}_{e.\mathsf{time}}\), so the event record agrees with its
producer result.
\end{proof}

\begin{proof}[Proof of Lemma~\ref{lem:reasons-characterize} on reasons characterize failed checks]
\leavevmode\par\noindent
The \(\mathsf{WitnessCheck}\) case is equation~(\ref{eq:tag:W}). The
\(\mathsf{AuthorityCheck}\) and \(\mathsf{ActuationCheck}\) cases are
equations~(\ref{eq:tag:(A)}) and~(\ref{eq:tag:X}). For
\(\mathsf{RequiredCheck}\), the conjuncts
\(e.\mathsf{time}=t\) and \(e.\mathsf{context}=\textit{ctx}\) of
\(\mathsf{RequiredClaimsOK}\) hold by the choice of indices, so it fails
exactly when the context is not selected or some required claim has no
reliable message. For \(\mathsf{ValidityCheck}\), the same two conjuncts of
\(\mathsf{ReliableAssertions}\) hold, so \(\mathsf{Validity}\) fails exactly
when an interpretation is undefined, the context is not selected, or
\(m\neq m_{\mathsf{noClaim}}\) and \(m\) has no denotation, no asserted claim,
or an asserted claim \(c\) that is not reliable. In the last case
Theorem~\ref{th:exhaustiveness} gives
\(\mathsf{diag}_{t,\textit{ctx}}
(m,c;e.\mathsf{sources})\neq\textsf{OK}\).
\end{proof}

\begin{proof}[Proof of Theorem \ref{th:diagnostic-adequacy} on rejection adequacy]
\leavevmode\par\noindent
A value in \(\mathsf{Reject}_{\textsf{Spec}}\) has the form
\(\mathsf{RejectFailed}(k,\pi)\) for some
\(k\in\mathsf{AcceptanceChecks}\) and reason \(\pi\).  Reject-soundness gives
\(\mathsf{Reason}_{\textsf{Spec}}(k,\pi;\textit{trace},w,e,w')\), hence
\(\mathsf{FailedCheck}_{\textsf{Spec}}(k;\textit{trace},\allowbreak w,\allowbreak e,\allowbreak w')\) by
Lemma~\ref{lem:reasons-characterize}.  The
equivalence (\ref{eq:tag:F}) in Section~\ref{sec:rej-diag} then gives
\[
\neg\mathsf{AcceptOK}_{\textsf{Spec},e.\mathsf{time},e.\mathsf{context}}
(\textit{trace},w,e,w').
\]
\end{proof}

\begin{proof}[Proof of Theorem \ref{th:safeguard-decision} on safeguard decision]
Left to right is
\(\mathsf{SafeguardAcceptSound}_{\textsf{Spec}}
(\mathsf{Safeguard}_{\textsf{Spec}})\).
Right to left: suppose
\(\mathsf{AcceptOK}_{\textsf{Spec},
e.\mathsf{time},e.\mathsf{context}}
(\textit{trace},w,e,w')\).  By totality the safeguard returns
\(\mathsf{Accept}(e,w')\) or a value in \(\mathsf{Reject}_{\textsf{Spec}}\);
the latter is impossible, since Theorem~\ref{th:diagnostic-adequacy} would give
\(\neg\,\mathsf{AcceptOK}_{\textsf{Spec},
e.\mathsf{time},e.\mathsf{context}}\allowbreak
(\textit{trace},w,e,w')\).
The rejection equivalence follows from this acceptance
equivalence and totality, since the two result forms are disjoint.
\end{proof}

\begin{proof}[Proof of Theorem~\ref{th:run-enforcement} on run-level enforcement]
(i) Membership in \(\mathcal P_{\textsf{Spec}}\) is a conjunction of one
condition per controlled step \(k\), each depending only on \(z_k\) and
\(e_k\). A run violates the policy if and only if some condition fails; the prefix ending
at \(z_{k+1}\) for the least such \(k\) violates it and ends with the
controlled step \(k\). Uncontrolled steps add no condition.

(ii) By induction on run length. Uncontrolled steps add no condition. A
controlled step \(k\) of \({\longrightarrow}^{G}\) is a step of
\(\longrightarrow\) appending \(e_k\), so
\(\mathsf{Step}_{\textit{Orch},t_k}(s_k,\vec s_k,\textit{trace}_k,w_k,e_k,
s_{k+1},\vec s_{k+1},w_{k+1})\);
\(\mathsf{StepOK}\) gives
\(\mathsf{InterfaceOK}_{t_k}(z_k,e_k,z_{k+1})\), hence
\(e_k.\mathsf{time}=t_k\). Together with
\(\mathsf{SystemControlled}(e_k)\) and the accepting call, the step is
\(\mathsf{ValidStep}_{\textsf{Spec}}\) for the recorded context.
Theorem~\ref{th:valid-accepted-controlled-transitions} gives
\(\mathsf{AllowedEvent}_{\textsf{Spec}}(\textit{trace}_k,\allowbreak w_k,e_k)\) and the
interface agreement.
\end{proof}

\begin{proof}[Proof of Proposition~\ref{prop:cat-adequacy} on the \texttt{cat} trace-shape abstraction]
Since
\(\mathsf{CurrentWitnessesOK}_{\textsf{Spec}}(\textit{trace},w;e)\) holds,
every claim in
\(\mathsf{NeedsWitness}_{\textsf{Spec}}(\textit{trace},\allowbreak w,e)\) has a current
trace witness.  In this instance these are \(\rho_{\mathsf{UK}}\) and the
picture, form, and confirmation claims for \(d\).  The naming assumption puts
the witnessing events, which \(e\) cites as event references, in the
corresponding support relations, which therefore
satisfy the four \texttt{requires} checks.  By \(\mathsf{TraceWitness}\), every
supporting event occurs in
\(\textit{trace}\) and is cited by \(e\), so it precedes \(e\) in \texttt{po}.
The hypothesis that every cited trace witness is current, together
with the definition of \(\mathsf{CurrentTraceWitness}\), rules out a later
event that invalidates a supporting event. Thus there is no path from an
invalidator through \texttt{invalidates} and the readiness-support relation to
\(e\), so \path{s_readiness_requires_current_form_state} holds. For
the converse, the \texttt{cat} model checks event names, not resources: at
\((\textit{trace}_0,w_1)\), the event
\(e^1_{\mathsf{ready}}\) cites the \(d_0\) events of \(\textit{trace}_0\) and
is accepted by the model, but none of those events witnesses the required
\(d_1\) claims.
\end{proof}

\begin{proof}[Proof of Theorem~\ref{th:transparency} on transparency]
The guard removes the step if and only if the safeguard rejects it. Under the hypotheses
of Theorem~\ref{th:safeguard-decision}, it accepts iff
\(\mathsf{AcceptOK}\) holds for the event's recorded context. The second claim
quantifies the first over all allowed controlled steps.
\end{proof}

\begin{proof}[Proof of Theorem \ref{th:greatest-observable-guard} on the most 
permissive sound guard relative to a supplied world set]
Fix
\[
w\in R_{\textit{trace},e,\textit{ctx}}.
\]
If
\(g_R^*(\textit{obs}_G(w))=1\), its defining universal condition includes
that \(w\), so \(A(w)\) holds. This proves soundness. If a sound
guard \(g\) accepts \(o\), soundness gives \(A(w)\) for every
\(w\in R_{\textit{trace},e,\textit{ctx}}\) with observation
\(o\), hence \(g_R^*(o)=1\). This proves maximality.
If \(w_a\) and \(w_b\) have the same observation, every guard
gives them the same decision. A guard that accepts every acceptable
case and rejects every unacceptable case can therefore exist only if
\(A(w_a)\) and \(A(w_b)\) have the same truth value. Conversely, suppose that
\(A\) is constant on every observation class. If \(A(w)\) holds, it holds at
every world in the class of \(w\), so \(g_R^*(\textit{obs}_G(w))=1\). If it does
not hold, \(w\) itself makes \(g_R^*(\textit{obs}_G(w))=0\).
Thus \(g_R^*(\textit{obs}_G(w))=1\) if and only if \(A(w)\). For a finite
list, decidable equality on \(O\) gives the observation classes, and
decidability of \(A\) determines whether every member of a class satisfies it.
If \(R\subseteq R'\), every condition checked for
\(R\) is also checked for \(R'\), which proves part~(iv).
\qedhere
\end{proof}

\begin{proof}[Proof of Theorem \ref{th:observable-safeguard} on safeguard decisions from observations]
Suppose such a safeguard exists. If \(g_R^*(o)=0\), some world in the class of
\(o\) does not satisfy \(A\), so acceptance soundness and totality force it to
reject, say with \(\mathsf{RejectFailed}(k,\pi)\). The verdict and \((k,\pi)\) are the
same throughout that observation class. Rejection soundness gives the
displayed \(\mathsf{Reason}\) for every world in the class, so
\((k,\pi)\in F(o)\).

Conversely, assume (\ref{eq:tag:O}). Return
\(\mathsf{Accept}(e,w_e(w))\) when
\(g_R^*(\textit{obs}_G(w))=1\). Otherwise return
\(\mathsf{RejectFailed}(k,\pi)\), where the pair
\((k,\pi)\in F(\textit{obs}_G(w))\) is picked by a fixed choice function on
the nonempty sets \(F(o)\). The definition of \(g_R^*\) makes every acceptance
sound, and the definition of \(F\) makes every rejection sound. This safeguard
is total, and its verdict and rejection reason are constant on each
observation class. Every member of \(F(o)\) is a reason throughout
the class of \(o\), and by (\ref{eq:tag:R}) and (\ref{eq:tag:F}) every reason
implies \(\neg A(w)\). Hence
Condition~(\ref{eq:tag:O}) excludes a class containing both a world where \(A\)
holds and a world where it does not; by
Theorem~\ref{th:greatest-observable-guard}.(b),
\(g_R^*(\textit{obs}_G(w))=1\) if and only if \(A(w)\) throughout \(R\).
Under the effective-enumeration hypotheses above, fix a rejected
class \(o\) and \(w_0\) in it. Every member of \(F(o)\) is an actual reason at
\(w_0\) and hence occurs in \(C_{w_0}\). Since \(F(o)\) is nonempty by
(\ref{eq:tag:O}), scanning \(C_{w_0}\) and testing each pair at every world in the
finite class terminates with a member of \(F(o)\).
\end{proof}

\begin{proof}[Proof of Theorem~\ref{th:selecting-observations} on selecting sufficient observations]
The equivalence of (a) and (b), and the claim about \(g_R^*\), are
Theorem~\ref{th:greatest-observable-guard}.(b).  If a pair
\((w_+,w_-)\in P\) belongs to no \(P_a\) for \(a\in G\), its two worlds have
the same \(G\)-observation although \(A(w_+)\) and \(A(w_-)\) differ, so
\(G\) is not sufficient.  Conversely, if every pair in \(P\) belongs to some
\(P_a\) with \(a\in G\), two worlds with the same \(G\)-observation cannot
have different values of \(A\).  This proves (b) if and only if (c).

A proposed \(G\) can be checked by examining every member of
\(R^+\times R^-\) and looking for a distinguishing service in \(G\), which
gives the stated bound and membership in NP.  For NP-hardness, reduce Set
Cover~\cite{Karp1972Reducibility}.  Given a universe \(U\), subsets
\(C_a\subseteq U\) indexed by \(a\in S\), and \(k\), take
\[
R=\{w_\bot\}\cup\{w_u\mid u\in U\},\qquad
\neg A(w_\bot),\qquad A(w_u)\quad(u\in U),
\]
and define binary observations by
{\abovedisplayskip0pt\[
\mathsf{observation}(a,w_\bot)=0,
\qquad
\mathsf{observation}(a,w_u)=
\begin{cases}
~1 & u\in C_a,\\[-1ex]
~0 & u\notin C_a.
\end{cases}
\]}
The opposite-decision pairs are \((w_u,w_\bot)\), and \(a\) distinguishes
that pair exactly when \(u\in C_a\).  A set \(G\) is therefore sufficient
exactly when \(\{C_a\mid a\in G\}\) covers \(U\), with the same bound \(k\).
The construction is polynomial and has one rejecting world.
\end{proof}

\end{document}

%% file: authors.tex
\author{Jade Alglave}
\email{j.alglave@ucl.ac.uk}
\orcid{0000-0002-8335-0852}
\authornote{Corresponding author.}
\affiliation{%
  \institution{ARM}
  \city{Cambridge}
  \country{UK}}
\affiliation{%
  \institution{University College London}
  \city{London}
  \country{UK}}

\author{Patrick Cousot}
\email{pcousot@cims.nyu.edu}
\orcid{0009-0003-4885-8268}
\affiliation{%
  \institution{Courant Institute, NYU}
  \city{New York}
  \state{NY}
  \country{USA}}

\renewcommand{\shortauthors}{J.~Alglave \& P.~Cousot}

%% file: query-preservation.tex
\subsubsection{Query-preserving representations}\label{sec:query-preserving}

Stability concerns every derived class. An application may also need an exact
answer about a particular concrete claim: deriving its class must justify
that claim, rather than merely some member of the class. We now construct the
coarsest permitted representation that preserves such answers. The
construction uses the completeness and strong-preservation view of
abstraction~\cite{DBLP:journals/jacm/GiacobazziRS00,RanzatoTapparo2007}; the
execution semantics remains that of the quotient rules defined above.

Let \(S=\overline{\textit{Scope}}\) be finite and closed under negation, and
let \(R\) be a finite rule system on \(S\), with decidable equality. Fix a
negation-congruent equivalence \(E_0\) specifying which identifications the
application permits. A representation \(E\) must refine \(E_0\), so
\(E\subseteq E_0\). Let \(Q\subseteq S\) be the designated concrete queries,
with \(\neg Q=Q\). Write
\[
 C\triangleq\textsf{Infer}^{\ast}_{R},\qquad
 C_E^\sharp\triangleq
 \textsf{Infer}^{\ast}_{\overline{\textit{forget}}_E(R)},
\]
where \(\alpha_E\), \(\gamma_E\), and
\(\overline{\textit{forget}}_E\) are the existing quotient maps for \(E\).
In particular, the quotient rules replace every premise and conclusion by its
class; no additional inference rules or query checks are inserted.

\begin{definition}[Exact concrete queries]\label{def:query-exact}
A negation-congruent equivalence \(E\) is \emph{\(Q\)-exact} when
\begin{equation}
 \forall B\subseteq S\;\forall q\in Q,\qquad
 q\in C(B)
 \quad\Longleftrightarrow\quad
 [q]_E\in C_E^\sharp(\alpha_E(B)).
 \label{eq:query-exact}
\end{equation}
\end{definition}

Two atoms are contextually indistinguishable for \(Q\) when adding either
one to any starting facts gives the same concrete query answers. Define
\begin{align}
 x\sim_Q y
 &\;\triangleq\;
 \forall H\subseteq S,\quad
 Q\cap C(H\cup\{x\})=Q\cap C(H\cup\{y\}),
 \label{eq:query-profile}\\
 xE^*y
 &\;\triangleq\;
 xE_0y\;\wedge\;x\sim_Qy\;\wedge\;
 \neg x\sim_Q\neg y.
 \label{eq:query-congruence}
\end{align}
The last condition preserves negation even when \(R\) has no symmetry between
positive and negative premises.

\begin{theorem}[Coarsest query-preserving quotient]
\label{th:optimal-query-quotient}
The relation \(E^*\) is a negation-congruent equivalence. For every
negation-congruent equivalence \(E\subseteq E_0\),
\[
 E\text{ is \(Q\)-exact}\quad\Longleftrightarrow\quad E\subseteq E^*.
\]
Consequently, \(E^*\) is the unique coarsest permitted \(Q\)-exact quotient.
\end{theorem}
\noindent\emph{Informally.} Every identification that can preserve these
concrete answers belongs to \(E^*\), and \(E^*\) realizes all of them
simultaneously. Since \(Q\) contains the negations of its claims, quotient
execution preserves all four justification statuses of each query. Exactness
holds for every starting fact set, so the same representation works across
source selections represented by those facts.

The proof connects contextual indistinguishability to quotient execution.
Let \(\rho_E=\gamma_E\alpha_E\), which adds every representative of an
available class. Alternately applying \(\rho_E\) and \(C\), starting from
\(B\), reaches exactly \(\gamma_E(C_E^\sharp(\alpha_E(B)))\). If
\(E\subseteq{\sim_Q}\), adding equivalent representatives preserves the
query answers at every stage. This establishes the concrete readback in
Equation~(\ref{eq:query-exact}). Conversely, quotient execution cannot
distinguish inputs \(H\cup\{x\}\) and \(H\cup\{y\}\) when \(xEy\), so
exactness requires their contextual query profiles to agree.

The construction is effective on the supplied finite instance. Compute
\(Q\cap C(H)\) for each of the \(2^{|S|}\) fact sets \(H\), then group atoms
by their \(E_0\)-class and the two profiles in
Equation~(\ref{eq:query-congruence}). This takes
\(2^{|S|}\operatorname{poly}(|S|,\|R\|)\) time, where \(\|R\|\) is the
explicit rule encoding size. It constructs a representation for the supplied
rules and queries; it does not construct the domain knowledge itself.

An application must include in \(Q\) the signed claims whose exact
derivability it uses, and must keep event-relevant distinctions separate in
\(E_0\). The result preserves derivability from the supplied facts under
\(R\). The soundness of those facts and rules, the
\(\mathsf{SourceSupport}\) licenses, and the independent claim-level
acceptance obligations still require their own justification.

%% file: query-complexity.tex
\paragraph{The cost of retaining exactly the required distinctions.}
The construction also separates a tractable fragment from a hard decision
problem. Here the finite scope, rules, queries, negation, and permitted
partition are represented explicitly. Rule arity is the number of premises;
the rule dependency graph has an edge from each premise to its conclusion.

\begin{theorem}[Complexity of query-preserving quotients]
\label{th:query-complexity}
Under the hypotheses of Theorem~\ref{th:optimal-query-quotient}:
\begin{enumerate}[label=(\alph*),leftmargin=*]
\item If every rule has at most one premise, the optimal quotient \(E^*\)
can be constructed in polynomial time using \(|S|\) concrete closures.
Specifically, put
\begin{equation}
\sigma_Q(x)=
\bigl(Q\cap C(\{x\}),\;Q\cap C(\{\neg x\})\bigr).
\label{eq:unary-query-profile}
\end{equation}
Then \(x\,E^*\,y\) exactly when \(x\,E_0\,y\) and
\(\sigma_Q(x)=\sigma_Q(y)\).
\item Given a permitted pair \(x\,E_0\,y\), deciding whether
\(x\,E^*\,y\) is coNP-complete. Hardness holds even when the rule
dependency graph is acyclic, every rule has at most two premises,
\(Q=\{q,\neg q\}\), and the only nonsingleton classes of \(E_0\) are
\(\{x,y\}\) and \(\{\neg x,\neg y\}\).
\end{enumerate}
\end{theorem}

The second statement concerns the decision for one permitted identification,
rather than assigning a complexity class to the entire synthesis task.
Its contexts range over every subset of \(S\), including internal rule
conclusions. The hardness therefore does not depend on restricting which
atoms a source may supply. These questions are related to Horn
subset-minimal explanation relevance~\cite{EiterGottlob1995} and uniform
equivalence under projection~\cite{OetschTompitsWoltran2007}; the result
above concerns the concrete-query guarantees of the quotient construction.
It supports a polynomial construction for unary rules, while even an acyclic
binary rule system need not admit polynomial-time decisions for individual
merges unless \(\mathrm{P}=\mathrm{NP}\).

%% file: query-preservation-proofs.tex
\begin{proof}[Proof of Theorem~\ref{th:optimal-query-quotient}]
For an equivalence \(E\) on \(S\), put
\(\rho_E=\gamma_E\alpha_E\) and
\(o(A)=Q\cap C(A)\). We first relate execution of the quotient rules to
concrete rule closure and saturation by \(E\).

\emph{Quotient execution.} A set \(D\subseteq S\) is \(E\)-saturated when
\(\rho_E(D)=D\). Such sets correspond bijectively to sets \(Y\subseteq S/E\)
by \(D=\gamma_E(Y)\), \(Y=\alpha_E(D)\). For each rule \(P/c\in R\),
\[
 P\subseteq\gamma_E(Y)\iff\alpha_E(P)\subseteq Y,
 \qquad
 c\in\gamma_E(Y)\iff[c]_E\in Y.
\]
It follows that \(\gamma_E(Y)\) is closed under \(R\) if and only if \(Y\)
is closed under \(\overline{\textit{forget}}_E(R)\). This also covers empty
premise sets and premises that merge into one class.

Starting from \(B\subseteq S\), define
\[
 D_0=\rho_E(B),\qquad D_{i+1}=\rho_E(C(D_i)).
\]
Both maps are extensive and monotone, so the sequence increases and, by
finiteness of \(S\), reaches a fixed point \(D_\infty\). At a fixed point,
\(D_\infty\subseteq C(D_\infty)\subseteq
\rho_E(C(D_\infty))=D_\infty\), so it is \(R\)-closed; it is also
\(E\)-saturated. If \(X\) is any saturated, \(R\)-closed superset of \(B\),
then \(D_0\subseteq X\), and \(D_i\subseteq X\) implies
\(D_{i+1}\subseteq\rho_E(C(X))=X\). Thus \(D_\infty\) is the least such
superset. The preceding correspondence of closed sets proves
\begin{equation}
 D_\infty=\gamma_E(C_E^\sharp(\alpha_E(B))).
 \label{eq:query-quotient-execution}
\end{equation}

\emph{Contextual invariance implies exact execution.} Suppose
\(E\subseteq{\sim_Q}\). For any \(A\subseteq S\), obtain \(\rho_E(A)\)
by adjoining its missing atoms one at a time. When adjoining \(y\), choose
\(x\in A\) with \(xEy\); if the current set is \(H\), then \(x\in H\).
Equation~(\ref{eq:query-profile}) gives
\[
 o(H)=o(H\cup\{x\})=o(H\cup\{y\}).
\]
Consequently,
\begin{equation}
 o(\rho_E(A))=o(A)\qquad(A\subseteq S).
 \label{eq:query-saturation-invariance}
\end{equation}
This is immediate also when \(A=\emptyset\). Closure idempotence gives
\(o(C(A))=Q\cap C(C(A))=o(A)\). Therefore
\[
 o(D_0)=o(B),\qquad
 o(D_{i+1})=o(\rho_E(C(D_i)))=o(C(D_i))=o(D_i).
\]
Since \(D_\infty\) is closed, \(Q\cap D_\infty=o(D_\infty)=o(B)\).
For every \(q\in Q\), Equation~(\ref{eq:query-quotient-execution}) now gives
\[
 q\in C(B)
 \iff q\in D_\infty
 \iff[q]_E\in C_E^\sharp(\alpha_E(B)).
\]
This reads back the particular concrete atom \(q\); it does not require its
class to be a singleton.

\emph{Exact execution requires contextual invariance.} Conversely, suppose
\(E\) is \(Q\)-exact, and take \(xEy\) and \(H\subseteq S\). The quotient
inputs \(\alpha_E(H\cup\{x\})\) and \(\alpha_E(H\cup\{y\})\) are equal.
For every \(q\in Q\), exactness therefore yields
\[
 q\in C(H\cup\{x\})
 \iff[q]_E\in C_E^\sharp(\alpha_E(H\cup\{x\}))
 \iff q\in C(H\cup\{y\}).
\]
Thus \(E\subseteq{\sim_Q}\). If \(E\) is negation-congruent, \(xEy\)
also implies \(\neg xE\neg y\), hence \(\neg x\sim_Q\neg y\).

\emph{Construction and maximality.} Equality of the functions
\(H\mapsto Q\cap C(H\cup\{x\})\) defines the equivalence \(\sim_Q\).
Its pullback along negation is also an equivalence, so
Equation~(\ref{eq:query-congruence}) defines an equivalence \(E^*\).
Negation exchanges the two profile conditions; since \(E_0\) is
negation-congruent and negation is involutive, \(E^*\) is
negation-congruent. It is contained in \(E_0\) and in \(\sim_Q\), so the
sufficiency argument makes it \(Q\)-exact. Every negation-congruent
\(Q\)-exact \(E\subseteq E_0\) satisfies both profile conditions by
necessity, hence \(E\subseteq E^*\). Conversely, any such
\(E\subseteq E^*\) is contained in \(\sim_Q\), so sufficiency makes it
\(Q\)-exact. This proves the characterization and the unique-coarsest claim.
Finally, \(\neg Q=Q\) and \(\neg[q]_E=[\neg q]_E\) give exact positive and
negative derivability for each query, and hence preservation of all four
justification statuses.
\end{proof}

%% file: query-complexity-proofs.tex
\begin{proof}[Proof of Theorem~\ref{th:query-complexity}]
\emph{(a)} For rules with at most one premise, including axioms,
\[
C(B\cup\{x\})=C(B)\cup C(\{x\}).
\]
Every derivation uses an axiom or starts from a single input atom.
Hence \(x\sim_Q y\) exactly when
\(Q\cap C(\{x\})=Q\cap C(\{y\})\). Necessity uses \(B=\emptyset\);
sufficiency uses the displayed union identity. Apply the same argument to
\(\neg x,\neg y\) in Equation~\ref{eq:query-congruence}.
Compute the \(|S|\) singleton closures, then partition each \(E_0\)-class
by the two profiles in Equation~\ref{eq:unary-query-profile}. Finite rule
closure and comparison of these finite profiles take polynomial time.

\emph{(b), membership.} By Equation~\ref{eq:query-congruence}, a permitted
pair fails to belong to \(E^*\) exactly when a context \(B\subseteq S\)
and a query in \(Q\) distinguish either \(x,y\) or \(\neg x,\neg y\).
These are polynomial-size certificates, checked by two finite Horn
closures. Thus the complement is in NP.

\emph{Hardness.} Reduce unsatisfiability of a 3-CNF formula \(\varphi\)
with variables \(v_1,\ldots,v_n\), where \(n\geq2\). Introduce atoms
\(x,y,q\), truth choices \(t_i,f_i\), chain atoms
\(g_1,\ldots,g_{n-1}\), and one atom \(h_j\) per clause. Set
\(g_0=x\), \(g_n=q\), and include the rules
\begin{equation}
\{g_{i-1},t_i\}\longrightarrow g_i,\qquad
\{g_{i-1},f_i\}\longrightarrow g_i,\qquad
\{t_i,f_i\}\longrightarrow q
\quad(1\leq i\leq n).
\label{eq:query-hardness-chain}
\end{equation}
For a literal \(\ell\), let \(a(\ell)=f_i\) when \(\ell=v_i\), and
\(a(\ell)=t_i\) when \(\ell=\neg v_i\). For clause
\(j=(\ell_{j1}\lor\ell_{j2}\lor\ell_{j3})\), add
\begin{equation}
\{a(\ell_{j1}),a(\ell_{j2})\}\longrightarrow h_j,
\qquad
\{h_j,a(\ell_{j3})\}\longrightarrow q.
\label{eq:query-hardness-clauses}
\end{equation}
Make \(\neg t_i=f_i\), \(\neg f_i=t_i\), and give every other atom a
distinct fresh negative partner. No fresh negative partner occurs in a
rule; \(y\) is also absent from the rules. Put \(Q=\{q,\neg q\}\),
and let \(E_0\) have only the nonsingleton classes \(\{x,y\}\) and
\(\{\neg x,\neg y\}\). The construction is polynomial. Every rule has at most
two premises, and the dependencies are acyclic: truth choices feed the chain
and clause atoms, which feed \(q\).

We show that \(x\not\sim_Q y\) exactly when \(\varphi\) is satisfiable.
Given a satisfying assignment, let \(B\) contain its one truth choice for
each variable. From \(B\cup\{x\}\), the chain derives \(q\). From
\(B\cup\{y\}\), no chain atom is reached, no conflict occurs, and no
clause is false, so \(q\) is not derived.

Conversely, let an arbitrary \(B\subseteq S\) distinguish \(x,y\).
Adding inert \(y\) cannot affect a query, and adding \(x\) cannot remove
one. Also \(\neg q\) has no rule and is unaffected by either addition.
Thus \(q\in C(B\cup\{x\})\) and \(q\notin C(B)\).
In particular, neither \(q\) nor \(x\) belongs to \(B\).
Conflict and clause derivations are independent of \(x\), so the new
derivation of \(q\) must use the chain. If \(B\) contains an internal
chain atom, take its largest index \(j\). Reaching \(q\) requires a truth
choice for every index greater than \(j\); those same choices derive
\(q\) from \(g_j\in B\) without \(x\), a contradiction. Consequently
the chain requires a truth choice in \(B\) for every variable.
The conflict rules exclude choosing both. These choices define a complete
assignment, and the clause rules exclude any falsified clause. The
assignment therefore satisfies \(\varphi\). Supplied clause atoms may
enable additional derivations of \(q\) without \(x\), but cannot enable
the chain; hence they do not invalidate this argument for arbitrary
contexts.

Finally, \(\neg x\sim_Q\neg y\), since neither atom occurs in a rule
or in \(Q\). Equation~\ref{eq:query-congruence} therefore gives
\(x\,E^*\,y\) exactly when \(\varphi\) is unsatisfiable. This proves
the stated restricted coNP-hardness and, with membership, completeness.
\end{proof}